\documentclass[final,5p,times,twocolumn]{elsarticle}

\usepackage{epsfig}
\usepackage{graphicx}
\usepackage{float}
\usepackage{nomencl}
\makenomenclature
\usepackage{mathtools}
\usepackage{amsmath,amssymb,amsfonts}
\newtheorem{assumption}{Assumption}
\usepackage[english]{babel}
\usepackage{minted}
\usepackage{balance}
\usepackage{graphicx}
\usepackage{subcaption}
\usepackage{textcomp}
\usepackage{xcolor}
\usepackage{array}
\usepackage{tabularx}
\usepackage{multirow}
\usepackage{longtable}
\usepackage{adjustbox}
\usepackage{booktabs} 
\usepackage{hyperref}

\usepackage{algpseudocode}
\usepackage[linesnumbered,ruled,vlined]{algorithm2e}
\usepackage{svg}
\newtheorem{definition}{Definition}[section]

\newtheorem{lemma}[definition]{Lemma}

\newtheorem{theorem}{Theorem}[section]
\newtheorem{proof}{Proof}[section]
\newtheorem{remark}{Remark}
\journal{Applied Energy}

\begin{document}

\begin{frontmatter}



\title{Toward Adaptive Grid Resilience: A Gradient-Free Meta-RL Framework for Critical Load Restoration}


\author[label1]{Zain ul Abdeen}

\affiliation[label1]{organization={Bradley Department of Electrical and Computer Engineering},%Department and Organization
            addressline={ Virginia Tech}, 
            city={Blacksburg},
            postcode={24060}, 
            state={Virginia},
            country={USA}}
\author[label2]{Waris Gill}

\affiliation[label2]{organization={Department of Computer Science},%Department and Organization
            addressline={ Virginia Tech}, 
            city={Blacksburg},
            postcode={24060}, 
            state={Virginia},
            country={USA}}
\author[label1]{Ming Jin}

\begin{abstract}

Restoring critical loads following extreme events requires adaptive control strategies to ensure resilience in power distribution systems. However, effective restoration remain challenging due to uncertainty in renewable generation, limited dispatchable resources, and nonlinear system dynamics. 
While reinforcement learning (RL) offers promise for sequential decision-making under uncertainty, existing approaches struggle with generalization, requiring extensive retraining when faced with new outage configurations or generation patterns.
% they often demand extensive task-specific tuning and struggle to generalize across unseen outage scenario due to complexity stemming from large policy search space. 
To address these limitations, we propose a \textit{meta-guided gradient-free RL} (MGF-RL) framework, that learns generalizable restoration policies from historical outage experiences, enabling rapid adaption to unseen scenarios with minimal task-specific tuning. Unlike conventional RL that trains separate policies for each scenario, MGF-RL leverage meta learning to distill transferable knowledge into an initialization that quickly adapts across diverse conditions. 
Specifically, MGF-RL integrates first-order meta-learning with evolutionary strategies, enabling scalable policy search without gradient computation while naturally handling the nonlinear, constrained dynamics of distribution systems. 
Through Validation on IEEE 13-bus and 123-bus test systems demonstrates that MGF-RL substantially outperforms standard RL, state-of-the-art meta-RL (MAML), and model predictive control (MPC) across multiple metrics including system reliability, restoration speed, and adaptation efficiency under renewable forecast errors. MGF-RL generalizes effectively to unseen outage scenarios and renewable generation patterns while requiring significantly fewer fine-tuning episodes than conventional RL approaches. Theoretical analysis establishes sublinear regret bounds characterizing adaptation efficiency as a function of task similarity and environmental dynamics. These results position MGF-RL as a scalable and resilient solution for real-time load restoration in distribution grids with high renewable penetration.

% generalizes effectively to unseen outage scenarios and renewable generation patterns, outperforming standard RL, state-of-the-art meta-RL, and model predictive control (MPC) approaches. Theoretical analysis establishes sublinear regret bounds that characterize the algorithm adaptation efficiency as a function of task similarity, providing convergence guarantees for both static and dynamic restoration environments. These results underscore the capability of MGF-RL as a scalable and resilient solution for real-time load restoration in modern distribution grids with high renewable penetration.
\end{abstract}


% \begin{highlights}
% \end{highlights}


\end{frontmatter}


\section{Introduction}
Reliable and high-quality electricity supply is essential for modern society \cite{8375946}. However, extreme events such as hurricanes \cite{kwasinski2019hurricane}, sever storms \cite{comes2014measuring}, system failure \cite{guo2017critical}, and cyber-attacks~\cite{li2019review}, pose escalating threats to power systems, leading to large-scale power outages and significant economic losses \cite{campbell2012weather, karimi2025resilience, kenward2014blackout}. In the U.S. alone, weather-induced blackouts result in estimated annual losses of \$20–55 billion \cite{campbell2012weather}. Following major disruptions, critical loads including healthcare facilities, emergency services, and essential infrastructure must be restored rapidly to prevent sever societal, security, and economic consequences \cite{karimi2025resilience, zhang2022curriculum}. 
% Approximately 80\% of large-scale power outages from 2003 to 2012 were due to severe weather \cite{o2017ipcc, xu2024resilience}, a trend expected to worsen due to climate change. 
% Given the central role of power systems in maintaining critical infrastructure resilience, load restoration plays an important role in mitigating outage impact \cite{ul2024enhancing, karimi2025resilience, schweikert2021vulnerability, banerjee2023autonomous, zhang2022curriculum, zhang2023primal}. 
% Following major outages and given the central role of power systems in maintaining the resilience of critical infrastructure, critical loads such as health care facilities, emergency services and essential infrastructure must be restored to prevent severe social, economics, and security consequences \cite{ul2024enhancing, karimi2025resilience, schweikert2021vulnerability, banerjee2023autonomous, zhang2022curriculum, zhang2023primal}. 

Smart grid technologies offer techniques for critical load restoration load restoration (CLR) through remotely controlled switches~\cite{selim2022deep}, and enabling post-event system reconfiguration~\cite{wang2019coordinating}. Typically, outage areas are re-energized by connecting to an alternative feeders or neighboring substations \cite{mohagheghi2011applications}. However, during extreme events, these resources are unavailable or unable to provide sufficient power, necessitating reliance on local distributed energy resources (DERs) for restoration of the distribution systems \cite{poudel2018critical,wang2018risk}. However, DER-based restoration introduces significant challenges, renewable generation uncertainty, limited dispatchable capacity, complex load prioritization requirements, and nonlinear power flow constraints during islanded operations \cite{wang2018risk}. This work focuses on optimally leveraging both dispatchable and intermittent DERs to maximize load restoration while managing renewable forecast uncertainty, a critical requirement for resilient distribution system operation.
% a major challenge with DER-based CLR is managing the uncertainty of renewable generation for optimal control \cite{wang2018risk}.
% This is also the focus of this paper: we aim at studying how optimally leverage DERs, both dispatchable and renewable sources, to maximize load pick-up and other resilience-related metrics.

%%%%%%%%%%%%%%%%%%%%%%%%

% We enforce AC power flow constraints through OpenDSS integration, ensuring voltage and thermal limits are satisfied during policy execution without requiring convex relaxations that could compromise solution feasibility.

%%%%%%%%%%%%
Traditional approaches for solving the CLR problem with intermittent DERs includes optimization based algorithms \cite{liu2020bi}, heuristic methods \cite{li2014distribution}, and deterministic techniques \cite{yao2019rolling}. While these methods have demonstrated effectiveness under linearized power flow assumptions, they often struggle with the dynamic and nonlinear nature of power systems during emergencies~\cite{huang}. MPC-based methods address some limitations by re-optimizing decisions at each control step, and leverage the latest renewable-based DERs forecasts to re-align the system with the optimal trajectory \cite{RC-MPC}.  
Liu et al. \cite{liu2020collaborative} proposed using MPC to coordinate multiple types of DERs in the restoration of distribution systems, later extended in \cite{zhao2018model} to optimize generators start-up for system restoration. Other MPC based restoration examples are \cite{yao2019rolling,eseye2021resilient}.
Stochastic programming techniques such as chance-constrained formulations \cite{ccop} use probabilistic constraints to satisfy the system constraints within specified threshold. Despite these advancement, these approaches face critical limitations during emergency restoration: (i) solving mixed-integer nonlinear programs repeatedly during restoration is computationally prohibitive, especially as system size scales~\cite{liu2020collaborative}; (ii) optimization methods require accurate system models and uncertainty characterization, which may be unavailable or inaccurate; and (iii) solutions are task-specific and cannot readily transfer acknowledge across different outage scenarios, requiring re-optimization for each new event. 

% Despite these advancements, CLR remains challenging due to load prioritization requirements, network topology constraints, renewable intermittency, and the computational burden of solving complex optimization problems repeatedly during restoration, necessitating the development of more adaptive and efficient methodologies.
%%%%%%%%

RL has emerged as a complementary approach that addresses several key limitations of optimization methods for CLR problem \cite{bedoy,9302946}, by learning optimal control policies from interactions with the dynamically changing environment~\cite{du2022deep}. RL adaptability to non-linear dynamic systems, without being restricted to specific model requirements, makes it a powerful tool for the CLR problem with following merits: (a) RL trains control policies offline, enabling fast inference during restoration without intensive on-demand optimization.
% and do not require intensive on-demand computation during restoration process
This allows RL policies to provide
restoration plan for a longer control horizon and a shorter
control interval--problems that cannot be efficiently
solved by optimization between control intervals. Once the policy is trained , RL provides fast real-time response due to the simplicity of policy evaluation. (b) RL agents learn directly from nonlinear power flow models without requiring convex relaxations or linearizations, thus preserving control accuracy.  
% and thus avoid losing control accuracy due to model simplification.
(c)~RL conducts end-to-end learning from historical data, eliminating the need to explicitly model uncertainty distributions or generate scenarios for stochastic programming.

% RL can directly learn from historical data and conduct an end-to-end uncertainty management, eliminating the need to estimate uncertainty distribution or generate scenarios for stochastic programming. 
Motivated by these merits, Zhao et. al. \cite{zhao2021learning} solve the CLR problem using graph based deep Q-network (DQN), and \cite{9302946} proposed a policy gradient-based RL approach to restore critical load using local DERs. However, due to the non-convex nature of policy optimization landscape, standard RL approaches often demand large-scale training data, complex hyperparameter tuning, and can struggle with adaption when confronted with altered operation conditions~\cite{zhao2021learning}. To address these challenges, curriculum-based RL has been proposed by progressively introducing tasks of increasing complexity to accelerate training convergence and improve solution quality for CLR problems~\cite{zhang2022curriculum}. Despite their advantages, existing RL approaches face a critical limitations of \textit{poor generalization across different outage scenarios}. These policies are trained for specific system conditions and require extensive fine tuning or retraining, when encountering different grid topologies, renewable availability patterns, or load distributions. This lack of adaptability creates a significant barrier to practical deployment, as emergency operators cannot afford to retrain policies during unfolding crises.

Meta-RL algorithms on the other hand, aims at \textit{learning how to learn} to create a generalized meta-policy that can quickly adapt to new tasks by leveraging shared knowledge from previously encountered tasks~\cite{khattar2022cmdp,finn}. Recently, Fan et al. \cite{fan2024enhancing} proposed meta-based graph DQN for load restoration to enhance adaptability of RL to new unseen task. Unlike their paper, we investigating RL capability to manage uncertainty in CLR which provided with imperfect renewable forecasts. Furthermore, their value- based method suffers from an over-estimation problem when performing the Q-value estimation, which may affect the stability of the algorithm. Nevertheless, popular meta-RL frameworks, such as model-agnostic meta-learning (MAML) are computationally demanding because they frequently require costly second-order derivative computations (i.e. Hessian matrix) for updating meta-parameters.  
These second-order updates, can hinder efficiency by amplifying computational burden, particularly when multiple gradient steps are needed during the test phase. To address these limitations, we propose a \textit{meta-guided gradient-free RL} framework that circumvents second-order derivative while preserving fast adaptability. Our approach couples within task gradient free Evolution-Strategy RL (ES-RL)~\cite{ES} updates with first order meta-update \cite{Meta}, which simplifies the learning process by avoiding complex Hessian computations and employing gradient averaging. Gradient averaging involves updating the model parameters by averaging the first-order gradients obtained from task-specific optimizations, which significantly reduces the variability and instability in updates.

\noindent\emph{\bf{Contribution:}} Building on our earlier work \cite{ul2024enhancing}, which focused on first-order meta-RL in a simplified CLR setting with DER dispatch alone, this paper significantly extends both problem scope and technical rigor. Compared to \cite{ul2024enhancing}, this extended study introduces: (i) more complex restoration actions including sequential load pickup decisions and coordinated DER setpoint control; (ii) explicit consideration of renewable forecast errors; and (iii) rigorous benchmarking against an MPC-based optimization method and contemporary meta-RL methods (MAML, AC-RL). By balancing computational efficiency with improved generalization, MGF-RL is positioned to address key challenges in CLR. The main contributions are summarized as follows:

\noindent 1)  We propose MGF-RL for CLR, which integrates a gradient-free evolution strategy for within-task policy optimization with first-order meta-updates for task generalization. MGF-RL reduces computational overhead and improves adaptation to unseen outage conditions, requires only 2-4 adaptation episodes for new scenarios compared to 15-20 episodes for standard RL (see Figure \ref{fig1}). 
% Unlike standard MAML requiring Hessian computation, MGF-RL achieves comparable or superior adaptation performance with 3-5\times lower computational cost and 2\times lower memory requirements (see Table 6), making it practical for real-time deployment in resource-constrained operational environments. approach enhanced by first-order meta-updates that accelerates policy learning for CLR. MGF-RL reduces computational overhead and improves adaptation to unseen outage conditions.

\noindent2) Through experimental evaluation on IEEE 13-bus and 123-bus distribution systems across diverse restoration scenarios, we observe MGF-RL demonstrates consistent advantages in system reliability, load restoration completeness, meta-learning effectiveness, and computational scalability. MGF-RL achieves 27-41\% System Average Interruption Duration Index (SAIDI) improvements over all baseline methods and reaching the critical 90\% milestone that baselines fail to meet within the control horizon. Evaluation on unseen test tasks reveals superior adaptation stability, with MGF-RL exhibiting positive performance improvements while gradient-based meta-RL methods show substantial degradation. Computational analysis confirms sub-linear scaling characteristics suitable for practical deployment.

\noindent3) We analyze performance evaluation under varying levels of renewable forecast error, a critical requirement for practical deployment where perfect forecasts are unavailable. MGF-RL maintains stable performance under typical forecast uncertainty while other methods exhibit substantial degradation. Analysis of controller behavior reveals that MGF-RL autonomously adopts conservative strategies under high uncertainty—prioritizing dispatchable resources and implementing gradual load pickup—without explicit programming of uncertainty-handling rules.

\noindent4) We provide theoretical analysis of task-averaged regret for MGF-RL establishing sublinear bounds  $\mathcal{O}\left(\frac{V_{M}+\epsilon_{M}}{M\sqrt{T}}+\frac{\hat{D}_{*}}{\sqrt{T}}\right)$, which improves with higher task similarity $(\text{i.e., lower}~~ \hat{D}_{*})$ in static settings. For dynamic environments, we derive a regret bound $\mathcal{O}\left(\min\left\lbrace V_{M}, \sqrt{M(1+\mathcal{P}_{M})} \right\rbrace +\frac{\hat{S}_{*}}{\sqrt{T}}+\frac{\epsilon_{M}}{M\sqrt{T}}\right),$ proving that adaptive learning rates can mitigate performance declines under changing environment conditions. This regret is influenced by task variability $V_{M}$ and the path length of optimal policy $\mathcal{P}_{M}$. 

    The remainder of this paper is organized as follows. Section~\ref{problemformaulation} formalizes the CLR problem. Section~\ref{Metasection} presents our proposed MGF-RL strategy designed for tackling the CLR problem. Theoretical analysis is the focus of Section~\ref{theoretical}, while Section~\ref{casestudy} offers case study results to substantiate our approach. Lastly, Section~\ref{conclusion} concludes the paper with potential directions for future research and key findings.

\section{Problem Formulation}\label{problemformaulation}
 \subsection{Critical Loads Restoration Problem}
 We consider a multi-step, prioritized CLR problem in which a distribution system is islanded from the main grid due to sever events. The objective is to optimize the restoration of prioritized critical loads using local DERs over discrete time intervals denoted by $\mathcal{T}=\left\lbrace 1,2,\dots,T \right\rbrace$, thereby \textcolor{black}{enhancing} system resilience throughout the outage. \textcolor{black}{Available DERs are categorized into: dispatchable fuel-based resources $(\mathcal{D}^f)$, battery storage $(\mathcal{D}^s)$, and renewable-based resources~$(\mathcal{R})$. Let $\mathcal{G} := \mathcal{R}\cup \mathcal{D}^f \cup \mathcal{D}^s$ denote the set of all DERs. The vector $\boldsymbol{\varsigma}=[\varsigma^1,\varsigma^2,\dots,\varsigma^N]^\top \in \mathbb{R}^{N}$ represents the priority for the set of critical loads $\mathcal{L}$, where larger $\varsigma^i$ indicates higher restoration priority for load $i$.}
 % where $\varsigma^i$ reflects the priority of load~$i$}, higher value of $\varsigma^{i}$ indicate greater priority. 
 At each control step $t\in \mathcal{T}$, the controller specifies the active power setpoints $\textbf{p}_t^\mathcal{G}\in \mathbb{R}^{|\mathcal{G}|}$ and power factor angles $\boldsymbol{\alpha}_t^\mathcal{G}\in \mathbb{R}^{|\mathcal{G}|}$ for all DERs, which determine reactive power injection through
$\mathbf{q}_t^{\mathcal{G}}=\mathbf{p}_t^{\mathcal{G}}\odot \tan(\boldsymbol{\alpha}_t^{\mathcal{G}})$. Here, $\odot$ denotes the Hadamard elementwise product, and $\tan([\alpha^1,\dots,\alpha^n]^{\top})=[\tan(\alpha^1),\dots,\tan(\alpha^n)]^{\top}$. Concurrently, active and reactive power restoration levels for each critical loads are represented by $\textbf{p}_t \in \mathbb{R}^N$ and $\textbf{q}_t \in \mathbb{R}^N$, respectively. 
Thus, the control objective is to maximize the following weighted sum over the restoration horizon, as outlined in ~\cite{zhang2022curriculum}:
\begin{equation}\label{eq1}
\sum_{t\in\mathcal{T}}\left(\boldsymbol{\varsigma}^{\top} \textbf{p}_t - {\mu}\boldsymbol{\varsigma}^{\top}[\textbf{p}_{t-1}-\textbf{p}_t]^{+}+\mathcal{V}_t\right)
\end{equation}
The first term, $\boldsymbol{\varsigma}^{\top} \textbf{p}_t$ emphasizes restoring high priority loads, and ${\mu}\boldsymbol{\varsigma}^{\top}[\textbf{p}_{t-1}-\textbf{p}_t]^{+}$ penalizes frequent fluctuations in load pick-up and shedding, by factor $\mu$, promoting monotonic restoration under variable DER outputs. Here, $\mathcal{V}_t:=-\lambda \|[\boldsymbol{\nu}_t-\bar{\boldsymbol{\nu}}]^{+}+[\underline{\boldsymbol{\nu}}-\boldsymbol{\nu}_t]^{+}\|_{2}^{2}$ represents the voltage violation penalty across all $N_b$ buses at time $t$,
where $\boldsymbol{\nu}_t\in \mathbb{R}^{N_b}$ is the vector of bus voltage magnitudes, $\underline{\boldsymbol{\nu}}=V^{\min}1_{N_b}\in \mathbb{R}^{N_b}$, and $\bar{\boldsymbol{\nu}}=V^{\max}1_{N_b}\in \mathbb{R}^{N_b}$ define the allowable voltage range. The coefficient $\lambda$ regulates the trade-off between restoration benefit and voltage quality. 
% The coefficient $\lambda$ encourages maintaining bus voltages to be within limits. 
It is worth noting that voltage bounds are included as a penalty term, as they represent system-controlled outcomes that cannot be directly constrained within RL framework.

At each control step $t\in\mathcal{T}$ in restoration process, DER and load decisions must satisfy certain operational constraints. Dispatchable fuel-based DERs constraint~\eqref{fuel_eq} specify allowable power output ranges for micro-turbines dictated by fuel availability, where $\tau$ is the control interval length and $E^{f}$ is the known fuel reserve limit. \textcolor{black}{Battery storage satisfies power and stat-of-charge (SOC) constraints \eqref{batteryeq}-\eqref{batteryeq3} and 
charge/discharge rates incorporating storage efficiency $\zeta_{t}$, taking values based on whether the battery is charging $\zeta_{t}=\zeta^{\text{ch}} \left(\text{i.e.}, p_{t}^{\theta}>0\right)$ or discharging $\zeta_{t}=\frac{1}{\eta^{\text{dis}}}  \left(\text{i.e.}, p_{t}^{\theta}<0\right)$. $S^{\theta}_{t}$ and $s_0$ are the current and initial SOC.} 
Renewable-based DER constraints are in \eqref{reneweq4}, set renewable generation equal to forecasted values $\hat{p}_{t}^{r}$, which fluctuate based on natural conditions and prioritized for use during restoration. The symbol $\hat{\cdot}$ denotes forecasted values. Finally, the load pickup constraint \eqref{loadpick} ensure that restoration decisions do not exceed available demand and maintain constant power factor.
\begin{equation}\label{fuel_eq}
    p_{t}^{f}\in[\underline{p}^{f},\bar{p}^{f}],\quad \alpha_{t}^{f} \in [\underline{\alpha}^{f},\bar{\alpha}^{f}],\quad \sum_{t\in \mathcal{T}}p_{t}^{f}.\tau \leq E^{f}
\end{equation}
\begin{equation}\label{batteryeq}
    -p^{\theta,\text{ch}}\leq p_{t}^{\theta}\leq p^{\theta,\text{dis}}
\end{equation}
\begin{equation}\label{batteryeq2}
    S_{t+1}^{\theta}=S_{t}^{\theta}-\zeta_{t}\cdot p_{t}^{\theta}\cdot \tau, \quad S_{0}^{\theta}=s_{0} 
\end{equation}
\begin{equation}\label{batteryeq3}
    \underline{S}^{\theta}\leq S_{t}^{\theta}\leq\bar{S}^{\theta},\quad \forall ~\theta \in \mathcal{D}^{s}
\end{equation}
\begin{equation}\label{reneweq4}
    p_{t}^{r}=\hat{p}_{t}^{r},\quad \alpha_{t}^{r}\in \left[\underline{\alpha}^{r},\bar{\alpha}^{r}\right],\quad \forall r\in \mathcal{R}
\end{equation}
\begin{equation}\label{loadpick}
    \bold{0}\leq \textbf{p}_{t}\leq \textbf{p}, \quad \bold{0}\leq \textbf{q}_{t}\leq \textbf{q}, \quad \text{p}_{t}^{i}/\text{q}_{t}^{i} = \text{p}^{i}/\text{q}^{i} 
\end{equation}
Since the distribution system is islanded, the balance between loads and DERs generation must be satisfied. Additionally, electrical consistency among power injections and bus voltages is enforced through a power flow model $h(\cdot)$, leading to following constraints:
\begin{equation}\label{powerflowconstraints}
\begin{split}
\boldsymbol{1}^{\top}_{N}&\textbf{p}_{t} =\boldsymbol{1}_{|\mathcal{G}|}^{\top}\textbf{p}^{\mathcal{G}}_{t}, ~~\boldsymbol{1}^{\top}_{N}\textbf{q}_{t} =\boldsymbol{1}_{|\mathcal{G}|}^{\top}\left(\textbf{p}^{\mathcal{G}}_{t}\odot\tan (\boldsymbol{\alpha}_{t}^{\mathcal{G}})\right),\\ &\boldsymbol{\nu}_t=h(\textbf{p}_{t},\textbf{q}_t,\textbf{p}_{t}^{\mathcal{G}},\boldsymbol{\alpha}_{t}^{\mathcal{G}})
\end{split}
\end{equation}
% Here, $\odot$ denotes the Hadamard elementwise product, and $\tan([\alpha^1,\dots,\alpha^n]^{\top})=[\tan(\alpha^1),\dots,\tan(\alpha^n)]^{\top}$. 
The power flow function $h(\cdot)$ calculates bus voltage magnitude $\boldsymbol{\nu}_{t}$ based on system power injection vectors and $\boldsymbol{\nu}_{t}$ is used to evaluate the voltage violation penalty $\mathcal{V}_{t}$. Specifically, for RL-based implementation, we instantiate $h(\cdot)$ using OpenDSS simulator, which perform power flow calculation,
% the distribution system simulator OpenDSS is used to instantiate $h$, 
while for optimization-based baseline controllers (e.g. MPC), we use the linear power flow model, LinDistFlow \cite{gan2014convex}, with the same parameters to maintain consistency across methods.

Integrating both the objective function \eqref{eq1} and associated constraints \eqref{fuel_eq}-\eqref{loadpick}, the optimal control problem for distribution system restoration is formulated as:
 \begin{equation}\label{OCP}
  \begin{split}
\text{maximize}_{\textbf{p}_{t},\textbf{q}_{t},\textbf{p}_{t}^{\mathcal{G}},\boldsymbol{\alpha}_{t}^{\mathcal{G}},\forall t\in \mathcal{T}} \quad\eqref{eq1}\\ \text{subject to}\quad{\forall t\in \mathcal{T}} \quad \eqref{fuel_eq}-\eqref{powerflowconstraints}
 \end{split}
 \end{equation}

 This problem \eqref{OCP} is framed as a Mixed-Integer Linear Programming (MILP). %formulation While this representation serves to enhance comprehension of the CLR problem, subsequent sections will employ RL techniques for its resolution.
 \textcolor{black}{Existing methodologies for the MILP formulation of CLR problem include No Reserve MPC (NR-MPC) \cite{liu2020collaborative} and Reserve Considered MPC (RC-MPC) \cite{RC-MPC}. NR-MPC iteratively solves the MILP in a receding horizon manner with updated renewable forecasts, while RC-MPC incorporates generation reserve penalties to hedge against renewable forecast errors, enhancing robustness. However, these MILP solutions are computationally intensive, less scalable, and rely on linear approximations, limiting their adaptability in rapidly changing environments. In the subsequent sections, we address these issues by proposing a meta-RL strategy (Section~\ref{Metasection}) that uses learning-based methods to handle uncertainties in DER outputs and system operations more adaptively.
 % , we introduce a meta-RL strategy. This approach leverages learning-based methods to adaptively handle the complexities associated with the CLR problem, including handling variability in DER outputs and operational uncertainties.
 }

\subsection{Sequence of CLR Problems}
\textcolor{black}{In the context of meta-learning, we consider a sequence of CLR problems indexed by $m = 1, \dots, M$. Each problem~$``m"$ in the sequence corresponds to a distinct environment or scenario, characterized by varying parameters such as critical load demands $(\textbf{p}^{(m)}, \textbf{q}^{(m)})$, priority vector $\boldsymbol{\varsigma}^{(m)}$, and renewable generation forecasts $\hat{p}^{r(m)}_{t}$. These variations create a diverse set of CLR problems, each reflecting unique operational challenges and uncertainties.
The goal of meta-learning is to learn a policy that can quickly adapt to new unseen CLR problems, by leveraging the experience gained from solving prior problems. This formulation enables evaluation of the generalization and adaptation capabilities of the proposed method across diverse restoration scenarios.
% This formulation allows us to study adaptability and generalization capabilities of the proposed methods in the face of varying environmental conditions and problem parameters.
}

\emph{\textbf{Simulation and Evaluation:}} To evaluate the proposed methods for the CLR problem, we conduct simulations based on the following setup: 1) It is assumed that the demand for each critical load~$i \in \mathcal{L}$, represented as $(\textbf{p}=[p^1,\dots,p^{N}]^\top$ $\in \mathbb{R}^{N} ~\text{and} ~\textbf{q}=[q^1,\dots,q^{N}]^\top \in \mathbb{R}^{N})$, stays constant over the course of the restoration horizon, reflecting conservative restoration planning under emergency conditions where load behavior is more predictable \cite{wang2018risk,zhang2022curriculum}. Furthermore, intelligent grid edge devices equipped with smart switches enable fine-grained load control, allowing partial and continuous restoration  at each time step \cite{zhang2022curriculum}. 2) Renewable generations from photovoltaic and wind sources are forecast in advance, although forecasts may contain errors. This reflects realistic operating conditions where generation uncertainty must be considered for robust decision-making. 3) Simulations focus on steady-state DERs dispatch and load restoration decisions, neglecting the transient dynamics of the distribution system. 4) At the beginning of the control horizon $t=1$, we assume that the radial distribution network has already been reconfigured and re-energized. Here, reconfiguration refers to a one-time offline switching action performed during the restoration planning stage. All DERs are synchronized, initiated by black-start-capable units, and ready for load pickup. The resulting topology remains fixed throughout the restoration process. These assumptions allow the RL agent to focus solely on post-reconfiguration scheduling of DER dispatch and load restoration. 
%%%%%%%%%%%%%%%%%%%%%%%%%%%%%%%%%%%%%%%%%%%%%%%%%%%%

\section{Meta RL for Sequential CLR}\label{Metasection}
 This section presents the proposed MGF-RL framework for solving the CLR formulated as optimal control problem in~\eqref{OCP}. We begin by reformulating the CLR as a Markov decision process (MDP) interaction to enable RL-based learning, then describe how MGF-RL leverages meta-learning principles to train a resilient CLR controller that rapidly adapt across diverse restoration scenarios. 
\subsection{Reformulating CLR as MDP}
The MDP formulation of CLR problem \eqref{OCP} allows us to leverage the framework of RL to learn optimal control policies that optimize long-term restoration performance under uncertainty. State, action, and reward are the key elements of MDP are defined as:

\emph{\textbf{State}} 
 $(\mathcal{S})$: The state $s_{t}\in \mathcal{S}$ captures system status at time step $t$ and serves as input to RL agent for decision-making. 
 % The state $s_{t}\in \mathcal{S}$ serve as the input to the policy for decision-making at each control step. It reflects system status at the current step and defined as:
$$s_{t}:=\left[\left(\boldsymbol{\hat{p}}^{r}_{t}\right)^\top, \left(\tilde{p}^{}_{t-1}\right)^\top, \left({S}^{\theta}_{t}\right)^\top, \left({E}^{\mu}_{t}\right)^\top, t \right]^{\top} \in \mathcal{S},$$
% where $t$ is the current control step, indicating temporal progress. 
where, $S_{t}^{\theta}\in \mathbf{R}^{|\mathcal{D}^{s}|}$ and ${E}^{\mu}_{t}\in \mathbf{R}^{|\mathcal{D}^{f}|}$ denote the state of charge and remaining fuel reserve, which indicates its residual capacity to support load. Second component 
$\tilde{p}^{}_{t-1}:= \text{diag}\left\lbrace\textbf{p}\right\rbrace^{-1}\textbf{p}_{t-1}\in \mathbb{R}^{N}$ reflects the load restoration level from the previous step (all loads initially start at 0 kW). The renewable generation forecast vector $\boldsymbol{\hat{p}}^{r}_{t}:=\left[p^{r}_{t},\hat{p}^{r}_{t+1|t},\dots,\hat{p}_{t+(\kappa/\tau)-1|t}\right]^{\top}\in   \mathbb{R}^{\kappa/\tau},$ is provided by the grid operator, 
where $\hat{p}^{r}_{t+x|t}$ denotes the forecast made at step $t$ for renewable generation at time $t+x$, parameter $\tau$ is the control interval length in hours, and the look-ahead length $\kappa$ determines the forecast horizon dimension $\kappa/\tau$.
% , or the number of forecast points over a $\kappa$-hour period. 
For instance, with $\tau=1/12$ (5-minute control interval) and $\kappa=4$ (four-hour look-ahead), the forecast vector contains $|\boldsymbol{\hat{p}}^{r}_t|=48$ data points. It is important to highlight two key aspects concerning the inclusion of $\boldsymbol{\hat{p}}^{r}_t$ in $s_t$: i) Renewable forecasts are essential for anticipating future generation trends, enabling the controller to make proactive dispatch decisions. ii) However, these forecasts~$\boldsymbol{\hat{p}}^{r}_{t}$ are inherently imprecise due to weather prediction limitations, affecting the RL controller performance. As demonstrated in \cite{zhang2022curriculum}, imperfect forecasts can lead to inaccurate control formulations and suboptimal decisions, a challenge also faced by MPC when solving optimization problems under erroneous forecasts. 
Section~\ref{secforecasterror} systematically evaluates robustness to forecast uncertainty.

% \begin{itemize}
%     \item Renewable generation forecasts are integral to grid control problems, making their inclusion in $s_t$ justified for informing about future generation capabilities, as indicated by $\boldsymbol{\hat{p}}^{r}_{t}$ which helps the controller anticipate environmental trends.
 
%     \item However, these forecasts $\boldsymbol{\hat{p}}^{r}_{t}$ are inherently imprecise, affecting the RL controller’s performance. As demonstrated by studies like \cite{zhang2022curriculum}, imperfect forecasts can lead to inaccurate control formulations and suboptimal decisions, a challenge also faced by MPC when solving optimization problems under erroneous forecasts.
% \end{itemize}
\emph{\textbf{Action}} $\left(\mathcal{A}\right):$ At each control step $t\in \mathcal{T}$, the action $a_t$ represents the control decision taken based on the current state $s_{t}$, determined by the RL policy $\pi:\mathcal{S}\to \Delta(\mathcal{A})$ parametrized by~$\phi$ i.e., $\boldsymbol{a_{t}}=\boldsymbol{\pi(s_{t};\phi)}$. The policy outputs a distribution over actions, where $\Delta(\mathcal{A})^{|\mathcal{S}|}$ denotes the set of probability distributions over all possible actions for each state. The action vector $a_t$ governs the decisions for both; (i) load restoration amounts, and (ii) the allocation of active and reactive power setpoints for selected DERs, and formally defined as: 
\begin{equation}\label{actionvector}
a_{t}:=\left[\left(\textbf{p}_{t}\right)^{\top}, (p^{\theta}_{t})^{\top},(p_{t}^{\mu})^{\top}, \left(\textbf{H}_{\alpha}\boldsymbol{\alpha}_{t}^{\mathcal{G}}\right)^{\top}\right]^{\top}\in\mathcal{A},
\end{equation}
where, the vector $\textbf{p}_t$ determines the active power restoration levels for critical loads, while $p^{\theta}_{t}$ and $p^{\mu}_t$ represents power output from energy storage systems and dispatchable fuel-based DERs, respectively. The selection matrix~$\textbf{H}_{\alpha}\in\mathbb{R}^{(|\mathcal{G}-1|)\times |\mathcal{G}|}$ ensures control of both active and reactive power through DER power factor angles $\boldsymbol{\alpha}_t^{\mathcal{G}}$.
% and reactive power from all DERs. 
Note that renewable generation outputs (i.e., $p_{t}^{w}$ and $p_{t}^{\rho}$) are excluded from $a_{t}$, as they are generally dispatched at full capacity unless curtailment is required due to grid constraints (e.g., voltage violations). Also, reactive power of loads $\textbf{q}_{t}$ is not explicitly included in action vector \eqref{actionvector}, instead it is inferred using fixed power factor assumption $q_{t}^{i}=p^{i}_{t}\cdot(q^{i}/p^{i}), \forall i \in \mathcal{L}$, based on the practical invariance of load characteristics during outage.
% which follows from assumption of load invariance.
% By defining the load restoration level and optimizing the use of available generation from dispatchable DERs, 
The chosen action $a_t$ directly impacts the subsequent state variables i.e., restored load $\tilde{p}_t$, state of charge $S^\theta_{t+1}$ of storage devices, and the remaining fuel level $E^\mu_{t+1}$ for dispatchable generators. Over the control horizon, RL agent learns a policy~$\pi$ to optimize actions to effectively formulate a strategy for system restoration.

\emph{\textbf{Reward}:} $r_{t}$ serves as a scalar evaluation of the control action~$a_{t}$, based on the state $s_t$. Corresponding to \eqref{eq1}, the
reward is defined as, $$r_t = \boldsymbol{\varsigma}^{\top} \textbf{p}_t - {\mu}\boldsymbol{\varsigma}^{\top}[\textbf{p}_{t-1}-\textbf{p}_t]^{+}+\mathcal{V}_t$$ and is computed using the simulation outcomes at $t$. The reward encourages high-priority load restoration while penalizing load fluctuations and voltage violations.
\begin{figure}
\centering
\includegraphics[width=0.95\linewidth]{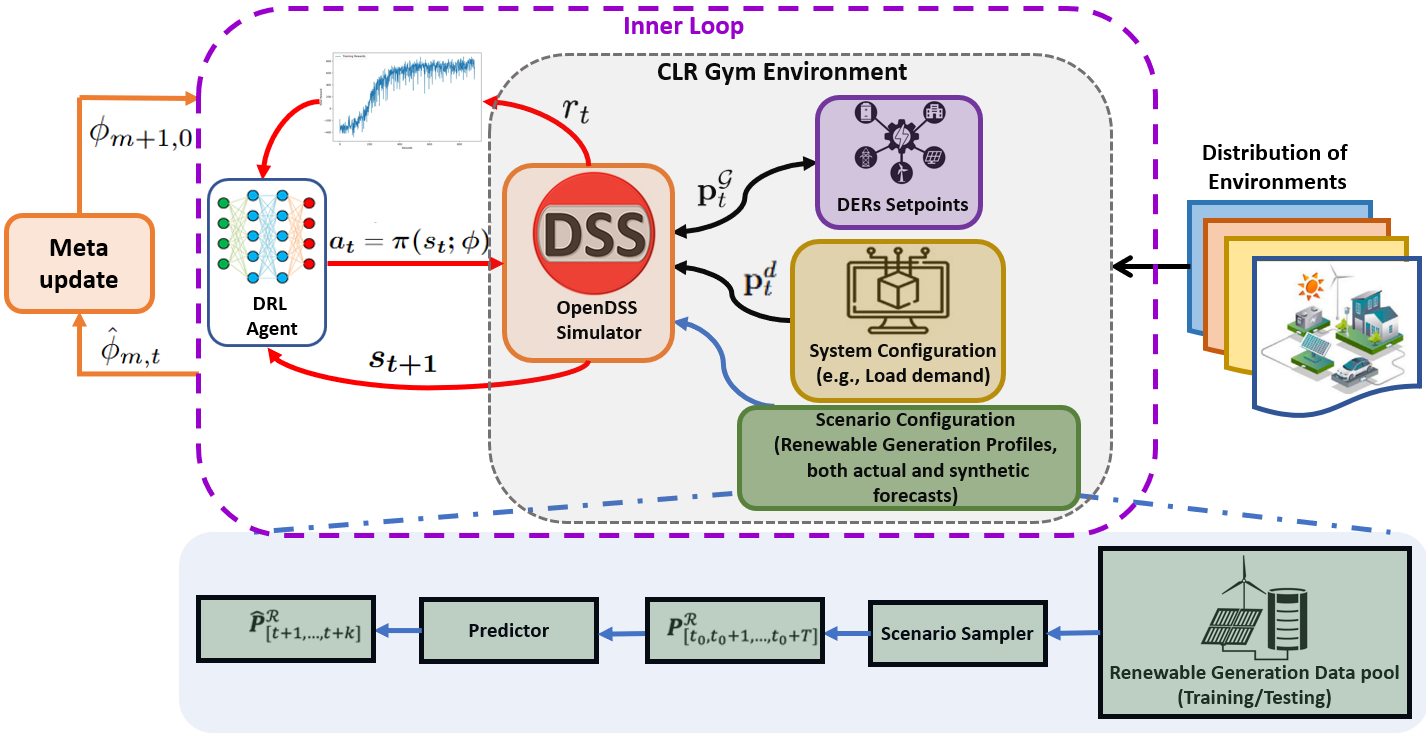}
\caption{The CLR learning environment integrates a DRL agent with OpenDSS for optimizing load restoration through iterative policy improvement. 
% The agent explores $T$-step scenarios via a reset interface and uses meta-updates to improve adaptability across tasks.
}
\label{CLRmodel}
\end{figure}

\emph{\textbf{Environment Interaction:}} 
We use a custom OpenAI Gym~\cite{brockman2016openai} learning environment that interfaces an RL agent with the OpenDSS~\cite{dugan2018open} grid simulator, as illustrated in Fig.~\ref{CLRmodel}. The CLR learning environment facilitates interaction between the RL agent and the simulated distribution network while enforcing operational constraints, such as box constraints on individual power elements and power balance constraints to maintain system stability. In essence, the aim of RL agent is to find an optimal control policy, denoted as $\pi^{*}$, which maximizes the expected total reward over control horizon $t\in 
\mathcal{T}$, i.e., 
\begin{align}
 \pi^{*}(a_t|s_{t})=\arg\max_{\pi}F\left(\phi\right)
\end{align} 
where, $F(\phi):=\mathbb{E}_{a_{t}\sim\pi(s_t;\phi), \xi \sim \Re^{tr}}\left(\sum_{t\in \mathcal{T}} r_{t}\right)$, and $\phi$ represents parameters of the policy network $\pi(s_{t};\phi)$, modeled as neural network and $\Re^{tr}$ is the set of renewable generation training scenarios.

% The expectation is taken over renewable generation scenarios $\Re^{tr}$, sampled from a training scenario set. 
The scenario configuration block in Fig.~\ref{CLRmodel} samples training $\Re^{tr}$ and testing $\Re^{ts}$ scenarios. During training, the RL agent iteratively learns the optimal policy $\pi^{*}(a_{t}|s_{t})$ by interacting with the environment and sampling scenarios $\xi \sim \Re^{tr}$.
% Upon sufficient training, the RL agent is exposed to various outage conditions, enabling it to make informed control decisions. 
Each scenario $\xi= \cup_{r\in \mathcal{R}}\left({p}_{1:T}^{r},\hat{{p}}_{1:T}^{r}\right)$ includes both the actual renewable generation profile ${p}^{r}_{1:T}$, which remains unknown to RL agent and used solely for simulation to realized power flows, and the forecasted generation values $\boldsymbol{\hat{p}}_{t}^{r}$, available to the agent at each step for decision making. We construct these scenarios using historical renewable generation data $\mathcal{P}^{\mathcal{R}}$ from the \textit{NREL WIND Toolkit}~\cite{draxl2015wind}. Since this dataset lacks operator forecasts, we synthesize $\boldsymbol{\hat{p}}^{r}_{t}$ using the forecast generation method $\mathcal{M}\left(p^{r}_{1:T},\Xi\right)$ proposed in \cite{zhang2022curriculum}, 
% which defines a mechanism $\mathcal{M}\left(p^{r}_{1:T},\Xi\right)$ 
which generates forecasts with controlled error level $\Xi$. This forecasting mechanism follows two key principles: (1) forecast accuracy degrades with longer horizons, $x-$step-ahead forecasts are less accurate than $x'-$step-ahead when $x>x'$. (2) Forecasts are updated at each step using the latest realized generation, ensuring high temporal corelation between $\boldsymbol{\hat{p}}^{r}_{t}$ and $\hat{p}^{r}_{t+1}$. 
% This simulation framework enables the RL agent to train across a diverse set of realistic outage conditions with varying DER availability and forecast uncertainty. 
% This setup ensures a realistic and dynamic interaction framework for training and evaluating the RL agent under uncertainty. 
% For details of $\mathcal{M}$, we refer the reader to \cite{zhang2022curriculum}. 
% Upon sufficient training, the RL agent is exposed to various outage conditions, enabling it to make informed control decisions.

At the beginning of each training episode (environment is reset via the reset function),
% the distribution network topology has been reconfigured and DERs synchronized. "Reconfiguration" here refers to a one-time offline switching action performed prior to the RL control phase, where network topology is adjusted to create a radial islanded configuration suitable for restoration. A
all DERs are initialized by black-start-capable units and ready for load pickup, with initial conditions, $S^{\theta}_0 = s_0$ (initial SOC), $E^{\mu}_0 = E^f$ (full fuel reserves), $\tilde{p}_0 = \mathbf{0}$ (no loads restored), and $t=1$. The topology remains fixed throughout the episode, allowing the RL agent to focus on DER dispatch and load restoration sequencing. This reset process enables the agent to experience diverse initial conditions across training episodes while maintaining consistency with operational practice where topology is established before iterative restoration begins. Upon sufficient training, the RL agent is exposed to various outage conditions, enabling it to make informed control decisions.

\emph{\textbf{Within-task RL Training}:} RL algorithms are generally categorized into \textbf{value-based methods} and \textbf{policy-based methods} for learning the optimal policy. Policy gradient methods, a subset of policy based approaches, updates policy parameters via gradient ascent as: 
% $\phi_{t+1}=\phi_{t}+\alpha {\nabla}_{\phi}F(\phi). $
\begin{equation}
\phi_{t+1}=\phi_{t}+\alpha {\nabla}_{\phi}F(\phi).  
\end{equation}
This necessitates the computation of the gradient $\nabla_{\phi}F(\phi)$. 
However, computing $\nabla_{\phi}F(\phi)$ can be challenging in practice due to two main reasons: (i) the environment dynamics are often unknown or non-differentiable, and (ii) estimating the gradient using samples can introduce high variance, slowing down convergence.
  % A significant challenge in policy based RL algorithm is the absence or inaccessibility of derivatives for the environment or policy functions. 
Most policy gradient-based method such as PPO \cite{PPO}, TRPO \cite{trpo}, and DDPG~\cite{ddpg} rely on the policy gradient theorem via \textit{back-propagation} (BP)~[\cite{sutton2018reinforcement}, Ch. 13] to estimate $\hat{\nabla}_{\phi}F(\phi)$. Despite their successes, these methods can trapped in local optima due to their reliance on local gradient information. To address these challenges, we adopt ES-RL, a \textit{gradient-free} optimization method that directly searches the parameter space,
% In contrast, we opt for a direct policy search algorithm based on evolution strategy (ES-RL), which performs optimization without relying on gradients. ES-RL employs a \textit{gradient-free}, optimization approach to estimate $\hat{\nabla}_{\phi}F(\phi)$, 
which offers scalability and global search capabilities.
% several advantages: (a) ES-RL does not require gradient information from the environment or the policy network, making it suitable for non-differentiable or complex environments; (b) The evaluation of multiple perturbations can be easily parallelized, leading to efficient utilization of computational resources.
ES-RL perturbs the policy parameters $\phi$ by adding random  Gaussian noise (i.e.,~$\boldsymbol{\varepsilon}=(\varepsilon_{1},\varepsilon_2,\dots, \varepsilon_{n})\sim \mathcal{N}(0,I)$),
with mean~$\bar{\phi}$ and fixed covariance $\sigma^2 I$. At each iteration, these perturbation generate $``n"$ new policies $\pi_{\phi_{t}+\sigma{\varepsilon}_{i}}$, which are then evaluated by interacting with the environment to gather experience and obtain stochastic returns $F_{i}=F(\phi_{t}+\sigma \varepsilon_{i})$. Thus, ES-RL aims at optimizing the Gaussian smoothed version of the expected reward, i.e.
% which can be expressed as:
$\mathbb{E}_{\phi\sim N(\bar{\phi},\sigma^2I)}F(\phi)=\mathbb{E}_{\boldsymbol{\varepsilon}\sim N(0,I)}\left[F(\bar{\phi}+\sigma \boldsymbol{\varepsilon})\right].$
Using the score function estimator, the gradient of smoothed version of original objective function with respect to~$\phi$ can be approximated as: $$\nabla_{\phi}\mathbb{E}_{\boldsymbol{\varepsilon}\sim N(0,I)}F(\bar{\phi}+\sigma\boldsymbol{\varepsilon})=\frac{1}{\sigma}\mathbb{E}_{\boldsymbol{\varepsilon}\sim N(0,I)}\left[\boldsymbol{\varepsilon} \cdot F(\bar{\phi}+\sigma\boldsymbol{\varepsilon})\right]$$This means the gradient can be estimated using \textit{zeroth-order} sampled rewards $F(\bar{\phi}+\sigma \boldsymbol{\varepsilon})$ without backpropagation. The policy parameters are then updated using gradient ascent:
\begin{equation}\nonumber
    \phi_{t+1}=\phi_{t}+\alpha\frac{1}{n\sigma }\sum_{i=1}^{n}F_{i}\varepsilon_{i}
\end{equation}

% The parameter $\sigma$ controls the exploration radius—larger values encourage global search but increase variance; smaller values reduce variance but may limit exploration. In practice, $\sigma$ is tuned to balance exploration and estimation accuracy.
The parameter $\sigma$ balances exploration and estimation accuracy: larger $\sigma$, accelerate convergence but may bias the perturbed function $\mathbb{E}_{\boldsymbol{\varepsilon}\sim N(0,I)}F(\phi+\sigma\boldsymbol{\varepsilon})$; smaller $\sigma$ reduce variance but may limit exploration. In our implementation, we use fixed based on preliminary tuning, though adaptive schemes exist~\cite{ES}. With sufficient iterations, $\phi_{t}$ is expected to converge to a near optimal policy.
% such as PPO \cite{PPO}, TRPO \cite{trpo}, and DDPG~\cite{ddpg}, that follow the gradient of the expected reward, which can get ensnared by local maxima, ES-RL leverages a global search heuristic due to its stochastic nature. It ES-RL has the ability to explore a diverse set of solutions, which increases the probability of escaping local optima.
 ES-RL offers several advantages: 
 
 \noindent(a) It does not require gradient information from the environment or the policy network, making it suitable for non-differentiable objective.
 
 \noindent(b) The evaluation of multiple perturbations can be easily parallelized, leading to efficient use of computational resources. 
 
 \noindent(c) By optimizing a smoothed objective, method can navigate highly non-convex landscapes and avoid poor local optima that typically trap gradient-based methods. 
 
 To demonstrate the efficiency of ES-RL, we compare it with PPO, DDPG, and DQN on benchmark mathematical optimization tasks. As shown in Table~\ref{table1} and Fig.~\ref{uniandbivariantcomparison}, ES-RL achieves better higher reward and converges significantly faster than gradient-based methods. These results support the use of ES-RL as an effective within-task optimization method in our framework.

 \begin{table}[ht]
\centering
\caption{ES-RL achieves the lowest runtime.}
\begin{tabular}{c c c c c c c c} 
 \hline
 \textbf{Algorithm} &ES-RL &PPO & DDPG&DQN & \\ 
 \hline
 
{Runtime}~~$f_{1}(\cdot)$  & 0.03s & 14.32s & 273.97s & 12.39s  &\\
 \hline
 {Runtime}~~$f_2(\cdot,\cdot)$ & 0.06s & 29.88s & 564.02s & 27.81s& \\
 \hline
\end{tabular}
\label{table1}
\end{table}
\begin{figure}[t]
    \centering
    {%  % Subfloat for IEEE compliance
        \includegraphics[width=0.235\textwidth]{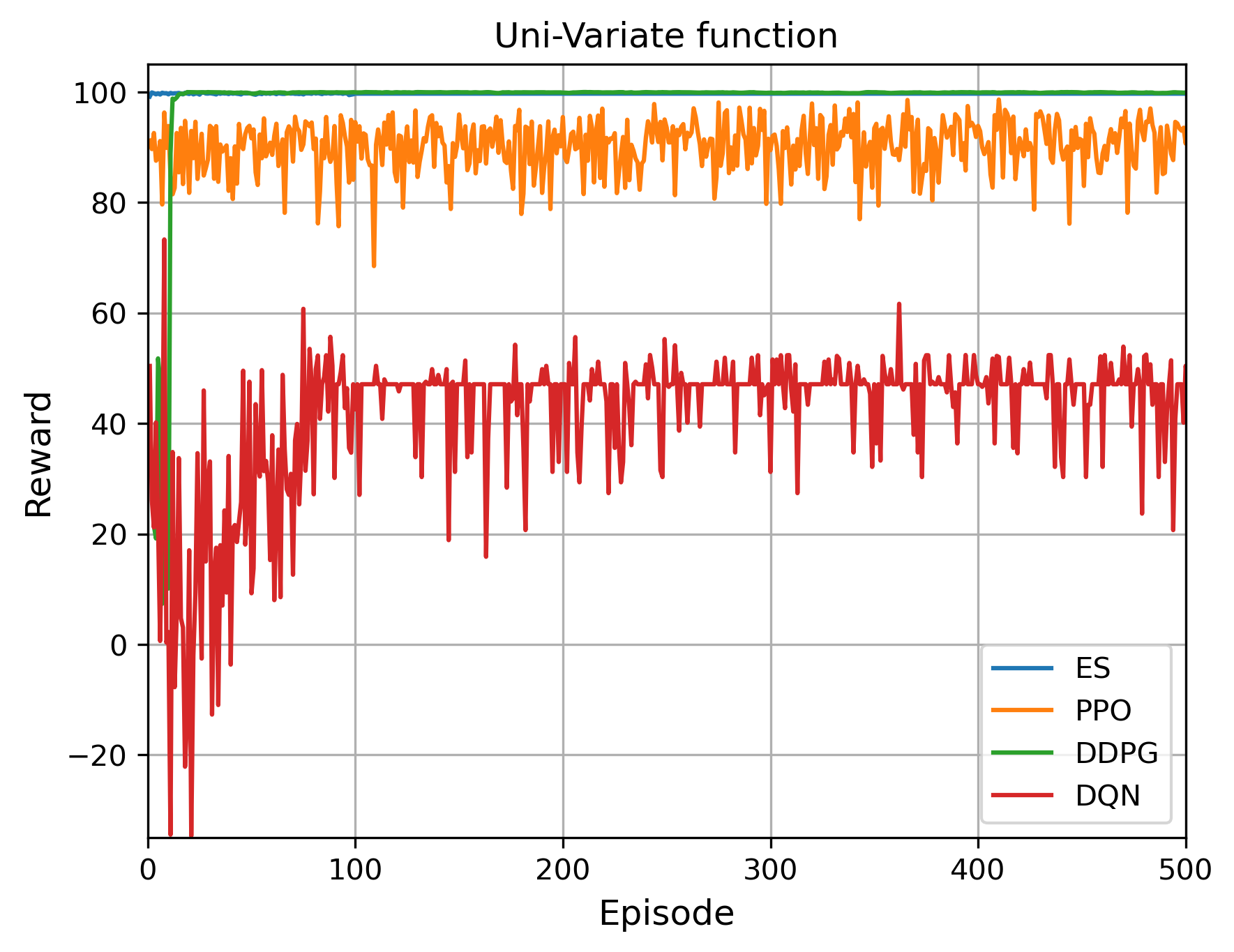}
        \label{fig:univariant}}
    \hfill
    {%
        \includegraphics[width=0.235\textwidth]{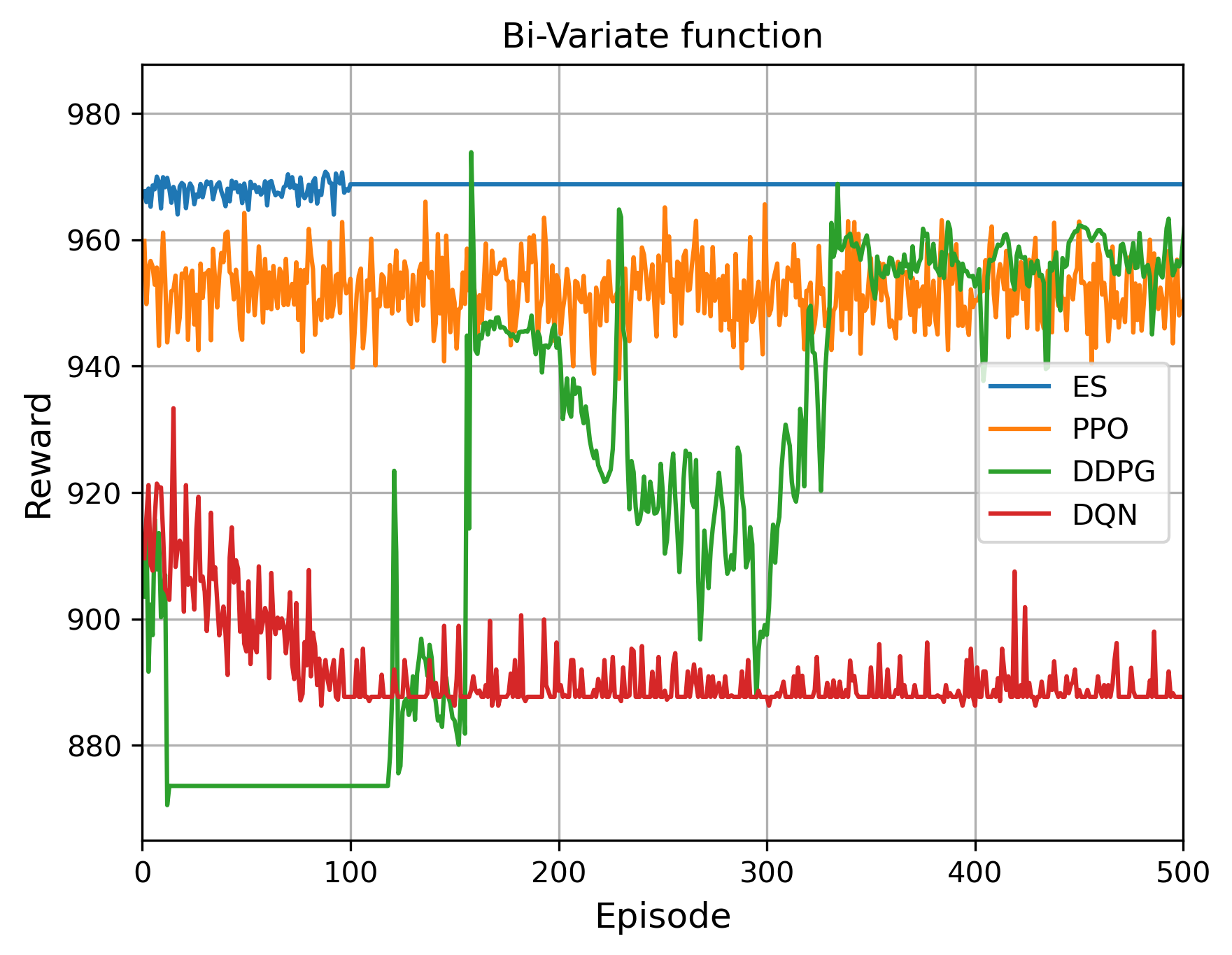}
        \label{fig:bivariate}}
    \caption{Cumulative reward curves for RL algorithms optimizing $f_{1}(x)=-x^{2}+10$ and 2D Ackley function $f_{2}(x,y)=100-20\exp(-0.2\sqrt{0.5(x^{2}+y^{2})})-\exp(0.5(cos(2\pi x)+cos(2\pi y )))+e+20$. The global search capability of ES-RL enables faster convergence and higher final rewards.}
    \label{uniandbivariantcomparison}
\end{figure}
\begin{algorithm}
{\small

\caption{MGF-RL}
\label{alg:meta-rl}

\DontPrintSemicolon  % Suppress semicolon at the end of lines for a cleaner look
\SetAlgoNoEnd
\SetAlgoNoLine
\SetKwFunction{FMain}{Train}
\SetKwFunction{FTesting}{Testing}
\SetKwProg{Fn}{Function}{:}{}
\KwIn{Meta training tasks $M$, training iterations $T$, learning rate for task-specific training $\alpha$ and meta-stepsize $\eta_{m}$.}
\KwOut{Meta-policy $\pi_M$ 
% with parameters $\hat{\phi}_{M,t}$ 
and optimal policy $\hat{\pi}_M$.}

\Fn{\FMain{}}{
Initialize random policy $\pi_{1,0}$ with parameters $\phi_{1,0}$\;

\For{each task $(m = 1, \dots, M)$}{
  Load initial policy with parameter $\phi_{m,0}$ for task $m$ \;
  \For{iteration $(t = 1, \dots, T)$}{
    Update policy parameters using ES-RL \;  % Assume ES-RL citation included in text
    Save the best model $\hat{\pi}_{m,t}$ with parameters $\hat{\phi}_{m,t}$ \;
  }
  \textbf{meta-update:} $\phi_{m+1,0} \leftarrow \phi_{m,0} + \eta_{m} (\hat{\phi}_{m,t} - \phi_{m,0})$ \;
}
Save meta-policy $\pi_M = \hat{\pi}_{M,t}$ \;
}

\Fn{\FTesting{}}{
% \textbf{For the Test Task:} \;
Load meta-policy $\pi_M$ \;
Fine-tune meta-policy at test time on the new unseen task to receive the optimal policy $\hat{\pi}_M$ \;
}
}
\end{algorithm}

\subsection{Proposed MGF-RL Framework}
The MGF-RL algorithm consists of two main phases: \textit{meta-training} and \textit{meta-testing}. The goal is to learn a meta-initialization for policy parameters that enables fast adaptation to new CLR tasks with minimal fine-tuning, thereby addressing the poor generalization problem inherent in standard RL approaches.

During meta-training, the agent is exposed to a set of $M$ training tasks sequentially, each representing a distinct restoration problem instance with different system parameters (e.g., load priorities, and DER forecast profiles). Critically, only one policy network is trained throughout the entire process—the same network is sequentially adapted across all $M$ tasks with its parameters serving as the initialization for the next task, enabling knowledge accumulation rather than training $M$ independent models. For each task $m$, the policy is first initialized with meta-parameters $\phi_{m,0}$ and optimized using ES-RL over~$T$ iterations. This within-task training enables the policy to adapt to task-specific characteristics. The meta-parameters are then updated based on aggregating improvements across tasks: 
% $\phi_{m+1,0} \leftarrow \phi_{m,0} + \eta_{m} (\hat{\phi}_{m,T} - \phi_{m,0})$,
$\phi_{m+1,0} \leftarrow \phi_{m,0} +  \eta_{m} (\hat{\phi}_{m,t} - \phi_{m,0})$, 
using first-order approximation of meta-gradient approach~\cite{Meta}, where meta-learning rate $\eta_m = 1/m$ decreases to balance exploration with stability. Through sequential application of this meta-update across all tasks, each $\phi_{m+1,0}$ incorporates adaptations from all previous tasks,
% the final meta-parameters $\phi_{M,0}$ encode weighted improvements from all $M$ training scenarios. This meta-update step allows to accumulates knowledge across tasks, 
guiding the initialization ${\phi}_{m+1,0}$ to a region of parameter space that generalizes well across CLR scenarios. The final meta-parameters $\phi_{M+1,0}=\hat{\phi}_{M,T}$ encode weighted improvements from all $M$ training scenarios.

In the meta-testing stage, the learned meta-parameters~$\hat{\phi}_{M,T}$ of policy $\pi_{M}$ serve as an initialization when faced with a new, unseen CLR problem, such as a novel outage configuration, unprecedented load priority pattern, or unfamiliar renewable generation profile. The policy is then fine-tuned using ES-RL for a small number of iterations, leveraging the knowledge gained during meta-training to rapidly adapt to the specific challenges of the new task. The complete MGF-RL procedure is outlined in Algorithm~\ref{alg:meta-rl}. The two-level structure of the MGF-RL, with ES-RL for within-task optimization combined with first-order meta-updates for across-task generalization, offers several synergistic benefits for CLR. It leverages ES-RL's resilience to local optima and ability to maintain solution diversity, while enabling fast adaptation to new tasks through meta-learned initialization. By employing first-order meta-gradients rather than second-order derivatives (as in MAML), MGF-RL avoids substantial computational burden while achieving comparable adaptation performance. Furthermore, the gradient-free ES approach naturally accommodates discrete control actions and non-differentiable constraints without requiring specialized constraint-handling mechanisms. This combination makes MGF-RL well-suited for tackling diverse CLR challenges.

% The key steps of the MGF-RL algorithm, as outlined in Algorithm \ref{alg:meta-rl}, provide a high-level overview of the process. The two-level structure of the MGF-RL algorithm, with ES-RL for within-task training and MGF-RL for across-task adaptation, offers several benefits. It allows the algorithm to leverage the strengths of ES-RL, such as its resilience to local optima and ability to maintain a diverse set of solutions, while also enabling fast adaptation to new tasks through the meta-learning process. This combination of robustness, adaptability, and generalization makes the proposed approach well-suited for tackling the challenges of the CLR problem in various scenarios.

%%%%%%%%%%%%%%%%%%%%%%%%%%%%%%%%%%%%%%%%%%%%%%%%%%%%%%%%%%%%%%%%%%%%%%

\section{Provable Guarantees for MGF-RL within-Online Framework}\label{theoretical}

This section provides theoretical analysis of the proposed MGF-RL algorithm. We begin by analyzing the convergence behavior of ES-RL algorithm within in a single-task setting. We then extend the analysis to meta-learning framework by introducing the notion of \textit{task-averaged regret}, 
 which leverage \textit{task-similarity} metrics to capture how different tasks are related and analyze its impact on convergence guarantees. 
Finally, we derive \textit{dynamic regret bounds} to quantify performance in non-stationary environments. We moved the detailed proofs in Appendix \ref{appendixtheoretical}.

\subsection{Within-Task Analysis}
We consider a sequence of MDPs, indexed by $m=1,\dots,M$, each task corresponds to distinct MDP. Within each task $m$, the agent applies ES-RL algorithm for $T$ iterations to refine the policy parameter $\left\lbrace \phi_{m,j}\right\rbrace_{j=0}^{T}$ , resulting in a suboptimal policy with parameters $\hat{\phi}_{m,T}$. Following Theorem \ref{esrlconvergence} provides convergence behavior of ES-RL used as within-task optimizer, and forms the basis for the meta-learning analysis in Section \eqref{section4b}.
\begin{theorem}[Theorem 6; \cite{analysis}]\label{esrlconvergence}
    Suppose ES-RL perform $T= \frac{4(d_{\phi}+4)^2 (L_{F})^2 }{\epsilon^2}$ iterations 
    % with learning rate $\alpha_{m}$=$\frac{R}{(N+4)(T+1)^{1/2}(L_{0}(F_{m}))}$ 
    to optimize per-task objective function $F_m:\mathbb{R}^{d_{\phi}}\to \mathbb{R}$, where $d_{\phi}:=\dim(\phi) $ be the number of trainable policy parameters. For each task $m$,  with learning rate $\alpha_{t}$=$\frac{1}{(d_{\phi}+4)(T+1)^{1/2}(L_{F})}$ and choose $\sigma\leq \frac{{\epsilon}}{2L_{F}d_{\phi}^{1/2}}$. Then for any $\epsilon>0$ the sub-optimality gap for each task $m$ satisfies:
%     \begin{equation}\label{esconv}
% R= \mathbb{E}\left[F_{m}\left(\hat{\phi}_{m,T}\right)\right]-F_{m}\left({\phi}^{*}_{m}\right)\leq \frac{2(N+4)L\|\phi^{*}_{m}-\phi_{m,0}\|}{\sqrt{T}},
%     \end{equation}}
    \begin{equation}\label{esconv}
\mathbb{E}\left[F_{m}\left(\hat{\phi}_{m,T}\right)\right]-F_{m}\left({\phi}^{*}_{m}\right)\leq \frac{2(N+4)L_{F} D_{KL}\left(\phi^{*}_{m}|\phi_{m,0}\right)}{\sqrt{T}},
    \end{equation}
where $\phi^{*}_m$ are the parameters of optimal policy $\pi^{*}_{m}$ and $L_{F}$ is the Lipschitz constant of function $F_{m}(\cdot)$. 
% $R$ being the bound of $\|\phi^{*}_{m}-\phi_{m,0}\|\leq R$. 
\end{theorem}
 
% \subsection{Task-Averaged Regret and Task Similarity}
We can observe from \eqref{esconv}, the regret bound depends on distance $D_{KL}\left(\phi^{*}_m|\phi_{m,0}\right)$, indicating that better policy  initialization parameters $\left\lbrace\phi_{m,0}\right\rbrace _{m=1}^{M}$ improves convergence, motivating the use of meta-learned parameter across tasks.  % Beyond the single task, we consider the lifelong extension of MDPs in which we have sequence of online learning problems $m=1,2,\dots,M$. In each single task $m$, the agent must sequentially optimize the policy $\left\lbrace\pi_{m,t}\right\rbrace _{t=1}^{T}$, so that the corresponding sub-optimality decay sub-linearly in $T$, as give in \eqref{esconv}. 
MGF-RL algorithm sequentially update these initial parameters~$\phi_{m,0}$ before running within-task ES-RL, with aim to minimize the \textit{task average optimality gap~(TAOG)} after $M$ tasks and defined as:
\begin{equation}
\begin{split}
    \bar{R}=\frac{1}{M}\sum_{m=1}^{M}&\mathbb{E}\left[F_m(\hat{\phi}_{m,T})\right] -F_m(\phi_{m}^{*})\\&\leq \frac{2(d_{\phi}+4)L_{F} \sum_{m=1}^{M}D_{KL}\left(\phi^{*}_{m}|\phi_{m,0}\right)}{M\sqrt{T}}
\end{split}
\end{equation}
The TAOG reflects how well the final policy with parameters $\hat{\phi}_{m,T}$ generalizes to unseen tasks and is influenced by task similarity. For a fixed initialization $\phi$,
% In meta-RL, the extent to which TAOG improves is influenced by the similarity among the sequential MDP tasks. 
% We now discuss the notion of similarity in an environment.
the \textbf{task similarity} is defined by $D_{*}=\min_{\phi\in \Delta(\mathcal{A})^{|\mathcal{S}|}}\frac{1}{M}\sum\limits_{m=1}^{M}D_{KL}\left(\phi^{*}_{m}|\phi\right)$, in a static environment. 
\subsection{Provable Guarantees for Practical MGF-RL}\label{section4b}
% In the previous section we present the analysis with access to best actions in hindsight $\phi^{*}_{m}$ for each task that can learn a good meta initialization or meta regularization. While $\phi^{*}_{m}$ is efficiently computable in some cases, often it is more practical to use approximation. one of the key steps to generalize the online-within-online methodology to relax the assumption of access the exact upper bounds of within-task performance by designing analysis to estimate and update their inexact versions.
In practical scenario, the meta-learner observes only the suboptimal policies $\hat{\phi}_{m}$ obtained from within-task ES-RL. Consequently, we estimate $D_{KL}\left({\phi^{*}_{m}|\phi}\right)$ with $D_{KL}\left({\hat{\phi}_{m}|\phi}\right)$. The approximation error induced by this substitution is bounded as: $D_{KL}\left({\phi^{*}_{m}|\phi}\right)-D_{KL}\left({\hat{\phi}_{m}|\phi}\right) \leq \epsilon_{m}$ in [Lemma 17, \cite{khattar2022cmdp}], thus the \textbf{empirical task similarity} defined as $\hat{D}_{*}=\min_{\phi\in \Delta(\mathcal{A})^{|\mathcal{S}|}}\frac{1}{M}\sum\limits_{m=1}^{M}D_{KL}\left(\hat{\phi}_{m}|\phi\right)$.
% , we denote the \textbf{empirical task similarity} defined in \cite{khattar2022cmdp} as $\hat{D}_{*}=\min_{\phi\in \Delta(\mathcal{A})^{|\mathcal{S}|}}\frac{1}{M}\sum\limits_{m=1}^{M}D_{KL}\left(\hat{\phi}_{m}|\phi\right)$.
% % which depends on the suboptimal policies returned by a within-task algorithm.
% The following section illustrate how task similarity influences TAOG.
% Once a task is complete, the meta learner only has access to sub-optimal policy $\hat{\pi}$. Thereby we estimate $D_{KL}\left({\phi^{*}_{m}|\phi}\right)$ with $D_{KL}\left({\hat{\phi}_{m}|\phi}\right)$ by plugging in $\hat{\phi}$ from within task MDP. The KL divergence estimation error bounded as [Lemma 17; \cite{khattar2022cmdp}]: 
% \begin{equation}\label{KLBound}
%     D_{KL}\left({\phi^{*}_{m}|\phi}\right)-D_{KL}\left({\hat{\phi}_{m}|\phi}\right) \leq \mathcal{O}\left(h\left(\frac{1}{\sqrt{T}}\right)+\frac{1}{\sqrt{T}}\right)=\epsilon_{m},
% \end{equation}
% where $h$ is strictly increasing continuous function with $h(0)=0$. 
We define the cumulative inexactness $\boldsymbol{\epsilon}_{M}:=\sum_{m=1}^{M}\epsilon_{m}$. 
% This quantity quantity decays with $T$ at the rate of $\mathcal{O}\left(h\left(\frac{1}{\sqrt{T}}\right)+\frac{1}{\sqrt{T}}\right)$ up to some approximation. 
\subsubsection{Static Regret Bounds}
We now analyze the impact of this approximation on the meta-update, where $\left\lbrace\phi_{m,0}\right\rbrace_{m=1}^{M}$ are updated via implicit online mirror descent (IOMD) or follow the regularized leader (FTRL) on the loss $\hat{l}_{m}\left(\phi\right)=D_{KL}\left(\hat{\phi}_{m}|\phi\right)$.
\begin{theorem}[TAOG for static environment]\label{TAOGstatic}
Let $\phi^{*}$ be the fixed meta initialization for all the tasks given by $\phi^{*}= \arg\min_{\phi\in \Delta(\mathcal{A})^{|\mathcal{S}|}}\sum_{m=1}^{M}\frac{D_{KL}\left(\hat{\phi}_{m}|\phi\right)}{M}$.
Within each task we run ES-RL for $T$ iterations to obtain $\left\lbrace\hat{\phi}_{m,0}\right\rbrace_{m=1}^{M}$. 
Let $B^{2}_{*}:= \max_{a,b\in X} B_{\Psi}(a,b)$, where \( B_{\Psi} \) denote the Bregman divergence with respect to a 1-strongly convex function \( \Psi: \text{dom}(\Psi) \to \mathbb{R} \) and the initialization $\left\lbrace \phi_{m,0}
\right\rbrace_{m=1}^{M}$ are updated by IOMD or FTRL with fixed learning rate $\eta$, then TAOG satisfies 
% \begin{equation}
%     \bar{R}\leq \mathcal{O}\left(\frac{1}{\sqrt{T}}\left(\hat{D}_{*}+\frac{V_{M}}{M}+\frac{B^{2}_{*}}{M}+\frac{\epsilon_{M}}{M}\right)\right).
% \end{equation}
\begin{equation}\nonumber
    \bar{R}\leq 2(d_{\phi}+4)L_{F}\left(\frac{B^{2}_{*}}{\eta M\sqrt{T}}+\frac{ V_{M}}{ M\sqrt{T}} +\frac{3\hat{D}_{*}}{\sqrt{T}}+\frac{\boldsymbol{\epsilon}_{M}}{M\sqrt{T}}\right),
\end{equation}
where, \( V_M(l_m) = \sum_{m=2}^{M} \max_{\phi_{m,0} \in \text{Dom}(l)} |l_m(\phi_{m,0}) - l_m(\phi_{m-1,0})| \) is the temporal variability of the loss sequence.
\end{theorem}
% \textbf{Proof.} We moved the detailed proof to appendix \ref{b32}.\\
\begin{remark}Theorem \ref{TAOGstatic} establishes a regret bound that preserves the benefits of task similarity within meta-reinforcement learning. However, there is a trade-off as the bound introduces an additional cost term associated with inexactness, denoted by $\epsilon_{M}$. Notably $\epsilon_{M}/M$ decreases, either by increasing the number of iterations $T$ within each task or by optimizing the meta-initialization $\phi^{*}$ to require minimal task-specific adaptation, the inexactness term's influence on overall performance is expected to be reduced. This suggests that improving task similarity or selecting a well-suited initialization can effectively mitigate the additional cost from $\epsilon_{M}$.

   % Theorem \ref{TAOGstatic} results reflects the trade-off between estimation accuracy and generalization. Increasing $T$ and $M$ reduces both the approximation error $\boldsymbol{\epsilon}_{M}$ and the impact of task-similarity $\hat{D}_{*}$, while lower variability $V_{M}$ improves the stability of meta updates.
   % or improving task similarity (i.e., lower $\hat{D}_{*}$) can reduce the overall regret. However, the bound also include an additional term $\boldsymbol{\epsilon}_{M}$, which accounts for the error from using approximate policies. Notably, $\epsilon_{M}/M$ decreases with larger $T$ or when the meta initialization $\phi^{*}$ is well-aligned with task optima, suggesting that better initialization can help mitigate this cost.
\end{remark}
% Theorem \ref{TAOGstatic} establishes a regret bound that preserves the benefits of task similarity within meta-reinforcement learning. However, there is a trade-off as the bound introduces an additional cost term associated with inexactness, denoted by $\epsilon_{M}$, and a dependency on the number of tasks $M$ that is less favorable than the regret bound presented in equation \eqref{eq13}. Notably $\epsilon_{M}/M$ decreases, either by increasing the number of iterations $T$ within each task or by optimizing the meta-initialization $\phi^{*}$ to require minimal task-specific adaptation, the inexactness term's influence on overall performance is expected to be reduced. This suggests that improving task alignment or selecting a well-suited initialization can effectively mitigate the additional cost from $\epsilon_{M}$.
%%%%%%%%

% In many setting, we have a changing environment, so it is natural ti study dynamic regret and compare with a sequence of potential time-varying initial policies $\left\lbrace\pi^{*}_{m,0}\right\rbrace_{m=1}^{M}$ with parameters $\left\lbrace\phi^{*}_{m,0}\right\rbrace_{m=1}^{M}$. To measure task similarity in this case, we define task-relatedness which can be measured as . This notion of task-reltaedness gives the measure of how far optimal 
\subsubsection{Dynamic Regret Analysis}
In many real scenarios, we encounter dynamic environments, making it natural to study dynamic regret and to compare performance against a sequence of time-varying initial policies $\left\lbrace\pi^{*}_{m,0}\right\rbrace_{m=1}^{M}$ with parameters $\left\lbrace\varphi^{*}_{m,0}\right\rbrace_{m=1}^{M}$. In dynamic environments, the task distribution may change over time, requiring the meta-learner to adapt to a time-varying sequence of comparators $\left\lbrace\varphi^{*}_{m}\right\rbrace_{m=1}^{M}$. The performance of MGF-RL in this setting is measured by the \textit{dynamic regret}: $L_{M}\left(\varphi^{*}_{1:M}\right):=\sum_{m=1}^M l_m(\phi_{m,0}) - \sum_{m=1}^M l_m(\varphi_{m}^*),$ where $\varphi_{m}^* \in \arg\min_{\varphi\in \Delta(\mathcal{A})} l_m(\varphi)$ represents the optimal initial policy for each task $m$, and the function $l_m(\cdot) = \mathbb{E} \left[ D_{\mathrm{KL}}(\phi_m^* | \cdot) \right]$ quantifies the online loss.
To quantify how dynamically changing environment impact the learning process we use two standard measure \textbf{path length} \cite{zhao2020dynamic}, $P_M := \sum_{m=2}^M \|\phi_{m,0}^* - \phi_{m-1}^*\|$, and \textbf{temporal variability} of the loss function, \cite{besbes2015non}. These quantities reflect the degree of change in the comparator sequence and the task losses, respectively.

\begin{theorem}[Dynamic regret for online learning]\label{dynamic regret}
     For any dynamical varying comparator $\varphi^{*}_{m}$. Let \( g_m' \in \partial l_m(\phi_{m+1,0}) \) be a subgradient of the loss at \( \phi_{m+1,0} \),and if $\exists$ $\gamma \in \mathbb{R}$ such that $B_{\Psi}(x,z)-B_{\Psi}(y,z)\leq \gamma\|x-y\|,$ and IOMD or FTRL applied with learning rates \( \{\eta_m\}_{m=1}^{M} \), the dynamic regret is bounded by:
\begin{equation}\nonumber
    \begin{split}
L_{M}\leq \sum_{m=1}^{M}\frac{\|\varphi^*_{m}-\varphi^*_{m-1}\|}{\eta_{m}}+\frac{B^2_{*}}{\eta_{m}}+\sum_{m=1}^{M}2\eta_{m}\|g_{m}\|_{*}\|g'_{m}\|_{*}
    \end{split}
\end{equation}
\end{theorem}
% \textbf{Proof.} We moved the detailed proof to appendix \ref{b41}.\\

% We notice that the Lipschitz continuity assumption is not a strong requirement. Indeed, when the function $\Psi$ is Lipschitz, the Lipschitz condition on Bregman divergence is automatically satisfied. It can be observed from last theorem that to set the learning rate fixed makes the algorithm less applicable in online settings where tasks are encountered sequentially. Moreover when the task-environment changes dynamically, a fixed initialization may not be the best candidate comparator, where it is natural to study dynamic regret by competing with a potential time varying sequences $\left\lbrace\varphi^*_m\right\rbrace_{m=1}^{M}$. Also, the tasks may share some common aspects of the optimization landscape,
% so adapting learning rates based on prior experience may further improve performance. 
\begin{theorem}[Regret with adaptive learning rate]\label{adaptivelr}
  Under the assumption of Theorem \ref{dynamic regret} and path-length satisfies $P_{M}(\varphi^{*}_{m})\leq \mathcal{C}_{1}$ with $\mathcal{C}_{1}\geq 0$. Then with decreasing rate $\eta_{m}=\frac{1}{\lambda_{m}}=\frac{\beta^{2}}{\sum_{m=1}^{M}\delta_{m}}$ and $\beta^{2}=\left(D^2_{b}+\gamma {P}_{M}\right)$ incur the dynamic regret bounded as 
 \begin{equation}\nonumber
     L_{M}\left(\varphi^{*}_{1:M}\right) \leq  \mathcal{O} \left(\min\left\lbrace V_M, 2\sqrt{3B^{2}_{*}+\gamma \mathcal{C}_{1} \sum_{m=1}^{M}\|g_{m}\|^2_{*}}\right\rbrace\right)
 \end{equation}

\end{theorem}
% \textbf{Proof.}We moved the detailed proof to  appendix \ref{B42}.\\

\begin{remark}
    If we assume $\|g_{m}\|_{*}^{2}\leq \max_{m\in[M]}\|g_{m}\|_{*}^{2}\leq 1$ and $\gamma=B_{*}=1$, then the above result give us the dynamic regret bound of $\mathcal{O} \left( \min \left\lbrace V_{M},\sqrt{M(1+\mathcal{C}_{1})} \right\rbrace \right)$. 
    This bound is tight for sequences whose path-length $P_{M}=\mathcal{C}_{1}$, matching the lower bounds for both the path-length and temporal variability.
\end{remark}
%%%%%%%%%%%
\begin{theorem}[TAOG for dynamic environment]\label{TAOGdynamic}
  Let $\left\lbrace\phi_{m,0}\right\rbrace_{m=0}^{M}$ %be the initialization 
  be determined by %implicit online mirror descent or 
   follow the average leader, %on $\mathbb{E}_{\hat{v}_{m}}\left[D_{KL}\left(\hat{\pi}|\cdot\right)\right]$
    and for each task we train the policy %gradient algorithm 
   for $T$ steps with learning rate $\alpha$ and obtain $\left\lbrace\hat{\phi}_{m,T}\right\rbrace_{m=1}^{M}$. If $\varphi_{m}^*$ is the optimal meta initialization for each task, then the TAOG is bounded as      
% \begin{equation}\label{TAOG}
%     \frac{1}{M}\sum_{m=1}^{M}\mathbb{E}\left[F_m(\hat{\phi}_{m,T})\right] -F_m(\phi_{m}^{*})\leq \mathcal{O}\left({\frac{V_{M}+D^*}{\sqrt{T}M}}\right).
% \end{equation} 
\begin{equation}\label{TAOG}
    \bar{R}\leq \mathcal{O}\left(\min \left\lbrace V_{M}, \sqrt{M(1+{P}_{M})}\right\rbrace+\frac{\hat{S}_{*}}{\sqrt{T}}+\frac{\boldsymbol{\epsilon}_{M}}{M\sqrt{T}}\right).
\end{equation}
where, $\hat{S}_{*} := \frac{1}{M}\sum\limits_{m=1}^{M}D_{KL}\left(\hat{\phi}_{m}|\varphi^{*}_{m}\right)$ is task-relatedness.
% \begin{equation}\nonumber
%     \frac{1}{M}\sum_{m=1}^{M}\mathbb{E}\left[F_m(\hat{\phi}_{m,T})\right] -F_m(\phi_{m}^{*})\leq \frac{10{D}^{*}}{\alpha T}+\frac{4V_{M}}{\alpha M T}+\frac{4\alpha |\mathcal{S}||\mathcal{A}|}{(1-\gamma)^{3}}.
% \end{equation}
% By tuning the learning rate $\alpha=\sqrt{\frac{\left(1-\gamma\right)^3}{4|\mathcal{S}||\mathcal{A}|}\left(10D^{*}+\frac{4V_{M}}{M}\right)}$, we get the tighter bound.
% \begin{equation}\label{optgap}
%     \frac{1}{M}\sum_{m=1}^{M}F(\pi_{m}^{*})-\mathbb{E}\left[F(\hat{\pi}_{m})\right] \leq \mathcal{O}\left(\sqrt{\frac{V_{M}}{\sqrt{TM}}+\frac{D^{*}}{T}}\right)
% \end{equation}
\end{theorem}
% \textbf{Proof.} We moved the detailed proof to  appendix \ref{B51}.\\
\begin{remark}
    % The bound in \eqref{TAOG} captures the interplay between temporal variability $V_{M}$, task relatedness $\hat{S}^*$
 % , the number of per-task updates $T$, and the total number of task $M$. 
 The bound in \eqref{TAOG} highlights smaller path length $P_{M}$, temporal variability $V_{M}$ and task relatedness $\hat{S}_{*}$, lead to tighter regret bounds, enabling better generalization. Adaptive learning rates help account for non-stationarity, and the regret decreases sublinearly in both the number of tasks $M$ and within-task iteration $T$. 
\end{remark}
   % The task-averaged regret upper bound \eqref{TAOG} is sensitive to temporal variability $V_{M}$. Specifically, a lower $V_{M}$  results in a tighter bound, indicating the algorithm performs better in environments with stable, less variable tasks. A larger $\hat{S}^*$ loosens the upper bound, implying that as tasks become more dissimilar, the algorithm may become less effective at generalizing across these tasks.% Thus, the algorithm is expected to perform better in scenarios where tasks share underlying similarities, which could be particularly relevant for applications where tasks are variations of a core problem.
% The terms $T$ and $M$ in the denominator suggest that increasing the number of iterations per task $T$ or the total number of tasks $M$ could lead to a reduced regret. However, the square root indicates a sub-linear rate.

\section{Case Study} \label{casestudy}
\subsection{Environment Description and Experiment Setup}
To evaluate the adaptability and effectiveness of the proposed MGF-RL method, experiments are conducted on modified IEEE-13 and IEEE-123 bus distribution systems. As illustrated in Fig.~\ref{bus_sys}, the IEEE-13 bus system includes 15 critical loads distributed across a multi-phase network and integrates four DERs: a battery energy storage, solar photovoltaic, wind turbine, and microturbine. 
% with a maximum charging and discharging capacity of $-P^{\theta,ch} = P^{\theta,dis} = 250$ kW, a solar photovoltaic (PV) system with a capacity of $p^{\rho} = 300$ kW, a wind turbine (WT) with a capacity of $p^{\omega} = 400$ kW, and a micro-turbine capable of providing 1200 kWh of energy through fuel reserves 
The detailed operational constraints of these DERs are summarized in Table~\ref{tableparam}. The load restoration process operates over a six-hour control horizon with a control interval of five minutes ($\tau = 1/12$), resulting in a total of 72 time steps~($|\mathcal{T}| = 72$). Each load $i \in \mathcal{L}$ is assigned a priority weight ${\varsigma}^{i}\in [0.2, 1.0] $ to reflect its criticality, with specific values listed in Table~\ref{tableparam}. The system enforces voltage limits between  0.95 and 1.05 p.u., with violations penalized by cost parameter $\lambda = 10^{8}$.

To train and evaluate the adaptability of the proposed meta-based RL method, a set of 60 distinct load restoration tasks is generated by varying the base load demand profile ($\bf{p}_t$), representing diverse grid operating conditions. The active power (kW) demands of the loads across all 60 tasks are illustrated in Fig.~\ref{baseload}, showing substantial variation in load magnitudes (ranging from 20 to 160 kW per load) and distribution patterns across different network locations. Among these 60 tasks, 32 tasks are used for meta-training, while the remaining unseen tasks are randomly selected for testing. The renewable generation profiles are constructed from historical data in the NREL WIND Toolkit \cite{draxl2015wind}, with forecast errors ranging from 0\% to 25\% to evaluate robustness under prediction uncertainty (analyzed in Section~\ref{secforecasterror}).
% The proposed method leverages the training tasks to learn meta-parameters, allowing the agent to quickly adapt to new tasks during the testing phase.
% The active power (kW) demands of the loads in all tasks are illustrated in Fig.~\ref{baseload}. 
% The proposed method leverages the 32 training tasks to learn meta parameters, enabling the agent to quickly adapt to unseen tasks during testing.
%%%%%%%%%%%%%%%%%%%%%%%%%%%%%%%%%%%%%%%%%%%%%%%%
\begin{figure}
\centering
\includegraphics[width=0.82\linewidth]{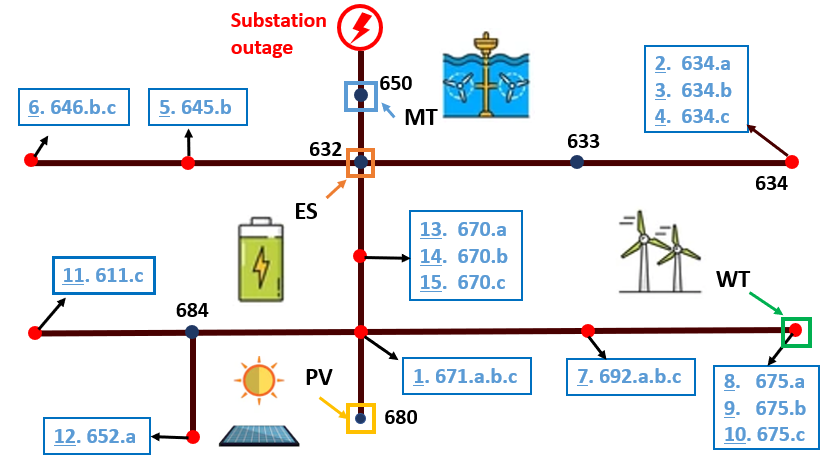}
\caption{Modified IEEE-13 bus system}
\label{bus_sys}
\end{figure}

%%%%%%%%%%%%%%%%%%%%%%%%%%%%%%%%%%%%%%%%%%%%%%%%%
\begin{table}[ht]
\begin{center}
\caption{DER Characteristics and Operational Constraints}
\begin{tabular}{c c } 
 \hline\hline
 \textbf{DERs} & \textbf{Parameters} \\ [0.5ex] 
 \hline\hline
 Energy Storage (ST) & $-P^{\theta,ch}=P^{\theta,dis}=250$\\&$160\leq S^{\theta}_{t}\leq 1250,~~~\alpha^{\theta}\in \left[0.\pi/4\right]$ \\ 
 \hline
 Micro-Turbine (MT) & $p^{\mu}\in [0,400],\quad E^{\mu}=1200$\\&$\alpha^{\mu}\in \left[0,\pi/4\right]$  \\
 \hline
 Photovoltaic (PV) & $p^{\rho}\in\left[0,300\right],\quad \alpha^{\rho}\in \left[0,\pi/4\right]$  \\
 \hline
 Wind Energy (WT) & $p^{\omega}\in\left[0,300\right],\quad \alpha^{\omega}\in \left[0,\pi/4\right]$ \\
 \hline
 Priority factor & $\boldsymbol{\varsigma}$=[1.0, 1.0, 0.9, 0.85, 0.8,0.8, 0.75,0.7, \\&~~\quad~~~~0.65, 0.5, 0.45, 0.4, 0.3, 0.3, 0.2]  \\ 
 \hline
\end{tabular}
\label{tableparam}
\end{center}
\end{table}
\begin{figure}[t]
\centering
\includegraphics[width=0.7\linewidth]{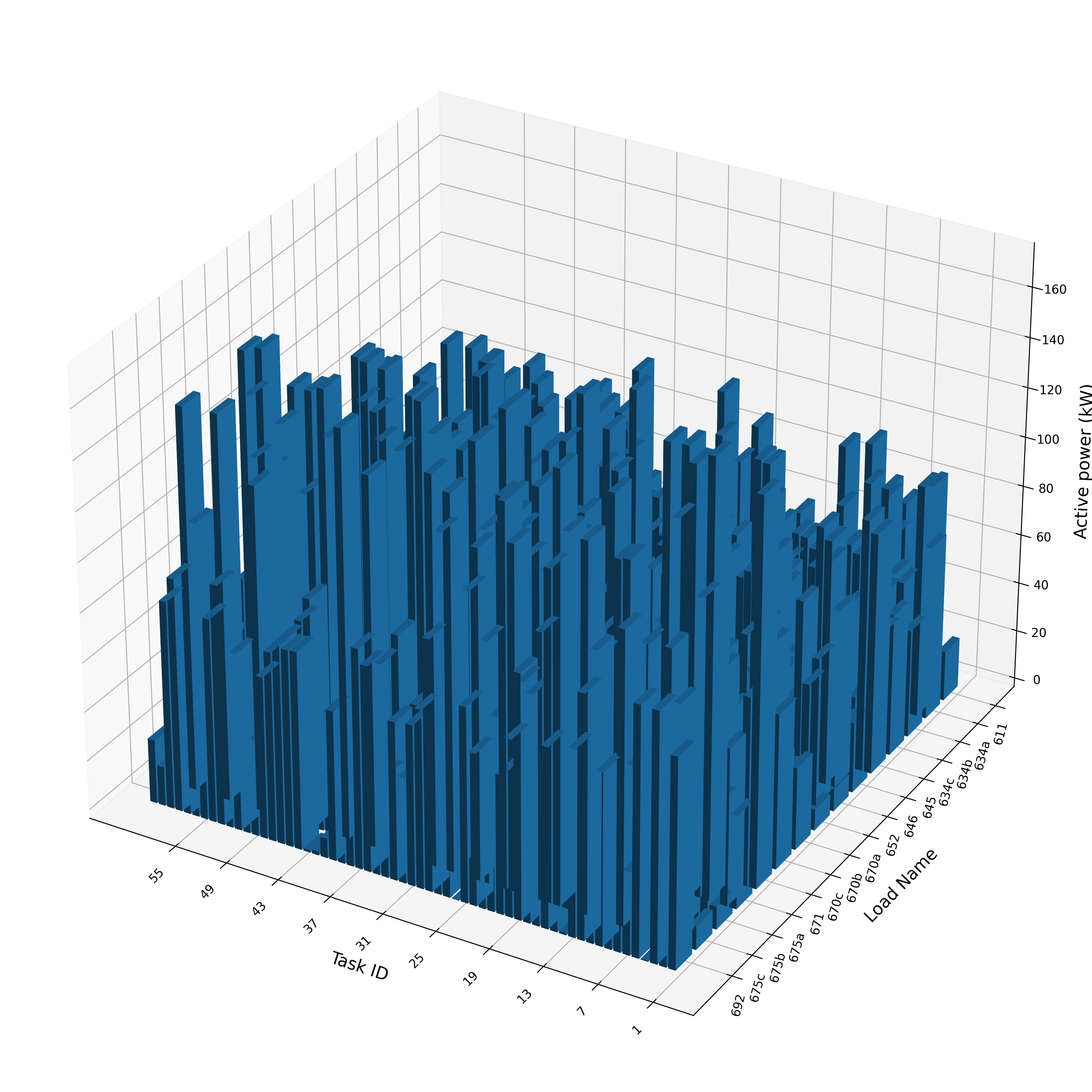}
\caption{Active power demand profiles (kW) across 60 restoration tasks. The x-axis represents Task ID, the y-axis shows load identifiers (node and phase, e.g., `675c' indicates node 675, phase c), and the z-axis indicates active power demand. Loads without phase labels are balanced three-phase loads. The substantial variation demonstrates task diversity for meta-learning evaluation.
}
\label{baseload}
\end{figure}
\subsection{Performance Comparison of Generalized Policy } 
We evaluate the efficacy of the proposed meta-RL Algorithm~\ref{alg:meta-rl} against several other strategies: ES-RL (a non-meta, task specific version trained from scratch on each task), a warm-start RL that transfers the final policy from one task as initialization to the next~\cite{ul2024enhancing}, MAML-RL using TRPO for within-task optimization \cite{finn}, and automated curriculum based RL (AC-RL), which applies supervised transfer from a simplified training setup, for detail AC-RL, (see \ref{ACL} for details). All methods are trained and tested under the same task distributions and resource models to ensure fair comparison.

Performance is evaluated using a set of metrics designed to reflect both learning efficiency and restoration quality. These include: (i) \textit{Mean cumulative reward} during fine-tuning on test tasks indicating overall policy quality; (ii) \textit{Jump-start performance} $\left(\Delta_{init}\right)$ defines as initial performance gap relative to ES-RL baseline, measuring the immediate adaption capability without fine-tuning; (iii) \textit{asymptotic reward gain} $\left(\Delta R\right)$ representing final converged reward improvement over ES-RL after fine-tuning, indicating long-term adaptation quality. Additionally, we report system-level reliability indicator such as \textit{SAIDI (System Average Interruption Duration Index)}, which measure the average outage duration per load, and \textit{restoration time}, defined as the elapsed time to restore a given percentage of total system load, measuring restoration speed.

\begin{figure}
\centering
\includegraphics[width=6.9cm]{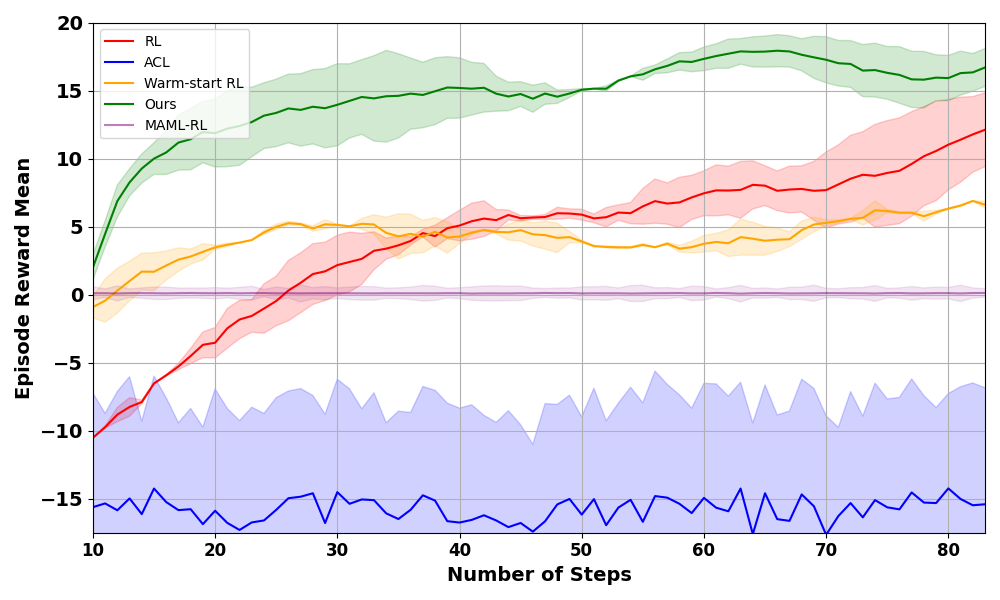}
\caption{Learning curves showing mean and variance of episode rewards over 5 runs, with MGF-RL achieving higher and stable performance across tasks. }
\label{fig1}
\end{figure}
Figure~\ref{fig1} presents the episode reward trajectory for each method during fine-tuning. MGF-RL consistently demonstrates higher reward curves with low variance across runs, indicating both effective generalization of training and reliable adaptation to unseen tasks. The meta-learned initialization provides strong jump-start performance, allowing MGF-RL to begin with near-optimal behavior and continue improving throughout fine-tuning. 
Warm-start RL benefits from prior policy transfer and shows a better starting point than ES-RL, but converges to a lower reward than MGF-RL. ES-RL steadily improves, but remains below MGF-RL throughout. AC-RL exhibits high variance across runs and under-performs on average, suggesting sensitivity to task conditions and instability. MAML-RL shows stable behavior but limited reward gain, likely due to conservative gradient-based updates that slow its adaptation process.

\begin{table}[ht]
\caption{MGF-RL consistently achieves positive values in both $\Delta_{init}$ and $\Delta R$ metrics, demonstrating superior initial adaptation and sustained improvement.}
\label{metric}
\centering
\begin{adjustbox}{width=\columnwidth}
\begin{tabular}{c c c c c c c c c} 
 \toprule
 \textbf{Ids} & \multicolumn{2}{c}{\textbf{MGF-RL(Ours)}}& \multicolumn{2}{c}{\textbf{Warm-Start-RL}}& \multicolumn{2}{c}{\textbf{AC-RL}}& \multicolumn{2}{c}{\textbf{MAML-RL}} \\
 \cmidrule(lr){2-3} \cmidrule(lr){4-5} \cmidrule(lr){6-7} \cmidrule(lr){8-9}
 & ${\Delta_{init}}$  & ${\Delta{R}}$ & ${\Delta_{init}}$ & ${\Delta{R}}$ & ${\Delta_{init}}$ & ${\Delta{R}}$ & ${\Delta_{init}}$ & ${\Delta{R}}$ \\
 \midrule
34 & \bf{9.95} & \bf{1.35} & -532.22 & 0.92 & -10.54 & -19.73 & 0.84 & -16.07 \\
35 & \bf{1.63} & \bf{0.21} & 1.30 & 0.12 & -9.14 & -7.59 & -4.99 & -12.06 \\
36 & \bf{279.37} & \bf{6.19} & -253.81 & 1.82 & -123.27 & -19.21 & -2.15 & -11.20 \\
37 & -69.23 & 10.05 & -321.31 & \bf{10.70} & 539.05 & -20.32 & \bf{698.75} & -16.25 \\
38 & \bf{5.43} & \bf{2.4} & 2.65 & 1.35 & -17.86 & -20.57 & -11.86 & -15.57 \\
39 & \bf{4.6} & \bf{10.41} & 2.23 & 4.09 & 0.15 & -18.40 & -4.15 & -15.40 \\
\bottomrule
\end{tabular}
\end{adjustbox}
\end{table}

To further examine task-level adaptation, we analyze  $\Delta_{init}$ and $\Delta R$ values across several test tasks presented in Table~\ref{metric}. MGF-RL achieves positive $\Delta_{init}$ and $\Delta R$ in all cases, indicating that its meta initialization not only provides early advantage but also support continued improvement. In contrast, warm-start RL shows variable initial performance and inconsistent long-term gains. AC-RL and MAML-RL frequently fall behind ES-RL, with negative values in both metrics across most tasks. While MAML-RL occasionally starts well, it consistently fails to sustain improvement, highlighting limitations in gradient-based generalization. For instance, in \textbf{Task 36}, MGF-RL demonstrates a significant $\Delta_{init}$ of 279.37 and a $\Delta R$ of 6.19, substantially outperforms than other meta-based algorithm, highlighting MGF-RL ability to adapt rapidly to new tasks while maintaining a trajectory of improvement throughout learning.

\subsection{Controller Evaluation}
To complement the learning-centric evaluation, in this subsection we analyze the operational behavior and decision-making quality of the trained MGF-RL controller controller from three perspectives: (1) how the controller react towards intermittent renewable generation, (2) comparative analysis of restoration strategies under a representative scenario, and (3) system-level reliability metrics aggregated across multiple test cases.

% from following aspects:% first a comparison of computational efficiency of proposed control  based on realistic metrics on a testing scenario, 
% first how the controller react towards intermittent renewable generation and then analyze control decisions made by MGF-RL compared to MPC and MAML-RL under a representative restoration scenario.

%%%%%%%%%%%%%%%%%%%%%%%%%%%%%%%%%%%%%%%%%%%%%%%%%%%%%%%%%%%%%%%%%%%%%%%%%%%%%%%%%%
% meta based RL controller based on multiple scenarios.
\begin{figure}[ht]
\centering
\includegraphics[width=6.0cm]{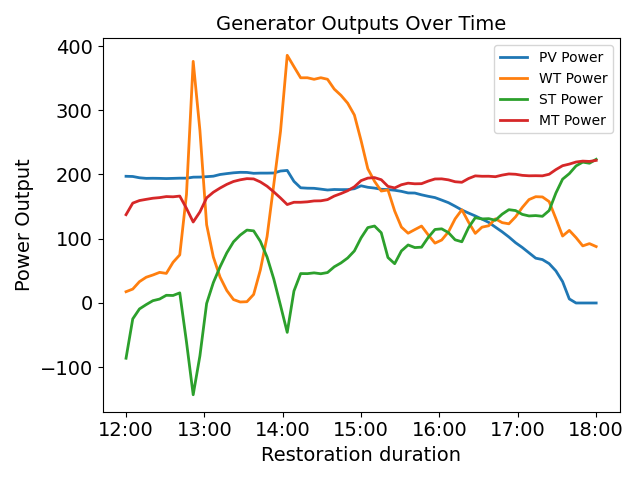}
\caption{Generation output from DERs over restoration duration }
\label{generationoutput}
\end{figure}

Fig.~\ref{generationoutput} illustrates the generation profiles of DERs over the control horizon under a specific restoration scenario. The WT output exhibits significant variability; however, the proposed MGF-RL controller effectively manages the dispatchable microturbine and energy storage systems at each timestep to mitigate fluctuations in renewable generation, which helps providing continuous support and stable power to loads that have been restored. During periods of high renewable output (e.g., 14:00-15:00), the controller strategically allocates excess generation to recharge battery storage rather than immediately restoring additional loads, building energy reserves for later periods when renewable generation may be insufficient. This conservative approach avoids the operational penalties associated with premature load restoration followed by necessary load shedding when resources become unavailable.  These observations suggest the proposed controller demonstrates decision-making capabilities at each time step. 
\begin{figure*}[ht]
    \centering
    % First Row - MGF-RL
    \begin{subfigure}[b]{\textwidth}
        \centering
        \subcaptionbox{MGF-RL}[0.98\textwidth]{ % Title for this row
            \begin{minipage}{0.23\textwidth}
                \centering
                \includegraphics[width=\textwidth]{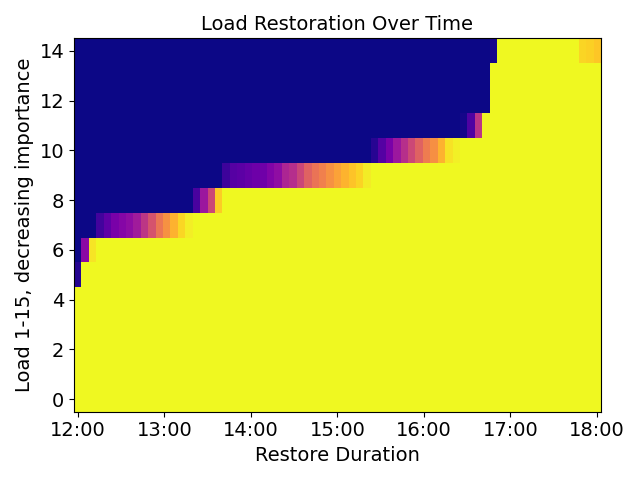}
            \end{minipage}
            \begin{minipage}{0.23\textwidth}
                \centering
                \includegraphics[width=\textwidth]{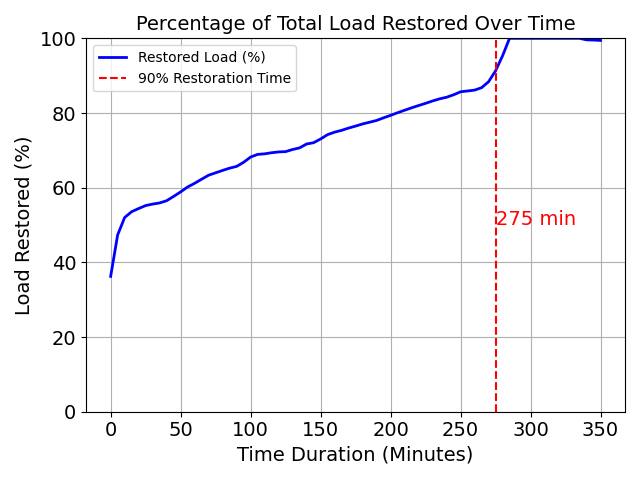}
            \end{minipage}
            \begin{minipage}{0.22\textwidth}
                \centering
                \includegraphics[width=\textwidth]{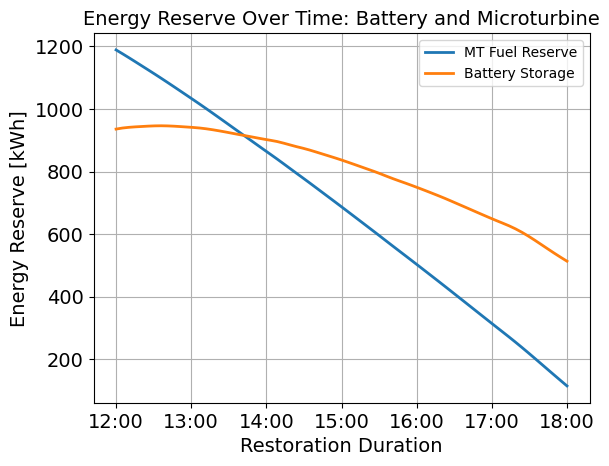}
            \end{minipage}
            \begin{minipage}{0.235\textwidth}
                \centering
                \includegraphics[width=\textwidth]{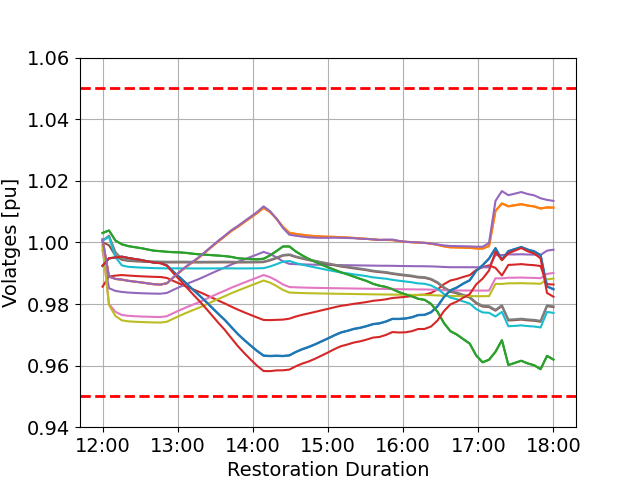}
            \end{minipage}
        }
    \end{subfigure}

    % Second Row - MPC
    \begin{subfigure}[b]{\textwidth}
        \centering
        \label{mpcsub}
        \subcaptionbox{MPC}[0.98\textwidth]{ % Title for this row
            \begin{minipage}{0.23\textwidth}
                \centering
                \includegraphics[width=\textwidth]{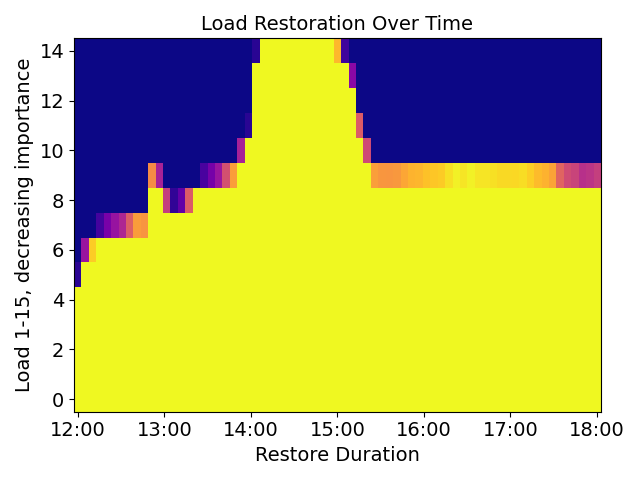}
            \end{minipage}
            \begin{minipage}{0.23\textwidth}
                \centering
                \includegraphics[width=\textwidth]{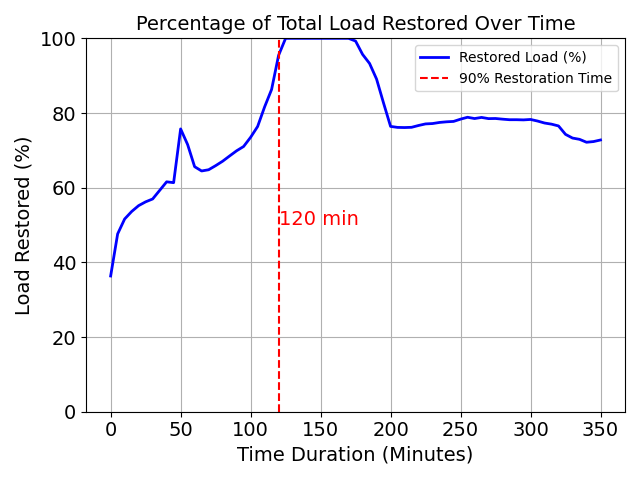}
            \end{minipage}
            \begin{minipage}{0.22\textwidth}
                \centering
                \includegraphics[width=\textwidth]{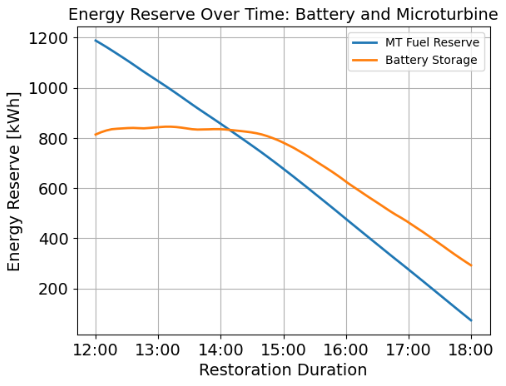}
            \end{minipage}
            \begin{minipage}{0.235\textwidth}
                \centering
                \includegraphics[width=\textwidth]{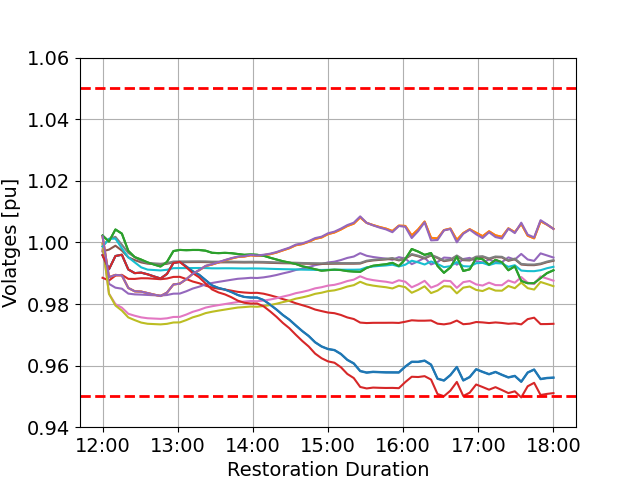}
            \end{minipage}
    }
    \end{subfigure}

    % Third Row - MAML-RL
    \begin{subfigure}[b]{\textwidth}
        \centering
        \subcaptionbox{MAML-RL}[0.98\textwidth]{ % Title for this row
            \begin{minipage}{0.24\textwidth}
                \centering
                \includegraphics[width=\textwidth]{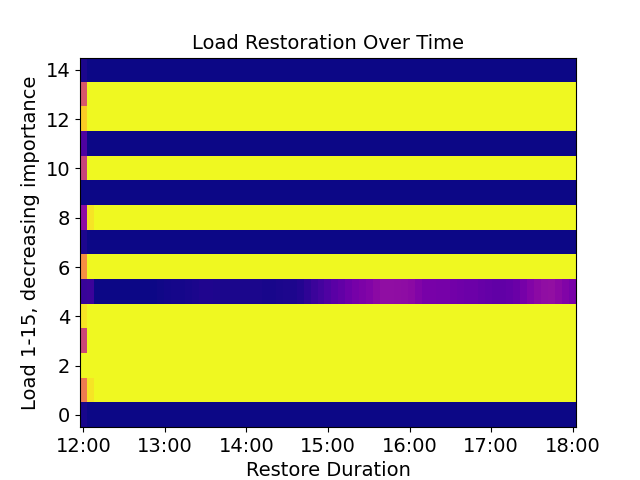}
            \end{minipage}
            \begin{minipage}{0.23\textwidth}
                \centering
                \includegraphics[width=\textwidth]{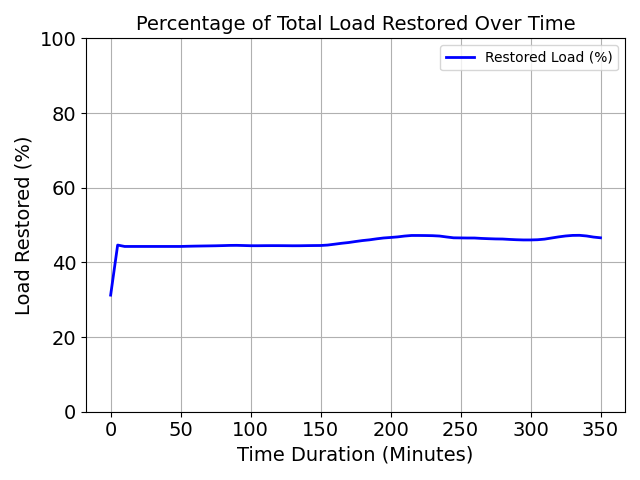}
            \end{minipage}
            \begin{minipage}{0.23\textwidth}
                \centering
                \includegraphics[width=\textwidth]{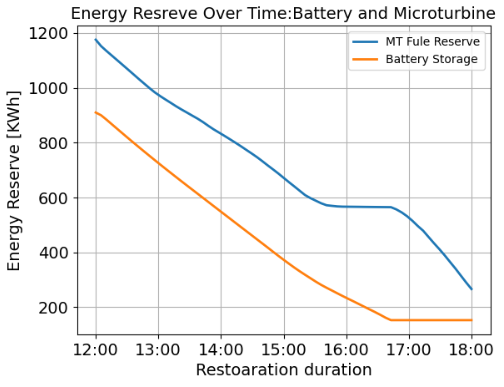}
            \end{minipage}
            \begin{minipage}{0.24\textwidth}
                \centering
                \includegraphics[width=\textwidth]{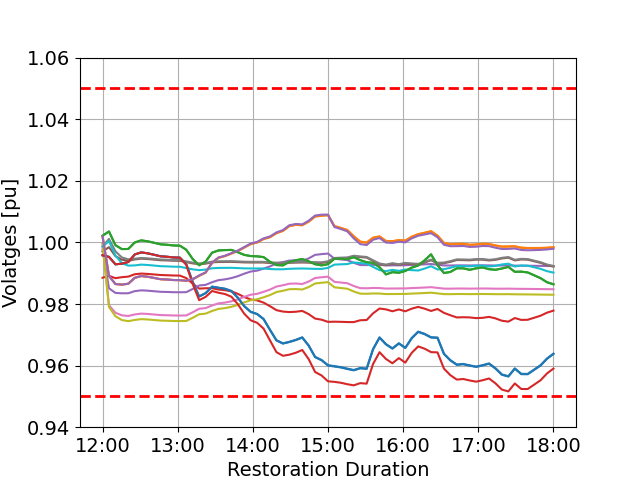}
            \end{minipage}
        }
    \end{subfigure}

    \caption{Load restoration comparison between MAML, MGF-RL, and MPC-based controllers. Brighter color means a higher percentage of a load is restored (i.e., Color yellow means 100\% restoration and dark blue means 0\%).}
    \label{singlescenario}
\end{figure*}
% It can be observed that although the WT generation changes violently, the MGF-RL based controller can properly control the microturbine and storage system at each step to compensate for the variability in renewable generation, which helps providing continuous support to loads that have been restored. Moreover when renewable generation is abundant, instead of greedily using it to restore more loads immediately, the controller choose to charge the power to the storage at this time since it has learned from training that restoring loads too soon might leads to penalty due to failure to sustain restored loads. From this scenario, we can see that the proposed controller can make some seemingly reasonable decisions at each step. To further study the quality of these decision, MGF-RL is compared with the baseline MPC and the state-of-the-art MAML-RL.
%%%%%%%%%%%%%%%%%%%%%%%%%%%%%%%%%%%%%%%%%%%%%

To further assess decision-making quality, Fig.~\ref{singlescenario} presents comparative analysis of three load‐restoration strategies: MPC, MGF‐RL, and MAML, under the same scenario. The comparison evaluates four key operational aspects across all methods: load restoration progression, percentage of load restoration progress over time, energy reserve usage, and voltage regulation performance. 
% Fig.~\ref{singlescenario} presents a comparative analysis of three load‐restoration strategies—MPC, MGF‐RL, and MAML—under the same scenario by evaluating their performance in terms of load restoration progression, total percentage of load restored, energy reserve usage, and voltage regulation. 
% Each row of subplots corresponds to a different metric, highlighting how each method manages uncertainty, allocates resources, and maintains grid stability under varying conditions.
The leftmost column reveals fundamental differences in restoration strategies. MPC adopts an aggressive approach, restores a large number of high-priority loads early in the horizon when renewable energy is abundant, but later sheds several previously restored loads (visible as transitions from yellow to blue in the heatmap) due to resource exhaustion and forecast inaccuracies during periods of low wind generation. This load cycling behavior is operationally undesirable, as it subjects critical loads to repeated interruptions and potentially damages sensitive equipment. In contrast, MGF‐RL demonstrates a strategic approach, where the controller starts by restoring four critical loads with high priority, and then gradually increase restoration level as it accounts for uncertainty in renewable generation from experience that predicted information cannot be totally relied on.
Thus rather than immediately consuming all the available renewable energy,  MGF‐RL proactively stores excess generation during high-wind periods, reducing the likelihood of needing to shed loads later. 
MAML-RL exhibits the least effective strategy, restores lower-priority loads prematurely and fails to achieve full restoration, likely due to limitations of gradient-based updates under abrupt wind or load changes. 
The second column quantifies cumulative restoration performance over the six-hour horizon. MGF-RL achieves the critical 90\% restoration milestone in approximately 275 minutes, while MPC exhibits instability ultimately failing to reach 90\% restoration with the time horizon and MAML-RL remain below 50\%, reflects MGF-RL ability to balance restoration with resource availability. Resource trajectories in column 3 illustrate that MGF-RL manages reserves more gradually: battery SOC decreases gradually from 1200 kWh to approximately 600 kWh, while fuel reserves decline from 1200 kWh to roughly 500 kWh, maintaining adequate reserves throughout the horizon for contingency response. The smooth, controlled depletion curves indicate coordinated dispatch that avoids premature resource exhaustion. MPC demonstrates more aggressive resource consumption early in the horizon, with steeper initial declines in both battery SOC and fuel reserves. This aggressive depletion leaves MPC vulnerable during later periods of low renewable generation, forcing load shedding to maintain system balance. MAML-RL's resource trajectories show erratic behavior with less strategic coordination between storage and dispatchable generation, contributing to its poor overall restoration performance. In column 4 Voltages profiles are more stable under MGF-RL, while MPC and MAML-RL show larger deviations due to aggressive or uncoordinated dispatch.

To get more comprehensive comparison beyond single-scenario analysis, Table~\ref{tablereliability} summarizes system-level performance based on three key reliability metrics, averaged over 15 test scenarios encompassing a range of load demands, renewable generation patterns, and forecast uncertainties. MGF-RL achieves the lowest SAIDI, representing a 27\%,  41\% and 33\% improvement over MPC, warm-start RL, and MAML-RL respectively. It is the only method to consistently restore 90\% of the load within the time horizon. These results indicate that MGF-RL not only learns effective dispatch policies but also maintains system reliability.

% To get more comprehensive comparison each controller examined on 15 testing scenarios. The table \ref{tablereliability} presents three key performance metrics for load restoration strategies: \textit{SAIDI, 90\% Restoration Time,} and the \textit{percentage of load restored} over the control horizon. \textit{SAIDI}, which measure average outage duration per load, shows, our method as the most effective, with a SAIDI of 135.3 minutes compared to Warm-start RL, has a significantly higher SAIDI of 230.0 minutes, indicating much slower restoration, whereas MPC and MAML yield intermediate SAIDI values (184.7 and 203.5 minutes, respectively). The ``\textit{90\% Restoration Time}” indicates the time required to restore 90\% of the load. Our method restores 90\% of the load within 305 minutes, while others did not achieve this level of restoration within the control horizon. Finally, the ``\textit{percentage of  Load Restored}” row shows our method fully restores the load (96\%). In contrast, Warm-start RL and MAML achieve only 68\% and 44\% load restoration, respectively, while MPC comes close to 82\% restoration. In general, our strategy excels at minimizing the duration of the interruption and maximizing the effectiveness of the restoration.
\begin{table}

\caption{Average values of reliability indices for load restoration strategies over 15 scenarios.}
\begin{adjustbox}{width=\columnwidth}
\begin{tabular}{c c c c c c  } 
 \hline
 \textbf{Metric} & \multicolumn{1}{c}{\textbf{{Ours}}} & \textbf{Warm-start RL} & \textbf{MAML} & \textbf{MPC}\\ 

 \hline
 SAIDI & 135.3 min   & 230.0 min & 203.5 min & 184.7 min \\
90\% Restore & 305 min  & not restore & not restore & not restore\\
\% Restored & 96\%  & 68\% & 44\% & 82\%\\
 % \hline
 % SAIFI & 0.99 & 0.66 & 0.60 & 0.46 & 0.8 \\

 \hline
\end{tabular}
\end{adjustbox}
\label{tablereliability}
\end{table}

\subsection{Impact of Forecast Error on Load Restoration Reward}\label{secforecasterror}

% The efficacy of load restoration strategies is significantly influenced by the accuracy of renewable energy generation forecasts. 
This section investigates how renewable forecast inaccuracies affect the performance of load restoration controllers and evaluates how different forecast horizons and learning strategies respond to varying levels of uncertainty. In practice, forecast accuracy varies depending on geography, resource type, and operator methodology, making it essential to assess how restoration policies generalize under imperfect information. To simulate this, we use 30 days of historical renewable generation data to create training scenarios ($\Re^{tr}$), and reserve the following 7 days for testing ($\Re^{ts}$). We consider six levels of forecast error: $\Xi_{T} = \left\lbrace0\%, 5\%, 10\%, 15\%, 20\%, 25\%\right\rbrace$, and train separate agents under each. We also vary the forecast lookahead window, using $\kappa \in \left\lbrace1, 2, 4, 6\right\rbrace$ hours to study how the depth of prediction influences learning and control quality.

% Fig.~\ref{9a} illustrates the performance of RL controllers under varying renewable forecast error levels $\Xi_{T}$ and highlights how 
% $\kappa$-lookahead strategies influence restoration rewards $\sum_{t\in \mathcal{T}}r(t)$. For our study, we used the 30 days of actual
 % renewable generation profiles to generate the training data set $\Re^{tr}$, while data from the following seven days are reserved for testing $\Re^{ts}$.  In real-world applications, forecast accuracy varies due to differences in geographic location, resource characteristics, and prediction methodologies employed by grid operators. To systematically analyze the impact of forecast inaccuracies on RL controller performance, we evaluate six different forecast error levels, $\Xi_{T}=\left\lbrace 0, 5\%, 10\%, 15\%, 20\%, 25\% \right\rbrace$. For each error level $\Xi_{T}$, a corresponding training scenario $\Re^{tr}$ is generated, and an RL agent is trained before being evaluated on~$\Re^{ts}$. Additionally, to assess the influence of forecast horizon length on control behavior, we investigate four different lookahead durations, denoted as $\kappa\in \left\lbrace 1, 2,4 , 6 \right\rbrace$, where $k$ represents the number of hours into the future that the controller considers when making decisions..

\subsubsection{Learning Behavior under Forecast Error levels}
Fig.~\ref{9a} shows learning curves for different forecast error levels $\Xi_{T}$. As expected, lower error levels lead to faster convergence and higher restoration rewards. As forecast error increases, learning slows and reward variance increases (e.g., $\Xi_{T}=0.20$), indicating greater difficulty in managing uncertainty. The inset highlights differences in convergence behavior, emphasizing how imperfect forecasts can hinder learning dynamics. This experiment quantifies the sensitivity of RL controllers with respect to forecast reliability and shows that even moderate errors impact restoration efficiency.

\subsubsection{Effect of Forecast Horizon ($\kappa$) Under Forecast Error}
Fig.~\ref{9b} compares $\kappa$-lookahead strategies under different forecast error levels. Longer lookaheads ($\kappa = 4, 6$) generally improve performance when forecast error is moderate or low, as they help anticipate future conditions. However, under high error ($\Xi_T \geq 20\%$), their advantage diminishes, as long-range forecasts become increasingly unreliable. Controllers using shorter horizons perform more consistently across all error levels, suggesting that conservative forecasting can help maintain robust control when prediction accuracy is limited. This highlights a key tradeoff between anticipation and reliability.

% In Fig.~\ref{9b}compares the performance of $\kappa-$lookahead RL controllers under different forecast error levels. The results indicate that controllers utilizing longer lookahead windows~$\kappa$ tend to achieve higher rewards when forecast errors are moderate, as they are able to better anticipate future conditions. However, as forecast errors increase, the rewards decline across all lookahead strategies, reflecting the growing difficulty in leveraging unreliable forecasts. Despite this, RL controllers maintain relatively stable performance, showcasing their adaptability to forecast inaccuracies. The figure emphasizes the importance of selecting appropriate lookahead strategies to optimize performance under varying levels of uncertainty. 
\begin{table}
\centering
\caption{ES-RL with forecast head 2 achieves the highest reward at $\Xi= 10\%$ and degrades more gracefully than PPO, DDPG, and MPC as $\Xi_{T}$ increases.
% Comparison of converged restoration reward with different RL controllers and MPC under different forecast error levels $\Xi_{T}$. ES-RL with forecast head 2 achieves the highest performance across most error levels, demonstrating superior robustness to forecast inaccuracies compared to PPO, DDPG, and MPC.
}
\begin{adjustbox}{width=\columnwidth}
\begin{tabular}{c c c c c c c c} 
 \hline
 \textbf{$\Xi$} & \multicolumn{4}{c}{\textbf{ES-RL (Forecast Head )}} & \textbf{PPO} & \textbf{DDPG} & \textbf{MPC}\\ 
 \hline
 & \textbf{1} & \textbf{2} & \textbf{4} & \textbf{6} & \textbf{2}&\textbf{2} & \\ [0.5ex]
 \hline
 0\% & 18.18 & 18.95 & 18.53 & 18.94 & -7.38 & 10.25  & \bf{19.98}\\
 \hline
 5\% & 18.65 & 18.59 & 18.54 & 18.56 & -5.52 & 11.31 & 18.63 \\
 \hline
 10\% & 16.63 & {18.29} & 18.19 & 15.81 & -37.47 & 09.51 & 17.28 \\
 \hline
 15\% & 18.63 & 17.58 & 16.65 & 15.79 & -38.74 & 08.18 & 15.10 \\
 \hline
 20\% & 18.48 & 17.96 & 16.68 & 15.19 & -12.05 & 06.08 &14.57 \\
 \hline
 25\% & 16.70 & 16.64 & 16.27 & 14.07 & -17.68 & 06.29 & 12.86\\ [0.5ex]
 \hline
\end{tabular}
\end{adjustbox}
\label{table3}
\end{table}
\subsubsection{Performance Comparison Across Controllers Under varying $\Xi_{T}$} Table~\ref{table3} provides a comparative evaluation of RL controllers (ES-RL, PPO~\cite{PPO}, DDPG~\cite{ddpg}) and MPC \cite{RC-MPC} across different $\Xi_T$ for load restoration rewards $\sum_{t\in\mathcal{T}}\left(\boldsymbol{\varsigma}^{\top} \textbf{p}_t - {\mu}\boldsymbol{\varsigma}^{\top}[\textbf{p}_{t-1}-\textbf{p}_t]^{+}+\mathcal{V}_t\right)$. As expected, all controllers perform better at lower error levels. Notably, MPC provides highest reward 19.98 at $\Xi_T = 0\%$, reflecting its reliance on deterministic forecasts. ES-RL approaches this upper bound with 2 and 6-hours forecast head reaching 18.95 and 18.94. However, MPC performance degrades faster than ES-RL as forecast error grows. For instance, at $\Xi_T = 25\%$, MPC reward drop to 12.86, while ES-RL with 2-hours forecast head still maintains 16.64. This highlights the greater resilience of model-free methods trained under uncertainty. Unlike MPC, which assumes accurate forecasts during planning, ES-RL learns to hedge against inaccuracy and adopt robust, cautious strategies over time. Among the ES-RL variants, short-horizon controllers are more robust under high uncertainty, as 
% Notably, at $\Xi_T = 10\%$, ES-RL with forecast head 2 hrs achieves the highest reward 19.19. 
forecast head 6 performs well under low error but drops off more steeply with higher uncertainty. 
% ES-RL balances responsiveness with adaptability and remains robust even as forecast quality degrades, 
The gradient-based method PPO struggles across all settings, yielding negative rewards and failing to adapt to uncertainty. DDPG performs better than PPO but performs less well relative to ES-RL, especially as $\Xi_T$ increases. These results validate that model-free methods like ES-RL, when paired with appropriately chosen forecast horizons, can outperform both model-based and gradient-based controllers under realistic forecast uncertainty.

\begin{figure}[t]
    \centering
    
    \begin{subfigure}[b]{0.2\textwidth}
        \centering
        \includegraphics[width=\textwidth]{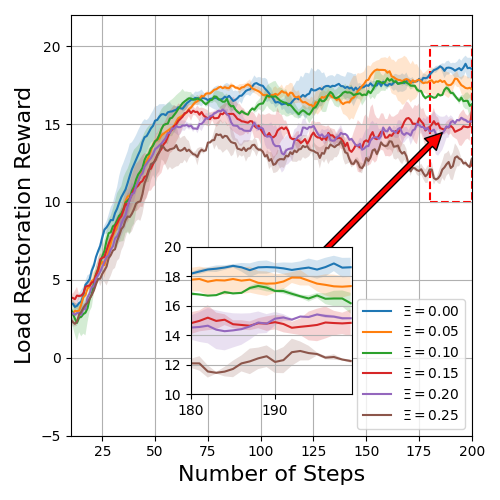}
        \caption{ }\label{9a}
    \end{subfigure}
    \hfill
    \begin{subfigure}[b]{0.25\textwidth}
        \centering
        \includegraphics[width=\textwidth]{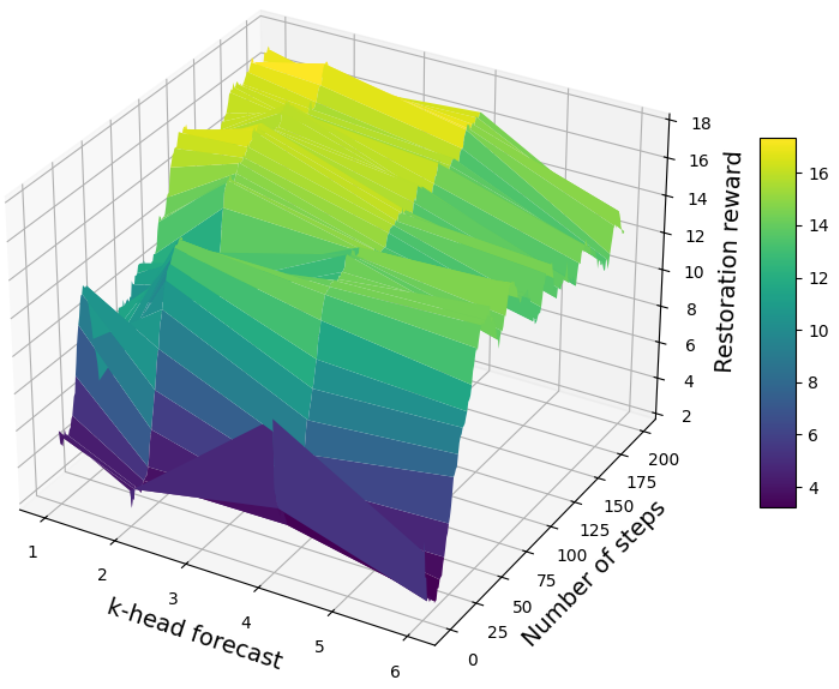}\caption{} \label{9b}
    \end{subfigure}
     \caption{(a) Learning curves under varying $\Xi_{T}$. (b) Restoration rewards for $\kappa-$lookahead RL controllers across forecast error levels. Longer lookaheads improve performance under low-to-moderate errors but are less effective as uncertainty increases.}
    % \label{123bussystem}
\end{figure}
\begin{table}[ht]
\centering
\caption{Comparison of control complexity and training time across systems}
\label{tab:scalability}
\begin{tabular}{lccccc}
\toprule
\textbf{System} & \textbf{\# Loads} & \textbf{\# DERs} & $\boldsymbol{|\mathcal{S}|}$ & $\boldsymbol{|\mathcal{A}|}$& \textbf{(min/task)} \\
\midrule
IEEE-13    & 15  & 4 & 114 & 20 &  $\approx$25 \\
IEEE-123   & 20  & 6 & 121 & 29 & $\approx$30 \\
\bottomrule
\end{tabular}
\end{table}

\subsection{Scalability to Large Test System}
We evaluate the scalability of MGF-RL on a modified IEEE-123 bus system to assess whether the method remains effective under increased control complexity. Unlike optimization-based methods, where computational cost increases with the number of variables (e.g., buses, lines), RL complexity depends primarily on the dimensionality of the state and action spaces, which are determined by the number of controllable DERs and critical loads. To increase the complexity, we increase the number of critical load $|\mathcal{L}|=20$ and DERs $|\mathcal{G}|=6$ in IEEE-123 bus system, resulting in an action space of 29 dimensions and a state space dimensions depends on renewable forecast, e.g. based on a 4-hour lookahead horizon the state dimension is 121. Table \ref{tab:scalability} compares the scale of two cases and shows the required time for convergence. We train the policy using the same number of iterations and policy architecture as for the 13-bus system, according to the results, scaling from the 13-bus system to the 123-bus system does not increase the computational resources significantly. 
% In our setting, $|\mathcal{G}|=6$ and $|\mathcal{L}|=20$ does not increase remarkably. 
% As a result, the increase in system size does not significantly affect the policy’s input or output dimensions. We train the policy using the same number of iterations as for the 13-bus system.
Fig.~\ref{123bussystem} shows the DER outputs and restoration progression for a representative test scenario, confirming that MGF-RL generalizes well to larger systems with more granular control demands.

% The proposed MGF-RL method is also tested in solving a CLR problem in a modified IEEE-123 bus test system. we train the policy for same number of iterations when scaling from 13-bus to 123-bus system thus computational resources does not increases unlike optimization based methods, in which the the problem scale (e.g., number of variables) increases linearly with the number of buses, the RL problem scale depend more on the state/action spaces, which are related to the numbers of DERs and loads. In this case, $|\mathcal{G}|=6$ and $|\mathcal{L}|=20$ does not increase remarkably and thus limit the increase of computational resources. In addition, in cases where multiple renewable generation share the same profile due to the proximity of location, the dimension of $s_{t}$ might remain the same though $|\mathcal{R}|$ increases. Fig.~\ref{123bussystem} shows the restoration process under a testing scenario, illustrating the effectiveness of the trained policy.
\begin{figure}
    \centering
    \begin{subfigure}[b]{0.23\textwidth}
        \centering
        \includegraphics[width=\textwidth]{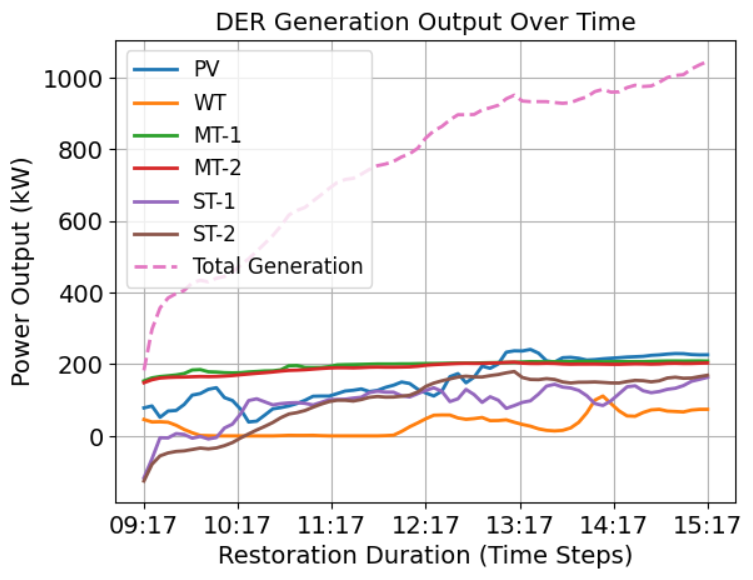}
    \end{subfigure}
    \hfill
    \begin{subfigure}[b]{0.23\textwidth}
        \centering
        \includegraphics[width=\textwidth]{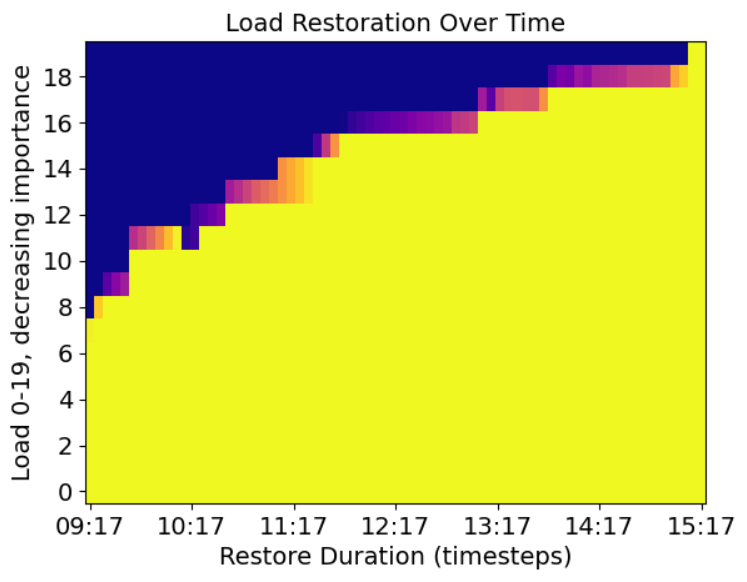}
    \end{subfigure}
    \caption{DERs generation profiles and load restoration process in one testing
 scenario in the IEEE 123-bus system}
    \label{123bussystem}
\end{figure}

\section{Conclusion}\label{conclusion}

We present a meta-learning framework for CLR that addresses key scalability and adaptability challenges in distribution system control under renewable uncertainty. The proposed MGF-RL algorithm integrates first-order meta-updates with gradient-free evolutionary strategies, enabling model-free policy learning in non-differentiable simulation environments. By leveraging shared structure across tasks, MGF-RL generalizes to unseen outage scenarios without requiring second-order derivatives or retraining for each instance. Empirical evaluations on modified IEEE-13 and IEEE-123 systems demonstrate consistent gains in convergence, restoration speed, and robustness to forecast errors compared to baseline RL and MPC approaches. Theoretical regret analysis further quantifies how task similarity and temporal variability impact learning performance in both static and dynamic settings, supporting the observed empirical trends. We benchmark MGF-RL against both model-based (MPC) and model-free (PPO, DDPG, MAML, AC-RL) baselines to assess its adaptability and generalization. While model-based strategies such as adaptive optimal power control (OPC) and robust distribution OPF (e.g., chance-constrained formulations) offer principled approaches under uncertainty, they often rely on full observability, linear approximations, and repeated reoptimization. In contrast, MGF-RL provides a scalable alternative by learning policies directly from data. This study focuses on post-reconfiguration DER scheduling and assumes fixed topology and constant load demand during outages. Future work will extend the framework to support dynamic load profiles, adaptive network topologies, and hybrid formulations that combine meta-RL with robust optimization methods for long-horizon, physics-informed planning and control.

\appendix
\section{Nomenclature}
\noindent
\begin{tabular}{ll}
    % \toprule
    \textbf{Notation} & \textbf{Definition} \\
    % \midrule
$\mathcal{D}^{f}$,$\mathcal{D}^{s}$,$\mathcal{R}$&Dispatchable fuel and battery storage resources\\& and Renewable-based energy.\\
{$S^{\theta}_{t}$,$\zeta_{t}$,$E^f$}&State of charge, battery storage efficiency and \\&maximum fuel reserve.\\
$\boldsymbol{\varsigma}$&{Load priority vector.}\\
{$\textbf{p}_t$,$\textbf{q}_{t}$}&{Active and reactive power.}\\
{$N$,$N_b$}&{Number of loads and number of buses.}\\
{$\hat{p}^{r}_{t}$, ${\Xi}$}&Renewable generation forecast, forecast error.\\
 ${\kappa}$&Forecast look-ahead. \\
{$\mathcal{T}$}&{Control interval length.}\\
% {$\pi(\cdot)$}&{Neural network as policy}\\
{$m\in[M]$}&{Set of all tasks.}\\
{$\pi_{M}(\cdot)$, $\phi_{M}$}&{Meta policy with meta policy parameters.}\\
{$\hat{\pi}_{m,t}$}&Sub-optimal policy received by with-in task \\&algorithm for task $m$.\\
{$\alpha_t$}&{Task specific learning rate for ES-RL.}\\
{$\eta_m$}&{Meta learning rate for meta-update at task $m$.}\\
{$T$}&{Number of iterations to update the policy.}\\
{$\boldsymbol{\varepsilon}$}&{Gaussian random noise.}\\
{$D^*,~\hat{D}^*$}&{Actual and empirical task-similarity.}\\
{$S^*, \hat{S}^*$}&{Actual and empirical task-relatedness.}\\
{$\phi^{*}_{m}$}&{Optimal meta initialization.}\\
{$P_{M},S_{M}$}&Path-length and squared path-length of optimal  \\&policy initialization $\pi^{*}_{m,0}$.\\
{$L_{F}$}&{Lipschitz constant of $F_{m}$.}\\{$L_g$} &Bound for $\nabla_{\phi}D_{KL}(\phi^{*}|\phi)$.\\
{$V_{M}(\cdot)$}&{Temporal variability of loss function $l_{m}(\cdot)$.}\\
{$\boldsymbol{\epsilon}_{M}$}&{Cumulative inexactness in KL divergence.}\\
\end{tabular}
\section{Automated Curriculum-based RL}\label{ACL}
Motivated by the work on curriculum-based RL for critical load restoration problem \cite{zhang2022curriculum}, a two stage automated curriculum learning approach is devised. The training stage of this approach focuses on learning a meta-policy, $\pi_{M}$, tailored to a simplified problem scope with a reduced action space as defined in \cite{zhang2022curriculum}:
$$a_{t}^{1}:=\left[(p^{\theta}_{t})^{\top},(p_{t}^{\mu})^{\top}, \left(\textbf{H}_{\alpha}\boldsymbol{\alpha}_{t}^{\mathcal{G}}\right)^{\top}\right]^{\top}\in\mathcal{A}$$
\begin{algorithm}[t]
{\small

\caption{Automated Curriculum Learning}
\label{alg:curriculum}

\DontPrintSemicolon  % Suppress semicolon at the end of lines for a cleaner look
\SetAlgoNoEnd
\SetAlgoNoLine
\SetKwFunction{FMain}{Train}
\SetKwFunction{FTesting}{Testing}
\SetKwProg{Fn}{Function}{:}{}

\Fn{\FMain{}}{
Initialize random policy $\pi_{1,0}$ with parameters $\phi_{1,0}$\;

\For{each task $(m = 1, \dots, M)$}{
  Load initial policy with parameter $\phi_{m,0}$ for task $m$ \;
  \For{iteration $(j = 1, \dots, T)$}{
    Train policy using ES-RL to solve simplified CLR problem with control agent the determine only DERs dispatch instead of learning both load pickup and DERs set point, i.e.,
    $a_{t}^{train}:=\left[\left(\textbf{p}_{t}\right)^{\top}, \left(\boldsymbol{\alpha}_{t}^{\mathcal{G}}\right)^{\top}\right]$
    
    % Update policy using ES-RL \;  % Assume ES-RL citation included in text
    Save the best model $\hat{\pi}_{m,j}$ and model parameter $\hat{\phi}_{m,j}$ \;
  }
  \textbf{meta-update:} $\phi_{m+1,0} \leftarrow \phi_{m,0} + \frac{\epsilon}{M} \sum_{m=1}^{M} (\hat{\phi}_{m,j} - \phi_{m,0})$ \;
}

Save meta-policy $\pi_M = \hat{\pi}_{M,j}$ \;
}

\Fn{\FTesting{}}{
% \textbf{For the Test Task:} \;
Load meta-policy $\pi_M$ \;
Generate behavior data and control trajectories using the meta-policy under training conditions i,e $\left\lbrace\left(s_{t}^{train}, a_{t}^{train},\dots\right)\right\rbrace$\;
Convert the generated actions $a_{t}^{train}$ to the testing-stage format $a_{t}$ by aligning them with the full action space (generation set points and load pickup decisions) using a greedy load restoration strategy\;
Train a neural network to map $s^{train}_t$ to $a_t$ using supervised learning, resulting in the initialization for test Stage policy parameters $\phi^{test}_0$\;
% Train a neural network using supervised learning on the state-action pairs generated by the meta-policy to handle both generation set points and load pickups\;
Fine-tune meta-policy at test time on the new unseen task using PPO to receive the optimal policy $\hat{\pi}_M$. \;
}
}
\end{algorithm}
In this context, the RL controller focuses solely on the dispatch of DERs, while load restoration is managed through a rule-based greedy approach that prioritizes higher demand loads given the available generation. Notably, during each control step $t$, the RL agent is provided with exact forecasts, i.e., the actual renewable generation data for the look-ahead period, rather than relying on $\hat{p}^{r}{t}$. This adaptation leads to a modified state representation:
$$s_{t}^{1}:= [(p^{r}_{t})^{\top},(\tilde{p}_{t-1})^{\top},(S^{\theta}_{t})^{\top}, (E^{\mu}_{t})^{\top},t]^{\top}.$$
This allows the RL agent to focus on mastering the grid control problem with reduced uncertainty. Within
each task during training we use the ES-RL, and the meta-parameters are updated through a first-order approximation of the meta-gradient, based on the cumulative performance across all tasks. This meta-update step is crucial as it accumulates knowledge across tasks and fosters the learning of a generalized meta-policy $\hat{\pi}_{m,j}$, which is then used to initialize the policies for the test tasks. A key challenge remains in the transfer of knowledge from the training to the testing stage due to the structural differences between the policy networks of the two stages ($\mathcal{A}^{1} \neq \mathcal{A}$), making direct weight copying impractical. To address this, the approach proposed in \cite{zhang2022curriculum} utilizes a behavior cloning technique to train a policy network that is compatible with the test stage format (see Algorithm~\ref{alg:curriculum}).

\section{Provable Guarantees for MGF-RL within-Online Framework}\label{appendixtheoretical}

In this section, we present a theoretical analysis of the proposed MGF-RL algorithm. We first provide regret bounds for the ES-RL algorithm when applied to a single task. We then extend the analysis to the multi-task setting, introducing \textit{task-averaged regret}, task similarity metrics, and their impact on convergence guarantees. Finally, we derive \textit{dynamic regret bounds} to quantify performance in non-stationary environments.
 \subsection{Preliminaries}
% Before delve into the proof of main results, we need to define some notation. For a function 
Before presenting the proofs of our main results, in this section, we establish necessary definitions and notational conventions. Let $f:\mathbb{R}^{d} \to \mathbb{R}$ be a function, the sub-gradient of $f$ at a point $x$ is denoted by $\partial f(x)$. We say that $f$ is $\mu$-strongly convex over a convex set $V \subseteq \text{int dom}(f)$ with respect to a norm $ \| \cdot \|$ if, for any $ x, y \in V $ and $ g \in \partial f(x) $, it holds that $f(y) \geq f(x) + \langle g, y-x \rangle + \frac{\mu}{2} \|x-y\|^2.$
Moreover, define $ \psi: \mathcal{X} \to \mathbb{R} $ as a strictly convex and continuously differentiable function on $ \text{int}~\mathcal{X} $. The Bregman Divergence associated with $ \psi $ is given by $ B_{\psi}(x, y) = \psi(x) - \psi(y) - \langle \nabla\psi(y), x - y \rangle $, assuming $\psi$ is strongly convex with respect to the norm $ \| \cdot \| $ on $ \text{int} \mathcal{X} $.
\begin{assumption}\label{assumption1}
    The meta initialization policy $\pi_{m,0}$ for any task $m$ lies inside a shrinkage simplex set i.e. $\pi_{m,0}(\cdot|s)\in \Delta \mathcal{A}_{\rho}:=\left\lbrace a_{1}e_{1}+\dots a_{n}e_{n} |\sum_{i=1}^{n}a_{i}=1,\quad a_{i}\geq \rho \right\rbrace$ for all $s\in \mathcal{S}$.
\end{assumption}
% \subsection{Within-Task Analysis}
% We consider a sequence of MDPs, indexed by $m=1,\dots,M$, where each task represents an individual RL problem. Within each task $m$, the agent refines the policy parameter $\left\lbrace \phi_{m,j}\right\rbrace_{j=0}^{T}$ over $T$ iterations using ES-RL algorithm to receive a suboptimal policy parameters $\hat{\phi}_{m,T}$. The following result provides convergence guarantees for ES-RL.
% \begin{theorem}[Theorem 6; \cite{analysis}]
%     Suppose ES-RL policy update for each task $m$, perform $T= \frac{4(N+4)^2 (L_{F})^2 }{\epsilon^2}$ iterations 
%     % with learning rate $\alpha_{m}$=$\frac{R}{(N+4)(T+1)^{1/2}(L_{0}(F_{m}))}$ 
%     to optimize objective function $F_m(\cdot)$. For $\sigma\leq \frac{{\epsilon}}{2L_{F}N^{1/2}}$, and  with learning rate $\alpha_{t}$=$\frac{1}{(N+4)(T+1)^{1/2}(L_{F})}$ then for any $\epsilon>0$ the sub-optimality gap for each task m is bounded by;
% %     \begin{equation}\label{esconv}
% % R= \mathbb{E}\left[F_{m}\left(\hat{\phi}_{m,T}\right)\right]-F_{m}\left({\phi}^{*}_{m}\right)\leq \frac{2(N+4)L\|\phi^{*}_{m}-\phi_{m,0}\|}{\sqrt{T}},
% %     \end{equation}}
%     \begin{equation}\label{esconv}
% \mathbb{E}\left[F_{m}\left(\hat{\phi}_{m,T}\right)\right]-F_{m}\left({\phi}^{*}_{m}\right)\leq \frac{2(N+4)L_{F} D_{KL}\left(\phi^{*}_{m}|\phi_{m,0}\right)}{\sqrt{T}},
%     \end{equation}
% where $\phi^{*}_m$ are the parameters of optimal policy $\pi^{*}_{m}$ and $L_{F}$ is the Lipschitz constant of $F_{m}(\cdot)$. 
% % $R$ being the bound of $\|\phi^{*}_{m}-\phi_{m,0}\|\leq R$. 
% \end{theorem}
 
\subsection{Task-Averaged Regret and Task Similarity}We can observe from \eqref{esconv}, that the regret bound is depending on the parameters of the policy initialization $\left\lbrace\pi_{m,0}\right\rbrace _{m=1}^{M}$. Beyond the single task, we consider the lifelong extension of MDPs in which we have sequence of online learning problems $m=1,2,\dots,M$. In each single task $m$, the agent must sequentially optimize the policy $\left\lbrace\pi_{m,t}\right\rbrace _{t=1}^{T}$, so that the corresponding sub-optimality decay sub-linearly in $T$, as give in \eqref{esconv}. The proposed MGF-RL algorithm aims to sequentially update the initial policy $\pi_{m,0}$ with parameters $\phi_{m,0}$, with aim to minimize the task average optimality gap (TAOG) after $M$ tasks and defined as:
\begin{equation}
\begin{split}
    \bar{R}=\frac{1}{M}\sum_{m=1}^{M}&\mathbb{E}\left[F_m(\hat{\phi}_{m,T})\right] -F_m(\phi_{m}^{*})\\&\leq \frac{2(N+4)L_{F} \sum_{m=1}^{M}D_{KL}\left(\phi^{*}_{m}|\phi_{m,0}\right)}{M\sqrt{T}}
\end{split}
\end{equation}
The definition of the TAOG consider the generalization ability of $\hat{\phi}_{m,T}$ to unseen test task. In meta-RL, the extent to which TAOG improves is influenced by the similarity among the sequential MDP tasks. We now discuss the notion of similarity in an environment.
For any fixed  initial policies parameters $\left\lbrace \phi\right\rbrace$, and given the optimal policies parameters $\left\lbrace\phi^{*}_{m}\right\rbrace_{m=1}^{M}$ for every $m$, the \textbf{task similarity} in a static environment can be measured by $D_{*}=\min_{\phi\in \Delta(\mathcal{A})^{|\mathcal{S}|}}\frac{1}{M}\sum\limits_{m=1}^{M}D_{KL}\left(\phi^{*}_{m}|\phi\right)$. This notion of task similarity is natural, as it implies that there exists a meta initialization $\phi$ with respect to which optimal policies for individual tasks are all close together. In practical scenario where only suboptimal policies are accessible, we denote the \textbf{empirical task similarity} defined in \cite{khattar2022cmdp} as $\hat{D}_{*}=\min_{\phi\in \Delta(\mathcal{A})^{|\mathcal{S}|}}\frac{1}{M}\sum\limits_{m=1}^{M}D_{KL}\left(\hat{\phi}_{m}|\phi\right)$, which depends on the suboptimal policies returned by a within-task algorithm. To understand the impact of task-similarity on the upper bounds of the TAOG, we present the following theorem shows the TAOG for the proposed MDP-within-online framework under the ideal setting where optimal policy parameters $\left\lbrace\phi_{m}^{*}\right\rbrace_{m=1}^{M}$ are available for each each task $m$ and the task similarity $D_{*}$ is known.
\begin{theorem}\label{TAOGwithoptimal}
    Suppose $D_{KL}\left(\phi^{*}_{m}|\cdot\right)$ is $\mu_{\phi}$-strongly convex and we are given with optimal policy parameters $\left\lbrace\phi_{m}^{*}\right\rbrace_{m=1}^{M}$ for each task and the task similarity $D_{*}$ is known. For each task $m$, we run the ES-RL algorithm for $T$ and meta policy is obtained by FTRL or OMD on the function $D_{KL}\left(\phi^{*}_{m}|\cdot\right)$, the TAOG is bounded as
    \begin{equation}\label{eq13}
        \begin{split}
\sum_{m=1}^{M}&\frac{\mathbb{E}\left[F_m(\hat{\phi}_{m,T})\right] -F_m(\phi_{m}^{*})}{M}\\&\leq 2L_{F}(N+4)\left(\frac{L_{g}^{2}\log(M+1)}{\mu_{\phi}M\sqrt{T}}+\frac{D_{*}}{\sqrt{T}}\right)
        \end{split}
    \end{equation}
    where $L_{g}$ is the upper bound on $\nabla_{\phi}D_{KL}\left(\phi^{*}|\phi\right)$.
\end{theorem}
\begin{proof}
    By the within-task gurantee for MDPs, we know that the task average optimality gap is well defined and bounded by  
 \begin{equation}
 \begin{split}
     \bar{R}&\leq \frac{2(N+4)L_{F} \sum_{m=1}^{M}D_{KL}\left(\phi^{*}_{m}|\phi_{m,0}\right)}{M\sqrt{T}}\\& \leq \frac{2(N+4)L_{F} }{\sqrt{T}}\sum_{m=1}^{M}\frac{D_{KL}\left(\phi^{*}_{m}|\phi_{m,0}\right)-D_{KL}\left(\phi^{*}_{m}|\phi^{*}\right)}{M}\\&+\frac{2(N+4)L_{F} }{\sqrt{T}}\sum_{m=1}^{M}\frac{D_{KL}\left(\phi^{*}_{m}|\phi^{*}\right)}{M}
 \end{split}
\end{equation}  
We split the total loss into the loss of meta-update algorithm and the loss if we always start initialized with $\phi^{*}$. Since each $D_{KL}(\phi^{*}_{m}|\cdot)$ is $\mu_{\phi}$-strongly convex because of  assumption \eqref{assumption1} and each $\phi_{m,0}$ is determined by playing OMD or FTRL, we have 
\begin{equation}\label{eq20}
\begin{split}
    \frac{2(N+4)L_{F} }{\sqrt{T}}\sum_{m=1}^{M}&\frac{D_{KL}\left(\phi^{*}_{m}|\phi_{m,0}\right)-D_{KL}\left(\phi^{*}_{m}|\phi^{*}\right)}{M} \\&\leq \frac{2(N+4)L_{F} L_{g}^{2}\log(M+1)}{M\mu_{\phi}\sqrt{T}}
\end{split}
\end{equation}
where, $L_{g}$ is the upper bound of $\nabla_{\phi}D_{KL}\left(\phi^{*}_{m,0}|\phi\right)$. since $\phi^{*}=\arg\min_{\phi}\sum_{m=1}^{M}D_{KL}\left(\phi^{*}_{m}|\phi\right),$ thus by definition of $D_{*}$, we have $D_{KL}(\phi^{*}_{m}|\phi^{*})\leq D_{*}$, so we have,
\begin{equation}\label{eq21}
    \frac{2(N+4)L_{F} }{\sqrt{T}}\sum_{m=1}^{M}\frac{D_{KL}\left(\phi^{*}_{m}|\phi^{*}\right)}{M}\leq \frac{2(N+4)L_{F} D_{*}}{\sqrt{T}}
\end{equation}
Now from the bounds in \eqref{eq20} and \eqref{eq21}, we can obtain the TAOG as:
\begin{equation}
    \bar{R}\leq 2(N+4)L_{F}\left(\frac{L_{g}^{2}\log(M+1)}{\mu_{\phi}M\sqrt{T}} +\frac{D_{*}}{\sqrt{T}}\right)
\end{equation}
\end{proof}
% \textbf{Proof.} We moved the detailed proof to appendix \ref{b21}.

\begin{remark}
Theorem \ref{TAOGwithoptimal} highlights the key advantage of incorporating multiple tasks, as the regret decay at a rate of $\mathcal{O}\left(\frac{\log (M+1)}{M}\right)$, benefiting from task similarity and improving upon single-task guarantees. However, this result relies on the assumption that the optimal policy~$\pi^{*}$ with parameters~$\phi^{*}$ is revealed after each task. In practical scenarios where this assumption does not hold, thus the plug-in estimator $\left\lbrace D_{KL}\left(\hat{\phi}_{m}|\cdot\right)\right\rbrace_{m=1}^{M},$ constructed using learned policy parameters $\hat{\phi}_{m}$ may be biased. This bias poses challenges for standard analyses of FTRL and OMD in the bandit setting. 
    % The above result reveals an interesting benefit brought by including more tasks, the regret decay at the rate of $\mathcal{O}\left(\frac{\log (M+1)}{M}\right)$ with more similarity, which improves upon single-task guarantee and serves as the initial point of our study. However, there are some limitations. For instance, if the optimal policy $\pi^{*}$ with parameters $\phi^{*}$ are not revealed after each task, it is not likely that the plug-in estimator $\left\lbrace D_{KL}\left(\hat{\phi}_{m}|\cdot\right)\right\rbrace_{m=1}^{M}$ with the learned policy parameters $\hat{\phi}_{m}$ is a unbiased estimator, ruling out existing analysis for FTRL or OMD in the bandit setting.
\end{remark} 
\subsection{Provable Guarantees for Practical MGF-RL}
In the previous section we present the analysis with access to best actions in hindsight $\phi^{*}_{m}$ for each task that can learn a good meta initialization or meta regularization. While $\phi^{*}_{m}$ is efficiently computable in some cases, often it is more practical to use approximation. one of the key steps to generalize the online-within-online methodology to relax the assumption of access the exact upper bounds of within-task performance by designing analysis to estimate and update their inexact versions.

Once a task $m$ is complete, the meta learner only has access to sub-optimal policy $\hat{\pi}$. Thereby we estimate $D_{KL}\left({\phi^{*}_{m}|\phi}\right)$ with $D_{KL}\left({\hat{\phi}_{m}|\phi}\right)$ by plugging in $\hat{\phi}$ from the with-in task MDP. The KL divergence estimation error bound is given as [Lemma 17; \cite{khattar2022cmdp}]: 
\begin{equation}\label{KLBound}
    D_{KL}\left({\phi^{*}_{m}|\phi}\right)-D_{KL}\left({\hat{\phi}_{m}|\phi}\right) \leq \mathcal{O}\left(h\left(\frac{1}{\sqrt{T}}\right)+\frac{1}{\sqrt{T}}\right)=\epsilon_{m},
\end{equation}
where $h$ is strictly increasing continuous function with property that $h(0)=0$. We define the cumulative inexactness $\boldsymbol{\epsilon}_{M}:=\sum_{m=1}^{M}\epsilon_{m}$. This quantity quantity decays with $T$ at the rate of $\mathcal{O}\left(h\left(\frac{1}{\sqrt{T}}\right)+\frac{1}{\sqrt{T}}\right)$ up to some approximation. 
\subsubsection{Static Regret Bounds}
With the above uniform bound on estimation error \eqref{KLBound}, our next step is to develop static regret bounds for the 
inexact implicit online mirror descent, which are used to furnish the upper bound on TAOG of the proposed inexact within online algorithm.
% \begin{theorem}[Static regret for implicit online update]
% Under the assumptions of Lemma \ref{lem4.1}, and if \( \eta_m \) is a decreasing sequence, the average static regret is bounded by:
%  \begin{equation}\label{sr2}
%  \begin{split}
%        \frac{1}{M}\sum_{m=1}^{M}l_{m}(\phi_{m,0})-l_{m}(\phi^{*}_{0})&\leq \frac{2}{M}\min \{ \sqrt{\beta \sum_{m=1}^{M}\mathbb{E}_{m}\left[g_{m}^{2}\right]}\} ,\\& (l_{1}(\phi_{1,0})-l_{M}(\phi_{M+1,0}) +V_{M})\},
%  \end{split}
%  \end{equation}
% where \( V_M(f) = \sum_{m=2}^{M} \max_{\phi_{m,0} \in \text{Dom}(f)} |l_m(\phi_{m,0}) - l_m(\phi_{m-1,0})| \) is the temporal variability of the loss function.
% \end{theorem}
The lemma below provides a bound on the static regret and will be used to prove the bound for implicit update.

\begin{lemma}[\cite{temporal}]\label{lem4.1}
Assuming the domain of the loss function is a non-empty closed convex set and the Bregman divergence is $\gamma $-Lipschitz continuous with $ D_b = \max_{a, b \in \text{Dom}(f)} B_{\psi}(a, b)$, let $\eta_m$ be a non-increasing sequence. Employing implicit online mirror descent or Follow The Regularized Leader (FTRL) on a sequence of loss functions $ \{l_m\}_{m=1}^M $ where $l_{m}(\phi_{m,0}):= D_{KL}\left(\phi^{*}_{m}|\phi_{m,0}\right)$, the static regret against a fixed comparator $\phi^*_0 $ is bounded by:
\begin{equation}
\begin{split}
\frac{1}{M}\sum_{m=1}^{M}l_m(\phi_{m,0}) &- l_m(\phi^*_0) \leq \frac{1}{M}\sum_{m=1}^{M} \delta_m\\&+ \frac{B_{\Psi}(\phi_{0}^{*},\phi_{m,0})- B_{\Psi}(\phi_{0}^{*},\phi_{m+1,0})}{\eta_m M} 
\end{split}
\end{equation}
where \( \delta_m = l_m(\phi_{m,0}) - l_m(\phi_{m+1,0}) - \frac{B_{\psi}(\phi_{m+1,0}, \phi_{m,0})}{\eta_m} \).
\end{lemma}
\begin{theorem}[Static regret for implicit online update]\label{stataicregretbound}
Let $l_{m}(\phi_{m,0}):= D_{KL}\left(\phi^{*}_{m}|\phi_{m,0}\right)$ for all $m\in [M]$. For any fixed comparator $\phi^{*}_{0}$ $=\arg\min \sum_{m=1}^{M}l_{m}(\phi_{0})$, if IOMD is run a sequence of loss function $\left\lbrace\hat{l}_{m}\right\rbrace_{m=1}^{M}$, where $\hat{l}_{m}(\phi_{m,0}):= D_{KL}\left(\hat{\phi}_{m}|\phi_{m,0}\right)$ for all $m\in [M]$. Let \( g_m' \in \partial l_m(\phi_{m+1,0}) \) be a subgradient of the loss at \( \phi_{m+1,0} \). Let \( B_{\Psi} \) denote the Bregman divergence with respect to a 1-strongly convex function \( \Psi: \text{dom}(\Psi) \to \mathbb{R} \) under the norm \( \|\cdot\| \). Then, for any learning rates \( \{\eta_m\}_{m=1}^{M} \), the static regret is bounded by:
\begin{equation}\nonumber
    \begin{split}
        \sum_{m=1}^{M}l_{m}\left(\phi_{m,0}\right)-&\sum_{m=1}^{M}l_{m}\left(\phi^{*}_{0}\right)\leq \sum_{m=1}^{M}2\eta_{m}\|g_{m}\|_{*}\|g'_{m}\|_{*}\\&+\sum_{m=1}^{M}\frac{B_{\Psi}\left(\phi^{*}_{0},\phi_{m,0}\right)-B_{\Psi}\left(\phi^{*}_{0},\phi_{m+1,0}\right)}{\eta_{m}}
    \end{split}
\end{equation}
% \begin{equation}  \sum_{m=1}^{M}\hat{l}_{m}\left(\phi_{m,0}\right)-\sum_{m=1}^{M}\hat{l}_{m}\left(\phi^{*}_{0}\right)\leq \sum_{m=1}^{M}\frac{B_{\Psi}\left(\phi^{*}_{0},\phi_{m,0}\right)-B_{\Psi}\left(\phi^{*}_{0},\phi_{m+1,0}\right)}{\eta_{m}}+\sum_{m=1}^{M}2\eta_{m}\|g_{m}\|_{*}\|g'_{m}\|_{*}
% \end{equation}
For learning rate $\eta_{m}=\eta$ for all $m\in [M]$, we have
\begin{equation}\nonumber
\begin{split}
    \sum_{m=1}^{M}l_{m}\left(\phi_{m,0}\right)-\sum_{m=1}^{M}l_{m}\left(\phi^{*}_{0}\right)&\leq l_{1}\left(\phi_{1,0}\right)-l_{M}\left(\phi_{M+1,0})\right) \\&+V_{M}+\frac{B_{\Psi}(\phi^{*}_{0},\phi_{1,0})}{\eta}
\end{split}
\end{equation}
where \( V_M(l_{m}) = \sum_{m=2}^{M} \max_{\phi_{m,0} \in \text{Dom}(l)} |l_m(\phi_{m,0}) - l_{m-1}(\phi_{m,0})| \) is the temporal variability of the loss function.
\end{theorem}
\begin{proof} Using the convexity of losses, we can bound the difference between $l_{m}(\phi_{m,0})$ and $l_{m}(\phi_{m+1,0})$ 
    \begin{equation}
    \begin{split}
        l_{m}(\phi_{m,0})- l_{m}(\phi_{m+1,0}) &\leq \langle g_{m}, \phi_{m,0}-\phi_{m+1,0}\rangle \\&\leq \|g_{m}\|_{*}\|\phi_{m,0}-\phi_{m+1,0}\|
    \end{split}
    \end{equation}
    where, $g_{m}\in \partial l_{m}(\phi_{m,0})$. Given that $\Psi$ is 1-strongly convex, we can use $B_{\Psi}(a,b)\geq \frac{1}{2}\|a-b\|^{2}$ for all $a,b \in \text{int}(X)$. Thus we have 
    \begin{equation}
        l_{m}(\phi_{m,0})- l_{m}(\phi_{m+1,0}) \leq \|g_{m}\|_{*}\sqrt{2B_{\Psi}(\phi_{m+1,0},\phi_{m,0})}.
    \end{equation}
    Note that 
    \begin{equation}
    \begin{split}
l_{m}(\phi_{m,0})- l_{m}(\phi_{m+1,0})&-\frac{B_{\Psi}(\phi_{m+1,0},\phi_{m,0})}{\eta_{m}}\\&\leq l_{m}(\phi_{m,0})- l_{m}(\phi_{m+1,0}).
    \end{split}
    \end{equation}
Hence, we have 
    \begin{equation}\label{eq27}
    \begin{split}
l_{m}(\phi_{m,0})- l_{m}(\phi_{m+1,0})&-\frac{B_{\Psi}(\phi_{m+1,0},\phi_{m,0})}{\eta_{m}}\\&\leq \|g_{m}\|_{*}\sqrt{2B_{\Psi}(\phi_{m+1,0},\phi_{m,0})}         
    \end{split}
    \end{equation}
    Now we simply needs to bound $\sqrt{2B_{\Psi}(\phi_{m+1,0},\phi_{m,0})}$. Using the fact that Bregman divergence is convex in its first argument, we get 
    \begin{equation}\label{eq28}
    \begin{split}
        2B_{\Psi}&(\phi_{m+1,0},\phi_{m,0}) \\&\leq \langle \nabla \Psi(\phi_{m+1,0})-\nabla \Psi(\phi_{m,0}),\phi_{m+1,0}-\phi_{m,0}\rangle\\&\leq \langle \eta_{m}g'_{m},\phi_{m,0}-\phi_{m+1,0}\rangle \langle \eta_{m}\|g'_{m}\|_{*}\|\phi_{m+1,0}-\phi_{m,0}\|\\&\leq \eta_{m}\|g'_{m}\|_{*}\sqrt{2B_{\Psi}(\phi_{m+1,0},\phi_{m,0})}
    \end{split} 
    \end{equation}
    Now assume $a=\sqrt{2B_{\Psi}(\phi_{m+1,0},\phi_{m,0})}$, then $\frac{a^{2}}{2}=B_{\Psi}(\phi_{m+1,0},\phi_{m,0})$, with that inequality \eqref{eq28}, takes the form 
    \begin{equation}
        \frac{a^{2}}{2}-\eta_{m}\|g'_{m}\|_{*}a\leq 0; \implies a (a-\eta_{m}\|g'_{m}\|_{*})\leq 0
    \end{equation}
    Since $a\leq 0 $, so  $(a-\eta_{m}\|g'_{m}\|_{*})\leq 0$. After solving for $a$, we have $\sqrt{2B_{\Psi}(\phi_{m+1,0},\phi_{m,0})}\leq 2\eta_{m}\|g'_{m}\|_{*}$.
Now \eqref{eq27}, takes the bound 
\begin{equation} \label{eq3.34}
\begin{split}
    l_{m}(\phi_{m,0})- l_{m}(\phi_{m+1,0})&-\frac{B_{\Psi}(\phi_{m+1,0},\phi_{m,0})}{\eta_{m}}\\&\leq 2\eta_{m}\|g_{m}\|_{*}\|g'_{m}\|_{*} 
\end{split}
    \end{equation}
    
    Now by Lemma \eqref{lem4.1}, we have
\begin{equation}
\begin{split}
 \sum_{m=1}^{M} l_m(\phi_{m,0}) &- \sum_{m=1}^{M}l_m(\phi^*_0) \leq\sum_{m=1}^{M} 2\eta_{m}\|g_{m}\|_{*}\|g'_{m}\|_{*}  \\&+ \sum_{m=1}^{M}\frac{B_{\Psi}(\phi_{0}^{*},\phi_{m,0})- B_{\Psi}(\phi_{0}^{*},\phi_{m+1,0})}{\eta_m },
\end{split}
\end{equation}
Now when the losses are not vary much over sequence of tasks $m$, we can use the notion of temporal variability $V_M$ of losses, for $\eta_{m}=\eta$ for all $m\in [M]$, we immediately have from bound in Lemma \eqref{lem4.1},
\begin{equation}\nonumber
\begin{split}
    \sum_{m=1}^{M}&l_m(\phi_{m,0}) - \sum_{m=1}^{M}l_m(\phi^*_0) \\&\leq \sum_{m=1}^{M}\frac{B_{\Psi}(\phi_{0}^{*},\phi_{m,0})- B_{\Psi}(\phi_{0}^{*},\phi_{m+1,0})}{\eta_m } + \sum_{m=1}^{M} \delta_{m},\\&\leq
    \frac{B_{\Psi}(\phi^{*}_{0},\phi_{1,0})}{\eta} + \sum_{m=1}^{M} l_m(\phi_{m,0}) - l_m(\phi_{m+1,0}) \\&\quad- \frac{B_{\psi}(\phi_{m+1,0}, \phi_{m,0})}{\eta_m} \\&
    \leq
    \frac{B_{\Psi}(\phi^{*}_{0},\phi_{1,0})}{\eta} +l_{1}(\phi_{1,0})-l_{M}(\phi_{M+1,0}) \\&\quad+ \sum_{m=2}^{M}\max_{\phi} l_m(\phi_{m,0}) - l_m(\phi_{M+1,0}) - \frac{B_{\psi}(\phi_{m+1,0}, \phi_{m,0})}{\eta_m}\\& \leq \frac{B_{\Psi}(\phi^{*}_{0},\phi_{1,0})}{\eta} +l_{1}(\phi_{1,0})-l_{M}(\phi_{M+1,0})+V_{M}
\end{split}
\end{equation}
    
\end{proof}
% We moved the detailed proof to  appendix \ref{b31}.
\begin{remark} The result above indicates that a constant learning rate can yield a regret bound of $\mathcal{O}(V_{M} + 1)$, which outperform the typical bound of $\mathcal{O}(\sqrt{M})$ when temporal variability is low. Even with high temporal variability, using a learning rate of $\mathcal{O}\left(\frac{1}{\sqrt{M}}\right)$ still achieves a worst-case regret of $\mathcal{O}(\sqrt{M})$. It is worth noting that this behavior also appears in the \textit{Follow the Regularized Leader} (FTRL) algorithm when applied to full losses, as opposed to linearized losses. \end{remark}

After establishing the static regret of online update, we can derive upper bounds on the TAOG for the proposed MGF-RL algorithm based on the empirical task similarity, denoted by $\hat{D}_{*}$.  

% \begin{theorem}[Static Regret for Implicit Online Updates]
% Let \( l_m(\phi_{m,0}) := D_{KL}(\phi_m^* \| \phi_{m,0}) \) for each \( m \in [M] \), where \( \phi_m^* \) is the optimal solution for the \( m \)-th round. Define \( \phi_0^* := \arg\min_{\phi_0} \sum_{m=1}^{M} l_m(\phi_{m,0}) \). Assume the Implicit Online Mirror Descent (IOMD) algorithm is applied to a sequence of loss functions \( \{\hat{l}_m\}_{m=1}^{M} \), where \( \hat{l}_m(\phi_{m,0}) := D_{KL}(\hat{\phi}_m \| \phi_{m,0}) \) for each \( m \in [M] \).

% Let \( g_m' \in \partial l_m(\phi_{m+1,0}) \) be a subgradient of the loss at \( \phi_{m+1,0} \). Let \( B_{\Psi} \) denote the Bregman divergence with respect to a 1-strongly convex function \( \Psi: \text{dom}(\Psi) \to \mathbb{R} \) under the norm \( \|\cdot\| \). Then, for any learning rates \( \{\eta_m\}_{m=1}^{M} \), the static regret is bounded by:
% \[
% \sum_{m=1}^{M} l_m(\phi_{m,0}) - \sum_{m=1}^{M} l_m(\phi_0^*) 
% \leq \sum_{m=1}^{M} \frac{B_{\Psi}(\phi_0^*, \phi_{m,0}) - B_{\Psi}(\phi_0^*, \phi_{m+1,0})}{\eta_m}
% + \sum_{m=1}^{M} 2\eta_m \|g_m\|_* \|g_m'\|_*
% \]
% where \( V_M(l_m) = \sum_{m=2}^{M} \max_{\phi_{m,0} \in \text{Dom}(l)} |l_m(\phi_{m,0}) - l_m(\phi_{m-1,0})| \) captures the temporal variability of the loss.
% \end{theorem}
\noindent\textbf{Proof of Theorem \eqref{TAOGstatic}}
\begin{theorem}
Let $\phi^{*}$ denote the fixed meta initialization for all the tasks given by $\phi^{*}= \arg\min_{\phi\in \Delta(\mathcal{A})^{|\mathcal{S}|}}\frac{1}{M}\sum_{m=1}^{M}D_{KL}\left(\hat{\phi}_{m}|\phi\right)$. Within each task we run ES-RL for $T$ iterations and we obtain $\left\lbrace\hat{\phi}_{m,0}\right\rbrace_{m=1}^{M}$. Let $B_{\Psi}$ be the Bregman divergence w.r.t $\Psi:X\to \mathbb{R} $ and let $B^{2}_{*};= \max_{a,b\in X} B_{\Psi}(a,b)$. Furthermore the initialization $\left\lbrace \phi_{m,0}
\right\rbrace_{m=1}^{M}$ are determined by applying IOMD or FTRL on the function $D_{KL}\left(\hat{\phi}_{m,0}|\cdot\right)$ for $m=1,\dots, M$. Then we have that  
% \begin{equation}
%     \bar{R}\leq \mathcal{O}\left(\frac{1}{\sqrt{T}}\left(\hat{D}_{*}+\frac{V_{M}}{M}+\frac{B^{2}_{*}}{M}+\frac{\epsilon_{M}}{M}\right)\right).
% \end{equation}
\begin{equation}
    \bar{R}\leq 2(N+4)L_{F}\left(\frac{B^{2}_{*}}{\eta M\sqrt{T}}+\frac{ V_{M}}{ M\sqrt{T}} +\frac{3\hat{D}_{*}}{\sqrt{T}}+\frac{\boldsymbol{\epsilon}_{M}}{M\sqrt{T}}\right).
\end{equation}
\end{theorem}
\begin{proof}
    By the within-task gurantee for MDPs, we know that the task average optimality gap is well defined and bounded by  
 \begin{equation}\label{eq33}
 \begin{split}
     \bar{R}&\leq \frac{2(N+4)L_{F} \sum_{m=1}^{M}D_{KL}\left(\phi^{*}_{m}|\phi_{m,0}\right)}{M\sqrt{T}}\\& \leq \frac{2(N+4)L_{F} }{\sqrt{T}}\sum_{m=1}^{M}\frac{D_{KL}\left(\phi^{*}_{m}|\phi_{m,0}\right)-D_{KL}\left(\phi^{*}_{m}|\phi^{*}\right)}{M}\\&\quad+\frac{2(N+4)L_{F} }{\sqrt{T}}\sum_{m=1}^{M}\frac{D_{KL}\left(\phi^{*}_{m}|\phi^{*}\right)}{M}.
 \end{split}
\end{equation} 
Now by KL divergence estimation error bounds in \eqref{KLBound}, the inequality \eqref{eq33} takes the form.
\begin{equation}\label{eq34}
\begin{split}
    \bar{R}\leq& \frac{2(N+4)L_{F} }{\sqrt{T}}\sum_{m=1}^{M}\frac{D_{KL}\left(\phi^{*}_{m}|\phi_{m,0}\right)-D_{KL}\left(\phi^{*}_{m}|\phi^{*}\right)}{M}\\&+\frac{2(N+4)L_{F} }{\sqrt{T}}\sum_{m=1}^{M}\frac{D_{KL}\left(\hat{\phi}_{m}|\phi^{*}\right)+\boldsymbol{\epsilon}_{m}}{M} 
\end{split}
\end{equation}
Since $\phi^{*}=\arg\min_{\phi}\sum_{m=1}^{M}D_{KL}\left(\hat{\phi}_{m}|\phi\right)$
, thus by definition of $\hat{D}_{*}$, we have we have,
\begin{equation}\label{eq35}
\begin{split}
\frac{2(N+4)L_{F} }{\sqrt{T}}&\sum_{m=1}^{M}\frac{D_{KL}\left(\hat{\phi}_{m}|\phi^{*}\right)+\boldsymbol{\epsilon}_{m}}{M}\\&\quad\leq \frac{2(N+4)L_{F} \left(\hat{D}_{*}+\frac{\boldsymbol{\epsilon}_{M}}{M}\right)}{\sqrt{T}}    
\end{split}
\end{equation}

Now we will upper bound the first term of \eqref{eq34}, since each $\phi_{m,0}$ is determined by playing IOMD or FTRL, then the following term can be upper bounded by the theorem \eqref{stataicregretbound}

\begin{equation}\label{eq36}
\begin{split}
    \frac{2(N+4)L_{F} }{\sqrt{T}}&\sum_{m=1}^{M}\frac{D_{KL}\left(\phi^{*}_{m}|\phi_{m,0}\right)-D_{KL}\left(\phi^{*}_{m}|\phi^{*}\right)}{M} \\&\leq\frac{2(N+4)L_{F} }{M\sqrt{T}} (\frac{B_{\Psi}(\phi^{*}_{0},\phi_{1,0})}{\eta} +D_{KL}\left(\phi^{*}_{m}|\phi_{1}\right)\\&\quad-D_{KL}\left(\phi^{*}_{m}|\phi_{M+1}\right)+V_{M})\\& \leq \frac{2(N+4)L_{F} }{M\sqrt{T}} \left[\frac{B^{2}_{*}}{\eta} +V_{M}\right]+\frac{4(N+4)\hat{D}_{*}L_{F} }{\sqrt{T}} 
\end{split}
\end{equation}
where $B^{2}_{*}:= \max_{\phi^{*}_{m},\phi}B_{\Psi}(\phi^{*}_{m},\phi)$. Now from the bounds in \eqref{eq35}and \eqref{eq36}, we can obtain the TAOG as:
\begin{equation}
    \bar{R}\leq 2(N+4)L_{F}\left(\frac{B^{2}_{*}}{\eta M\sqrt{T}}+\frac{ V_{M}}{ M\sqrt{T}} +\frac{3\hat{D}_{*}}{\sqrt{T}}+\frac{\boldsymbol{\epsilon}_{M}}{M\sqrt{T}}\right).
\end{equation}
\end{proof}

% Theorem \ref{TAOGstatic} establishes a regret bound that preserves the benefits of task similarity within meta-reinforcement learning. However, there is a trade-off as the bound introduces an additional cost term associated with inexactness, denoted by $\epsilon_{M}$, and a dependency on the number of tasks $M$ that is less favorable than the regret bound presented in equation \eqref{eq13}. Notably $\epsilon_{M}/M$ decreases, either by increasing the number of iterations $T$ within each task or by optimizing the meta-initialization $\phi^{*}$ to require minimal task-specific adaptation, the inexactness term's influence on overall performance is expected to be reduced. This suggests that improving task alignment or selecting a well-suited initialization can effectively mitigate the additional cost from $\epsilon_{M}$.
%%%%%%%%

% In many setting, we have a changing environment, so it is natural ti study dynamic regret and compare with a sequence of potential time-varying initial policies $\left\lbrace\pi^{*}_{m,0}\right\rbrace_{m=1}^{M}$ with parameters $\left\lbrace\phi^{*}_{m,0}\right\rbrace_{m=1}^{M}$. To measure task similarity in this case, we define task-relatedness which can be measured as . This notion of task-reltaedness gives the measure of how far optimal 
\subsubsection{Dynamic Regret Analysis}

In many setting, we have a changing environment, so it's natural to study dynamic regret and to compare performance against a sequence of time-varying initial policies $\left\lbrace\pi^{*}_{m,0}\right\rbrace_{m=1}^{M}$ with parameters $\left\lbrace\varphi^{*}_{m,0}\right\rbrace_{m=1}^{M}$. To measure task similarity in such cases, we define a notion of task-relatedness, quantified by $S_{*}:=\frac{1}{M}\sum\limits_{m=1}^{M}D_{KL}\left(\phi^{*}_{m}|\varphi^{*}_{m,0}\right)$. This task-relatedness metric provides a measure of the divergence between the optimal policy for each task and a time-varying comparator. We denote the empirical task-relatedness as $\hat{S}_{*} := \frac{1}{M}\sum\limits_{m=1}^{M}D_{KL}\left(\hat{\phi}_{m}|\varphi^{*}_{m,0}\right)$, which depends on the suboptimal policy produced by the within-task learning algorithm. To analyze the performance of Proposed MGF-RL in dynamic settings, we consider its dynamic regret bound, given by: 
\begin{equation}\label{eq18}
    L_{M}\left(\varphi^{*}_{1:M}\right):=\sum_{m=1}^M l_m(\phi_{m,0}) - \sum_{m=1}^M l_m(\varphi_{m}^*),
\end{equation} 
where $\varphi_{m,0}^* \in \arg\min_{\varphi\in \Delta(\mathcal{A})} l_m(\varphi)$ represents the optimal initial policy for each task, and the function $l_m(\cdot) = \mathbb{E} \left[ D_{\mathrm{KL}}(\phi_m^* | \cdot) \right]$ quantifies the regret.

% It can be shown that it is impossible to achieve sublinear dynamic regret in the worst-case. However, if one puts some restriction on the sequence
While achieving sublinear dynamic regret in the worst-case scenario is impossible, however, certain regularity assumptions on the sequence~$\phi^{*}_{1:M,0}$ allow us to derive meaningful upper bounds~\eqref{eq18} in M. 
% There are various measures which can be used to model the regularity of the environment.   
% By leveraging the strong convexity of the loss function (in this case, the KL divergence), previous studies have shown that the dynamic regret can be upper-bounded by the path length of the comparator sequence, 
Various measures can model the regularity of the environment, guiding the analysis of sublinear regret behavior. One widely used measure is the \textbf{path length} of the comparator sequence, defined as $P_M := \sum_{m=2}^M \|\phi_{m,0}^* - \phi_{m-1}^*\|$, which captures the cumulative variation between successive comparators \cite{zinkevich2003online,zhao2020dynamic}. For strongly convex loss functions (in this case, the KL divergence), this bound can be further improved using the \textbf{squared path length}, defined as $S_M := \sum_{m=2}^M \|\phi_{m,0}^* - \phi_{m-1}^*\|^2$, which is often significantly smaller than the direct path length \cite{khattar2022cmdp}. We use another measure of non-stationary given by the \textbf{temporal variability} of the loss function, defined as  \( V_M(l_{m}) = \sum_{m=2}^{M} \max_{\phi_{m,0} \in \text{Dom}(l)} |l_m(\phi_{m,0}) - l_{m-1}(\phi_{m,0})| \) in \cite{temporal}. These measures provide key insights into how dynamically changing environments impact the learning process.

We extend these results to the setting of inexact online gradient descent, allowing the learner to query an inexact gradient of the function.

\noindent\textbf{Proof of Theorem \ref{dynamic regret}}
\begin{theorem}[Dynamic regret for online learning]
    Let $l_{m}(\phi_{m,0}):= D_{KL}\left(\phi^{*}_{m}|\phi_{m,0}\right)$ for all $m\in [M]$. For any dynamical varying comparator $\varphi^{*}_{m}$ $=\arg\min \sum_{m=1}^{M}l_{m}(\varphi)$, if IOMD is run a sequence of loss function $\left\lbrace\hat{l}_{m}\right\rbrace_{m=1}^{M}$, where $\hat{l}_{m}(\phi_{m,0}):= D_{KL}\left(\hat{\phi}_{m}|\phi_{m,0}\right)$ for all $m\in [M]$. Let \( g_m' \in \partial l_m(\phi_{m+1,0}) \) be a subgradient of the loss at \( \phi_{m+1,0} \). Assum \( B_{\Psi} \) denote the Bregman divergence with respect to a 1-strongly convex function \( \Psi: \text{dom}(\Psi) \to \mathbb{R} \) under the norm \( \|\cdot\| \). Let there exists $\gamma \in \mathbb{R}$ such that $B_{\Psi}(x,z)-B_{\Psi}(y,z)\leq \gamma\|x-y\|,\quad \forall \text{dom}(\Psi)$. Then, for any learning rates \( \{\eta_m\}_{m=1}^{M} \), the static regret is bounded by:
\begin{equation}\nonumber
    \begin{split}
L_{M}=\sum_{m=1}^{M}l_{m}\left(\phi_{m,0}\right)&-\sum_{m=1}^{M}l_{m}\left(\varphi^{*}_{m}\right)\leq \sum_{m=1}^{M}\frac{\|\varphi^*_{m}-\varphi^*_{m-1}\|}{\eta_{m}}\\&+\frac{B^2_{*}}{\eta_{m}}+\sum_{m=1}^{M}2\eta_{m}\|g_{m}\|_{*}\|g'_{m}\|_{*}
    \end{split}
\end{equation}
\end{theorem}
\begin{proof}
    By the convexity of loss function, we have
    \begin{equation}\label{eq6.42}
    \begin{split}\eta_{m}&\left(l_{m}\left(\phi_{m+1,0}\right)-l_{m}\left(\varphi^*_{m}\right)\right) \leq \langle \eta_{m} g'_{m}, \phi_{m+1,0}-\varphi^{*}_{m}\rangle \\& \quad\quad\leq \langle \nabla \Psi (\phi_{m,0})-\nabla \Psi (\phi_{m+1,0}), \Psi_{m+1,0}-\varphi^{*}_m\rangle  \\&\quad \quad= B_{\Psi}(\varphi^*_m,\phi_{m,0})-  B_{\Psi}(\varphi^*_m,\phi_{m+1,0})\\&\quad\quad\quad- B_{\Psi}(\phi_{m+1,0},\phi_{m,0})
    \end{split}
    \end{equation}
where the second inequality in \eqref{eq6.42} follows from first order optimality condition. Now we consider the first two term of 
\eqref{eq6.42}.
\begin{equation}
    \begin{split}
        \sum_{m=1}^{M}&\frac{\left(B_{\Psi}(\varphi^*_m,\phi_{m,0})-  B_{\Psi}(\varphi^*_m,\phi_{m+1,0})\right)}{\eta_{m}}\\&\leq \frac{\left(B_{\Psi}(\varphi^*_1,\phi_{1,0})-  B_{\Psi}(\varphi^*_1,\phi_{2,0})\right)}{\eta_{1}}\\& \quad\quad\quad\quad\quad+\sum_{m=2}^{M}\frac{\left(B_{\Psi}(\varphi^*_m,\phi_{m,0})-  B_{\Psi}(\varphi^*_m,\phi_{m+1,0})\right)}{\eta_{m}}
    \end{split}
\end{equation}
Now by definition of $B^2_*= \max_{x,y\in dom(\Psi)}B_{\Psi}(x,y)$, we have
\begin{equation}
    \begin{split}
        \sum_{m=1}^{M}&\frac{B_{\Psi}(\varphi^*_m,\phi_{m,0})-  B_{\Psi}(\varphi^*_m,\phi_{m+1,0})}{\eta_{m}}\\&\quad\quad\leq \frac{B^{2}_{*}}{\eta_{1}}+\sum_{m=2}^{M}\left(\frac{B_{\Psi}(\varphi^*_m,\phi_{m,0})}{\eta_{m}}- \frac{B_{\Psi}(\varphi^*_{m-1},\phi_{m,0})}{\eta_{m-1}}\right) \\& = \frac{B^{2}_{*}}{\eta_{1}}+\sum_{m=2}^{M}\left(\frac{B_{\Psi}(\varphi^*_m,\phi_{m,0})}{\eta_{m}}- \frac{B_{\Psi}(\varphi^*_{m-1},\phi_{m,0})}{\eta_{m}} \right)\\& \quad\quad+ \sum_{m=2}^{M}\left(\frac{B_{\Psi}(\varphi^*_{m-1},\phi_{m,0})}{\eta_{m}}- \frac{B_{\Psi}(\varphi^*_{m-1},\phi_{m,0})}{\eta_{m-1}} \right)\\& \leq \frac{B^{2}_{*}}{\eta_{1}}+\gamma \sum_{m=2}^{M}\left(\frac{\|\varphi^*_m-\varphi^*_{m-1}}{\eta_{m}}\right)\\&\quad+\sum_{m=2}^{M}\left(\varphi^*_{m-1}-\phi_{m,0}\right)\left(\frac{1}{\eta_{m}}-\frac{1}{\eta_{m-1}}\right) \end{split}
\end{equation}
thus we have
\begin{align}
    \sum_{m=1}^{M}&\frac{B_{\Psi}(\varphi^*_m,\phi_{m,0})-  B_{\Psi}(\varphi^*_m,\phi_{m+1,0})}{\eta_{m}}\\& \leq \frac{B^{2}_{*}}{\eta_{1}}+\gamma \sum_{m=2}^{M}\left(\frac{\|\varphi^*_m-\varphi^*_{m-1}\|}{\eta_{m}}\right)+\left(\frac{B^{2}_{*}}{\eta_{m}}-\frac{B^{2}_{*}}{\eta_{m-1}}\right)
\end{align}
   
Now adding $l_{m}(\phi_{m,,0})$ on both sides of \eqref{eq6.42} and summing over all tasks $m\in [M]$, we have 
\begin{equation}\label{eq6.45}
    \begin{split}
        \sum_{m=1}^{M} l_{m}(\phi_{m,0})&-l_{m}\left(\varphi^{*}_{m}\right) \leq \frac{B^{2}_{*}}{\eta_{M}}  +\gamma \sum_{m=2}^{M}\frac{\|\varphi^*_m-\varphi^*_{m-1}\|}{\eta_{m}}\\& +\sum_{m=1}^{M}\left(l_{m}(\phi_{m,0})- l_{m}(\phi_{m+1,0})-\frac{B_{\Psi}(\phi_{m+1,0},\phi_{m,0})}{\eta_{m}}\right) 
    \end{split}
\end{equation}
From \eqref{eq3.34} we can bound $\sum_{m=1}^{M} l_{m}(\phi_{m,0})- l_{m}(\phi_{m+1,0})-\frac{B_{\Psi}(\phi_{m+1,0},\phi_{m,0})}{\eta_{m}}\leq \sum_{m=1}^{M}2\eta_{m}\|g_{m}\|_{*}\|g'_{m}\|_{*} $. Thus \eqref{eq6.45} is bounded as: 
\begin{equation}
    \begin{split}
        \sum_{m=1}^{M} l_{m}(\phi_{m,0})-l_{m}\left(\varphi^{*}_{m}\right) &\leq \frac{B^{2}_{*}}{\eta_{M}} +\sum_{m=1}^{M}2\eta_{m}\|g_{m}\|_{*}\|g'_{m}\|_{*}\\&+\gamma \sum_{m=2}^{M}\frac{\|\varphi^*_m-\varphi^*_{m-1}\|}{\eta_{m}} 
    \end{split}
\end{equation}
\end{proof}
% \textbf{Proof.} We moved the detailed proof to appendix \ref{b41}.\\

We notice that the Lipschitz continuity assumption is not a strong requirement. Indeed, when the function $\Psi$ is Lipschitz, the Lipschitz condition on Bregman divergence is automatically satisfied. It can be observed from last theorem that to set the learning rate fixed makes the algorithm less applicable in online settings where tasks are encountered sequentially. Moreover when the task-environment changes dynamically, a fixed initialization may not be the best candidate comparator, where it is natural to study dynamic regret by competing with a potential time varying sequences $\left\lbrace\varphi^*_m\right\rbrace_{m=1}^{M}$. Also, the tasks may share some common aspects of the optimization landscape,
so adapting learning rates based on prior experience may further improve performance. 

\noindent\textbf{proof of Theorem \ref{adaptivelr}}
\begin{theorem}[Regret with adaptive learning rate]
 Let $\mathcal{C}_{1}\geq 0$ be positive constant. Under the assumption of Theorem \ref{dynamic regret} for dynamic sequence of comparator $\varphi^{*}_{m}$ whose path-length $P_{M}(\varphi^{*}_{m})\leq \mathcal{C}_{1}$ for all $m\in [M]$. Then with decreasing learning rate $\eta_{m}=\frac{1}{\lambda_{m}}=\frac{\beta^{2}}{\sum_{m=1}^{M}\delta_{m}}$ and $\beta^{2}=\left(D^2_{b}+\gamma {P}_{M}\right)$ incur the dynamic regret against upper bounded as 
 \begin{equation}\nonumber
     L_{M}\left(\varphi^{*}_{1:M}\right) \leq  \mathcal{O} \left(\min\left\lbrace V_M, 2\sqrt{3B^{2}_{*}+\gamma \mathcal{C}_{1} \sum_{m=1}^{M}\|g_{m}\|^2_{*}}\right\rbrace\right)
 \end{equation}

\end{theorem}
\begin{proof}
    Let $\delta_{m}:= l_{m}(\phi_{m,0})-l_{m}(\phi_{m+1,0})-\frac{B_{\Psi(\phi_{m+1,0},\phi_{m,0})}}{\eta_{m}}$. Then for decreasing learning rate $\eta_{m}=\frac{1}{\lambda_{m}}=\frac{\beta^{2}}{\sum_{m=1}^{M-1}\delta_{m}}$ from \eqref{eq6.45}, we have 
    \begin{equation}
        \begin{split}
            \sum_{m=1}^{M} l_{m}(\phi_{m,0})-l_{m}\left(\varphi^{*}_{m}\right) &\leq \frac{B^{2}_{*}}{\eta_{M}} +\sum_{m=1}^{M}\delta_{m}+\gamma \sum_{m=2}^{M}\frac{\|\varphi^*_m-\varphi^*_{m-1}\|}{\eta_{m}} \\& \leq \lambda_{M+1}B^{2}_{*} \\&+\gamma\sum_{m=2}^{M} \lambda_{m}\|\varphi^*_{m}-\varphi_{m-1}^{*}\| +\beta^{2}\lambda_{M+1} 
        \end{split}
    \end{equation}
    The above inequality follows from the fact the $\left\lbrace\lambda_{m}\right\rbrace_{m=1}^{M}$ is an increasing sequence. And by definition of path-length ${P}_{M}:= \sum_{m=2}^{M}\|\varphi_{m}^{*}-\varphi_{m-1}^{*}\|$, we have
  \begin{equation}\label{4.48}
        \begin{split}
            \sum_{m=1}^{M} l_{m}(\phi_{m,0})-l_{m}\left(\varphi^{*}_{m}\right) &\leq \lambda_{M+1}B^{2}_{*} +\gamma {P}_{M}\lambda_{M+1}\\&\quad+\beta^{2}\lambda_{M+1} \\&=\left(B^{2}_{*}+\gamma {P}_{M}+\beta^{2}\right)\lambda_{M+1}\\ & \leq \left(B^{2}_{*}+\gamma \mathcal{C}_{1}+\beta^{2}\right)\lambda_{M+1}
        \end{split}
    \end{equation}  
    The rest of proof is similar to the one in \cite{temporal}. Now by definition, from the choice of $\lambda_{m}$, we have
    \begin{equation}
        \begin{split}
\beta^{2}&\lambda_{M+1}=\sum_{m=1}^{M}\delta_{m} \\&\quad=\sum_{m=1}^{M}l_{m}(\phi_{m,0})-l_m(\phi_{m+1,0})-\lambda_{m}B_{\Psi}(\phi_{m+1,0},\phi_{m,0}) 
        \end{split}
    \end{equation}
Since $B_{\Psi}(\phi_{m+1,0},\phi_{m,0}) \geq 0$, so
 \begin{equation}
 \begin{split}
     \beta^{2}\lambda_{M+1} &\leq\sum_{m=1}^{M}l_{m}(\phi_{m,0})-l_m(\phi_{m+1,0})\\& \leq l_{1}(\phi_{1,0})-l_{M}(\phi_{M+1,0})\\&+ \sum_{m=2}^{M}\max_{\phi\in\text{dom}(\Psi)}l_{m}(\phi)-l_{m-1}(\phi) \\& \leq l_{1}(\phi_{1,0})-l_{M}(\phi_{M+1,0}) +V_{M}
 \end{split} 
    \end{equation}
    Therefore the dynamic regret in \eqref{4.48} will be
    \begin{equation}
    \begin{split}
        \sum_{m=1}^{M}& l_{m}(\phi_{m,0})-l_{m}(\varphi^{*}_{m})\\&\quad\leq \left(B^{2}_{*}+\gamma\mathcal{C}_{1}+\beta^{2}\right) \frac{l_1{\phi_{1,0}}-l_{M}(\phi_{M+1,0})+V_{M}}{\beta^{2}}
    \end{split}
    \end{equation}
    On the other hand, from the definition of $\delta_{m}$, we have $\delta_{m}\leq l_{m}(\phi_{m,0})-l_{m}(\phi_{m+1,0})\leq \langle g_{m}, \phi_{m,0}-\phi_{m+1,0}\rangle$. By assumption in Theorem \ref{TAOGstatic}, we have $B^2_{*}\geq B_{\Psi}(a,b)\geq \frac{\|a-b\|^{2}}{2}$, which implies $\|a-b\|^{2}\leq \sqrt{2}B_{*}$. Therefore, $\delta_{m}\leq \sqrt{2}B_{*}\|g_{m}\|_{*}$. Thus, when $\beta^{2}= B^{2}_{*}+\gamma \mathcal{C}_{1}$, we have
    \begin{equation}
    \begin{split}
    \sum_{m=1}^{M}& l_{m}(\phi_{m,0})-l_{m}(\varphi^{*}_{m}) \\&\leq  \mathcal{O} \left(\min\left\lbrace V_M, 2\sqrt{3B^{2}_{*}+\gamma \mathcal{C}_{1} \sum_{m=1}^{M}\|g_{m}\|^2_{*}}\right\rbrace\right).
    \end{split}
    \end{equation}
\end{proof}
% \textbf{Proof.}We moved the detailed proof to  appendix \ref{B42}.\\

% \begin{remark}
%     If we assume an upper bound $\|g_{m}\|_{*}^{2}\leq \max_{m\in[M]}\|g_{m}\|_{*}^{2}\leq 1$ and that $\gamma=B_{*}=1$, then the above result give us the dynamic regret bound of $\mathcal{O} \left( \min \left\lbrace V_{M},\sqrt{M(1+\mathcal{C}_{1})} \right\rbrace \right)$. This bound is tight for sequences whose path-length $P_{M}=\mathcal{C}_{1}$, matching the lower bounds for both the path-length and temporal variability.
% \end{remark}
%%%%%%%%%%%
\begin{theorem}[TAOG for dynamic environment]
  Let the initial parameters $\left\lbrace\phi_{m,0}\right\rbrace_{m=0}^{M}$ %be the initialization 
  for each task determined by %implicit online mirror descent or 
   follow the average leader. %on $\mathbb{E}_{\hat{v}_{m}}\left[D_{KL}\left(\hat{\pi}|\cdot\right)\right]$
   For each task we train the policy %gradient algorithm 
   for $T$ steps with learning rate $\alpha$ and obtain $\left\lbrace\hat{\phi}_{m,T}\right\rbrace_{m=1}^{M}$. Let $\phi_{m}^*$ is the optimal meta initialization for each task, then the task average optimality gap is bounded as      
% \begin{equation}\label{TAOG}
%     \frac{1}{M}\sum_{m=1}^{M}\mathbb{E}\left[F_m(\hat{\phi}_{m,T})\right] -F_m(\phi_{m}^{*})\leq \mathcal{O}\left({\frac{V_{M}+D^*}{\sqrt{T}M}}\right).
% \end{equation} 
\begin{equation}
    \bar{R}\leq \mathcal{O}\left(\min \left\lbrace V_{M}, \sqrt{M(1+{P}_{M})}\right\rbrace+\frac{\hat{S}_{*}}{\sqrt{T}}+\frac{\boldsymbol{\epsilon}_{M}}{M\sqrt{T}}\right).
\end{equation}
% \begin{equation}\nonumber
%     \frac{1}{M}\sum_{m=1}^{M}\mathbb{E}\left[F_m(\hat{\phi}_{m,T})\right] -F_m(\phi_{m}^{*})\leq \frac{10{D}^{*}}{\alpha T}+\frac{4V_{M}}{\alpha M T}+\frac{4\alpha |\mathcal{S}||\mathcal{A}|}{(1-\gamma)^{3}}.
% \end{equation}
% By tuning the learning rate $\alpha=\sqrt{\frac{\left(1-\gamma\right)^3}{4|\mathcal{S}||\mathcal{A}|}\left(10D^{*}+\frac{4V_{M}}{M}\right)}$, we get the tighter bound.
% \begin{equation}\label{optgap}
%     \frac{1}{M}\sum_{m=1}^{M}F(\pi_{m}^{*})-\mathbb{E}\left[F(\hat{\pi}_{m})\right] \leq \mathcal{O}\left(\sqrt{\frac{V_{M}}{\sqrt{TM}}+\frac{D^{*}}{T}}\right)
% \end{equation}
\end{theorem}
\begin{proof}
    By the within-task guarantee for MDPs, we know that the task average optimality gap is well defined and bounded by  
 \begin{equation}\label{eq42}
 \begin{split}
     \bar{R}&\leq \frac{2(N+4)L_{F} \sum_{m=1}^{M}D_{KL}\left(\phi^{*}_{m}|\phi_{m,0}\right)}{M\sqrt{T}}\\& \leq \frac{2(N+4)L_{F} }{\sqrt{T}}\sum_{m=1}^{M}\frac{D_{KL}\left(\phi^{*}_{m}|\phi_{m,0}\right)-D_{KL}\left(\phi^{*}_{m}|\varphi^{*}_{m}\right)}{M}\\&\quad+\frac{2(N+4)L_{F} }{\sqrt{T}}\sum_{m=1}^{M}\frac{D_{KL}\left(\phi^{*}_{m}|\varphi^{*}_{m}\right)}{M}.
 \end{split}
\end{equation} 
Now by KL divergence estimation error bounds in \eqref{KLBound}, the inequality \eqref{eq42} takes the form.
\begin{equation}\label{eq43}
\begin{split}
    \bar{R}\leq& \frac{2(N+4)L_{F} }{\sqrt{T}}\sum_{m=1}^{M}\frac{D_{KL}\left(\phi^{*}_{m}|\phi_{m,0}\right)-D_{KL}\left(\phi^{*}_{m}|\varphi^{*}_{m}\right)}{M}\\&+\frac{2(N+4)L_{F} }{\sqrt{T}}\sum_{m=1}^{M}\frac{D_{KL}\left(\hat{\phi}_{m}|\varphi^{*}_{m}\right)+\boldsymbol{\epsilon}_{m}}{M} 
\end{split}
\end{equation}
Since $\varphi^{*}_m=\arg\min_{\varphi}\sum_{m=1}^{M}D_{KL}\left(\hat{\phi}_{m}|\varphi\right)$
, thus by definition of $\hat{S}_{*}$, we have we have,
\begin{equation}\label{eq44}
    \frac{2(N+4)L_{F} }{\sqrt{T}}\sum_{m=1}^{M}\frac{D_{KL}\left(\hat{\phi}_{m}|\varphi^{*}_m\right)+\boldsymbol{\epsilon}_{m}}{M}\leq \frac{2(N+4)L_{F} \left(\hat{S}_{*}+\frac{\boldsymbol{\epsilon}_{M}}{M}\right)}{\sqrt{T}}
\end{equation}

Now we will upper bound the first term of \eqref{eq43}, since each $\phi_{m,0}$ is determined by playing IOMD or FTRL, then the following term can be upper bounded by the theorem \eqref{adaptivelr}

\begin{equation}\label{eq45}
\begin{split}
    \frac{2(N+4)L_{F} }{\sqrt{T}}&\sum_{m=1}^{M}\frac{D_{KL}\left(\phi^{*}_{m}|\phi_{m,0}\right)-D_{KL}\left(\phi^{*}_{m}|\varphi^{*}_m\right)}{M} \\& \leq \frac{2(N+4)L_{F} }{M\sqrt{T}} \mathcal{O}\left(\min \left\lbrace V_{M}, \sqrt{M(1+{P}_{M})}\right\rbrace\right) 
\end{split}
\end{equation}
 Now from the bounds in \eqref{eq44}and \eqref{eq45}, we can obtain the TAOG as:
\begin{equation}\label{c.42}
    \bar{R}\leq \mathcal{O}\left(\min \left\lbrace V_{M}, \sqrt{M(1+{P}_{M})}\right\rbrace+\frac{\hat{S}_{*}}{\sqrt{T}}+\frac{\boldsymbol{\epsilon}_{M}}{M\sqrt{T}}\right).
\end{equation}

\end{proof}
% \textbf{Proof.} We moved the detailed proof to  appendix \ref{B51}.\\

   The task-averaged regret upper bound \eqref{c.42} is sensitive to temporal variability $V_{M}$. Specifically, a lower $V_{M}$  results in a tighter bound, indicating the algorithm performs better in environments with stable, less variable tasks. A larger $\hat{S}^*$ loosens the upper bound, implying that as tasks become more dissimilar, the algorithm may become less effective at generalizing across these tasks.% Thus, the algorithm is expected to perform better in scenarios where tasks share underlying similarities, which could be particularly relevant for applications where tasks are variations of a core problem.
The terms $T$ and $M$ in the denominator suggest that increasing the number of iterations per task $T$ or the total number of tasks $M$ could lead to a reduced regret. However, the square root indicates a sub-linear rate.

 \bibliographystyle{elsarticle-num} 
 \bibliography{sample}




\end{document}